\PassOptionsToPackage{dvipsnames}{xcolor}
\documentclass{article}
\usepackage{iclr2027_conference,times}

\usepackage{amsmath,amsfonts,bm}

\def\eqref#1{equation~\ref{#1}}

\def\1{\bm{1}}

\DeclareMathAlphabet{\mathsfit}{\encodingdefault}{\sfdefault}{m}{sl}
\SetMathAlphabet{\mathsfit}{bold}{\encodingdefault}{\sfdefault}{bx}{n}

\newcommand{\KL}{{\mathrm{KL}}}

\DeclareMathOperator*{\argmin}{arg\,min}

\usepackage{amsmath,amssymb,amsthm,mathtools}
\usepackage{booktabs,multirow}
\usepackage{graphicx}
\usepackage{subcaption}
\usepackage{algorithm}
\usepackage[noend]{algpseudocode}
\usepackage{microtype}
\usepackage{tikz}
\usetikzlibrary{arrows.meta,positioning,fit}
\usepackage{hyperref}
\usepackage{url}
\usepackage{wrapfig}
\usepackage{caption} \usepackage{wasysym}

\usepackage{tcolorbox}
\theoremstyle{definition}
\newtheorem{definition}{Definition}
\theoremstyle{assumption}
\newtheorem{assumption}{Assumption}
\theoremstyle{plain}

\newtheorem{proposition}{Proposition}
\newtheorem{theorem}{Theorem}
\newcommand{\SB}{\operatorname{SB}}
\newcommand{\tsbm}{\textsc{TSBM}}
\newcommand{\hold}{h_{\mathrm{old}}}

\newcommand{\Fref}{b}
\newcommand{\Prwd}{P^{r}}
\newcommand{\nur}{p_1^{r}}
\title{Tilted Schrödinger Bridge Matching}

\author{Sergei Kholkin\thanks{Corresponding email: \texttt{kholkinsd@gmail.com}}\\Applied AI Institute
\And Evgeny Burnaev\\Applied AI Institute\\AXXX
\And Alexander Korotin\\Applied AI Institute\\AXXX}

\iclrfinalcopy

\begin{document}
\maketitle
\addtocontents{toc}{\protect\setcounter{tocdepth}{-1}}

\begin{abstract}

Schrödinger bridges provide an entropy-regularized framework and a principled solution for unpaired domain translation. In practice, a pretrained bridge may need to be adapted to human preferences or physical constraints through a reward a problem closely related to reward tilting in diffusion models but underexplored for Schrödinger bridges. We introduce Tilted Schrödinger Bridge Matching (\tsbm{}), a post-training method for fine-tuning a learned bridge $P$ between source $p_0$ and target $p_1$ toward a reward-tilted target $p_1^r\propto p_1e^r$, while preserving source $p_0$. We formulate this adaptation as alternating optimization initialized from $P$, provide theoretical justification, and derive a practical algorithm based on Adjoint Matching. We evaluate \tsbm{} on unpaired image-to-image translation targeting digit properties in MNIST and facial attributes in CelebA.

\end{abstract}

\vspace{-4mm}
\section{Introduction}
\vspace{-3mm}

Schrödinger Bridges (SBs) seek a path law that minimizes relative entropy with respect to a reference process while matching prescribed endpoint distributions \citep{schrodinger1931}. They connect stochastic control \cite{bellman1957dynamic}, entropic optimal transport \cite{lLeonard2014}, and diffusion-based generative modeling \cite{song2021scorebased}. SB methods learn \textit{minimal energy} transport from a source distribution $p_0$ to a terminal distribution $p_1$, with applications including unpaired domain transfer \citep{shi2023dsbm, gushchin2024light, kholkin2026ipmf}, trajectory inference \cite{tong2024simulationfree} and sampling from Boltzmann distributions \citep{liu2025asbs, tamogashev2026data}.

After training, however, one may wish to adjust the trained model to account for human preferences \cite{uehara2024entropy}, to enforce physical constraints \cite{li2025palsb, chung2023dps} or presence of image attributes \cite{dhariwal2021}. The standard way to induce such alignment is via diffusion model reward finetuning or reward tilting \cite{domingoenrich2025}, where the goal is to adjust the diffusion model output distribution from $p_0$ to $\nur(x)\propto p_1(x)e^{r(x)}$, where $r$ is a real-valued differentiable reward function. 

However, while there are a lot research for reward finetuning for Flow Matching \cite{liu2025flowgrpo} and Diffusion Models \cite{domingoenrich2025, black2024ddpo, clark2024draft} the reward finetuning for Schrödinger Bridges is underexplored. For memoryless diffusion references \cite{song2021scorebased}, a terminal reward yields the desired tilt $p_1^r\propto p_1e^r$. In diffusion bridges \cite{liu2023i2sb}, however, dependence between $X_0$ and $X_1$ generally biases the resulting terminal marginal. Existing reward-finetuning \cite{li2025palsb} and inference-time alignment \cite{chung2023cddb} methods are approximate and introduce different biases, which may lead to artifacts or reward hacking.

In this work, we tackle the problem of \textbf{reward tilting the Schrödinger Bridge models}. Our contributions are:
\begin{enumerate}
    \item \textbf{Theory.} We show that a Schrödinger bridge with target tilted by differentiable reward $p_1^r(x)\propto p_1(x)e^{r(x)}$ can be cast as stochastic optimal control relative to an exact pretrained bridge with target $p_1$ (\wasyparagraph~\ref{sec:theory}), enabling finetuning to the new target while preserving the source marginal.

    \item \textbf{Method.} We introduce Tilted Schrödinger Bridge Matching (\tsbm{}), a finetuning procedure which learns an incremental control via alternating Controller and Corrector Matching (\wasyparagraph~\ref{sec:method}). We validate \tsbm{} reward steering for reward tilted unpaired image translation on Colored MNIST and CelebA (\wasyparagraph~\ref{sec:experiments}) datasets.
    
\end{enumerate}

\begin{figure}[t]
\centering
\vspace{-7mm}
\includegraphics[width=\linewidth]{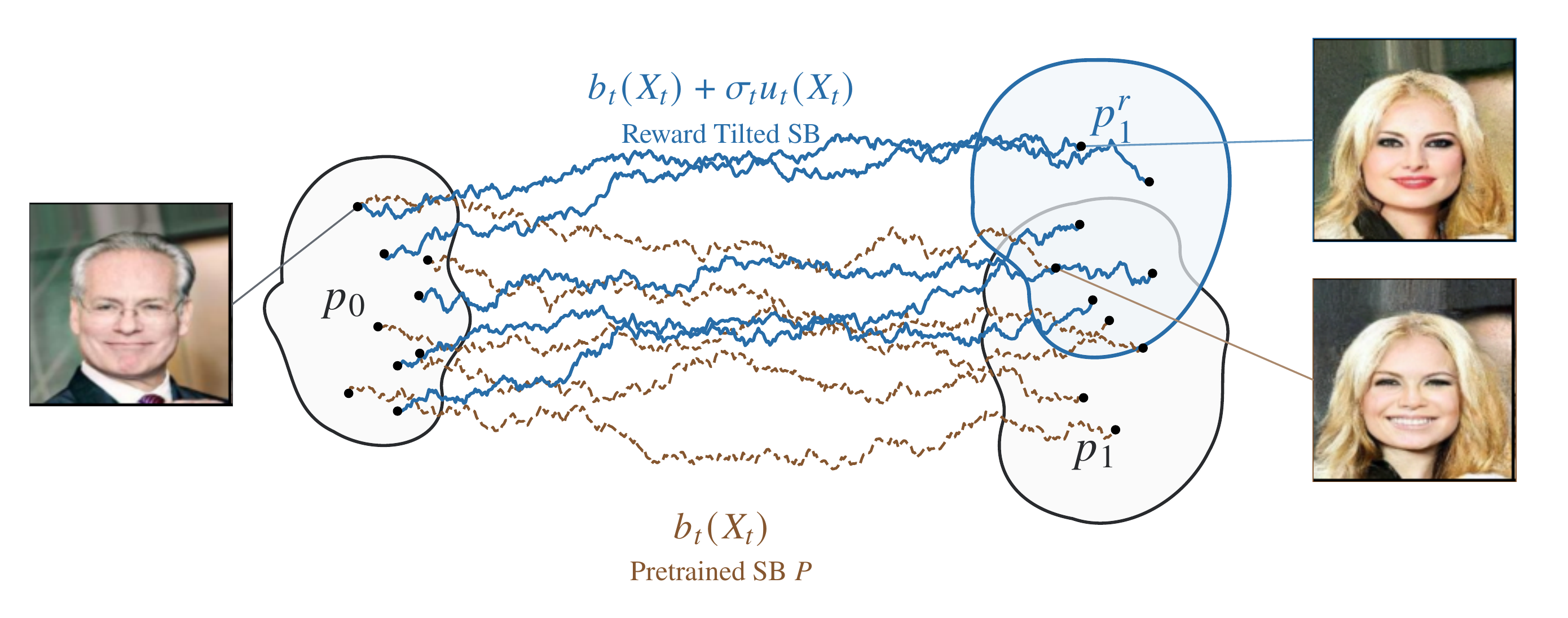}
\vspace{-7mm}
\caption{\textbf{Reward tilting a pretrained Schr\"odinger bridge.} The pretrained SB $P$ transports $p_0$ to $p_1$ with drift $b_t$ (\textcolor{brown}{brown}). TSBM learns an incremental control $u_t$, resulting in  drift $b_t+\sigma_tu_t$ (\textcolor{NavyBlue}{blue}), to target the reward $r(x)$ tilted marginal, i.e., $p_1^r(x) \propto p_1(x)r(x)$.}
\vspace{-5mm}
\label{fig:mechanism}
\end{figure}

\vspace{-2mm}
\paragraph{Notation.}
We work in $\mathbb{R}^d$, which is the D-dimensional Euclidean space equipped with the Euclidean norm $\Vert \cdot \Vert$. We use $\Omega$ to denote the space of trajectories, i.e., continuous $\mathbb{R}^D$-valued functions of $t \in [0, 1]$, then $X \in \Omega$ is the trajectory and $X_t \in \mathbb{R}^d$ is its slice at time $t$. We write $Q \in P(\Omega)$ to denote a path law $Q$, which is the  element of probability distributions set on the trajectories $\Omega$. For a path law $Q$, let $p_t^Q(x)$ denote the density of $X_t$ and let $p_{t\mid s}^Q(y\mid x)$ denote the conditional density of $X_t$ given $X_s=x$, for $0\leq s<t\leq1$. We assume that these distributions admit densities with respect to Lebesgue measure whenever the densities are used. 

In the whole manuscript we set the reference process $R$, which we assume can be represented as a diffusion with globally Lipschitz in $x_t$ drift function and continuous, finite and positive volatility coefficient $\sigma_t$:

\vspace{-8mm}
\begin{equation}
    R: \mathrm{d}X_t=f_t(X_t)\,\mathrm{d}t
    +\sigma_t\,\mathrm{d}W_t, \quad X_0\sim p_0.
\end{equation}
\vspace{-5mm}

The diffusion $R$ then can be augmented by controller $u_t: \mathbb{R}^d \times \mathbb{R}_{+}\rightarrow \mathbb{R}^d$, and the \textit{controlled diffusion} is denoted as $R^{u}$:
\vspace{-2mm}
\begin{equation}
    R^{v}: \mathrm{d}X_t=[f_t(X_t)+\sigma_tu_t(X_t)]\,\mathrm{d}t
    +\sigma_t\,\mathrm{d}W_t,\quad X_0\sim p_0,
\end{equation}
\vspace{-4mm}
where $W_t$ is standard Brownian motion.

\vspace{-1mm}
\section{Background}\label{sec:background}
\vspace{-2mm}

\vspace{-2mm}
\paragraph{Stochastic optimal control.}
In our work we consider commonly used in generative modeling Stochastic optimal control (SOC) problem  \cite{bellman1957dynamic} the quadratic cost control-affine problem formulation \cite{domingoenrich2024socm}. For the original reference $R$, our goal is add such a control $u_t$ to it to minimize the following objective:

\vspace{-3mm}
\begin{equation}
\begin{aligned}
    u^* = \min_v\;&\mathbb{E}_{R^u}\!\left[
    \frac12\int_0^1 \|u_t(X_t)\|^2\,\mathrm{d}t - r(X_1)
    \right] \\
    \text{s.t.}\;&
    \mathrm{d}X_t =
    \bigl[f_t(X_t)+\sigma_t u_t(X_t)\bigr]\,\mathrm{d}t
    + \sigma_t\,\mathrm{d}W_t,
    \qquad X_0\sim p_0,
\end{aligned}
\label{eq:soc_def}
\end{equation}
\vspace{-3mm}

where $X_t$ is the process trajectory slice, $r$ is a terminal reward, $f_t$ is reference drift, $\sigma_t$ is reference volatility, which are not optimized and $u^*$ is called optimal controller. 

Under the usual regularity assumptions the optimal distribution on endpoints has form:

\vspace{-4mm}
\begin{equation}
    p_{0, 1}^{R^{v_g}}(X_1|X_0)
    =p_{0, 1}^{R}(X_1|X_0)\exp\!\left[r(X_1)+V_0(X_0)\right],
    V_0(x)=-\log\mathbb{E}_{R}\!\left[e^{r(X_1)}\mid X_0=x\right],
    \label{eq:soc-optimal-law}
\end{equation}
\vspace{-3mm}

when $p_{0, 1}^{R}$ is memoryless, i.e., $p_{0, 1}^{R}(X_0, X_1)= p_{0}^{R}(X_0)p_{1}^{R}(X_1)$ the SOC problem admits a particularly simple characterization: the terminal marginal of the controlled process satisfies $p_1^{R^{v_g}}(X_1)\propto p_1^{R}(X_1)r(X_1)=p_1^r(X_1)$. In contrast, for non-memoryless models, such as Diffusion Bridges \cite{liu2023i2sb} or Schrödinger Bridges \cite{debortoli2021}, this simplification no longer holds and the resulting terminal marginal in general does not coincide with the desired reward-tilted distribution $p_1^r$. We provide further discussion in Appendix~\ref{app:reward-tilting-nonmemoryless}.

\vspace{-2mm}
\paragraph{Controller Matching.}  In this work, we are focused on Adjoint Matching \cite{domingoenrich2025} algorithm for solving SOC problems thanks to its scalability and robustness. Given a SOC problem Eq~\ref{eq:soc_def} one can define a "lean" adjoint state $\tilde{a}$:

\vspace{-3mm}
\begin{equation}
\begin{aligned}
\frac{\mathrm{d}\tilde{a}_t(X)}{\mathrm{d}t}
= -\nabla_x f_t(X_t)^\top \tilde{a}_t(X), \quad \tilde{a}_1(X)
= -\nabla_x r(X_1),
\label{eq:lean-adj}
\end{aligned}
\end{equation}
\vspace{-3mm}

where $\nabla_x f_t(X_t)$ is the Jacobian of reference drift $f_t(X_t)$. Which in turn allows to define the optimal controller $u^*$, as the argmin of  matching objective:

\vspace{-4mm}
\begin{equation}
    u^* = \argmin_u\mathbb{E}_{X \sim P^{\overline{u}}} \left[ \int_0^1
    \left\|u(X_t,t)+\sigma_t\tilde{a}_t(X)\right\|^2\mathrm{d}t\right], \quad \overline{u} = \text{stopgrad}(u),
    \label{eq:controller-loss}
\end{equation}
\vspace{-3mm}

where differentiation through the generation trajectory $X \sim P^{\overline{u}}$ may be ommitted, which makes algorithm scalable and popular in different applications \cite{havens2025adjointsampling, liu2025asbs, gutjahr2026constrained}. Although the SDE simulation for sampling $X \sim P^{\overline{u}}$ and ODE simulation to calculate the "lean" adjoint $\tilde{a}_t$ are still required and can be rather expensive.

\vspace{-2mm}
\paragraph{Schr\"odinger bridges.}

While regular SOC problem Eq~\ref{eq:soc_def} has only one marginal constraint, i.e., $X_0 \sim p_0$, the Schrödinger Bridge Problem imposes two marginal constraints on the endpoints, i.e., $X_0 \sim p_0$ and $X_1 \sim p_1$, lefts out the terminal reward $r(x)$, but keeps the $u_t(X_t, t)$ minimization term, which can be read as minimal energy constraint:

\vspace{-4mm}
\begin{equation}
\begin{aligned}
    u^* = \argmin_v&\mathbb{E}_{R^u}\!\left[
    \frac12\int_0^1 \|u_t(X_t)\|^2\,\mathrm{d}t
    \right] =  \argmin_v
    \text{KL}(R^u \Vert R)
    \\
    \text{s.t.}\;&
    \mathrm{d}X_t =
    \bigl[f_t(X_t)+\sigma_t u_t(X_t)\bigr]\,\mathrm{d}t
    + \sigma_t\,\mathrm{d}W_t,
    \qquad X_0\sim p_0, \quad X_1 \sim p_1.
\end{aligned}
\label{eq:sb}
\end{equation}
\vspace{-3mm}

We optimize over admissible finite-energy controls and assume the bridge is attained in this class. The resulting controlled process $P$ is called Schrödinger Bridge between source distribution $p_0$ and target distribution$p_1$ with reference $R$ and is denoted in our paper as:

\vspace{-4mm}
\begin{equation}
    P=\SB_R(p_0,p_1)
\end{equation}
\vspace{-5mm}

In other words, it takes the reference diffusion $R$ and searches for closest to it, in KL divergence, diffusion that satisfies marginal constrains \citep{lLeonard2014}. The most popular choice of $R$ is Wiener process with constant volatility $\epsilon$, i.e., $W^\epsilon$. In this case, minimizing the path-space KL divergence is equivalent to minimizing the expected control energy required to transport $p_0$ to $p_1$, providing a principled formulation of unpaired domain translation \cite{shi2023dsbm}.

\vspace{-2mm}
\paragraph{Schr\"odinger bridges and Stochastic Optimal Control.} Under mild regularity assumptions this problem can be converted into regular SOC problem with one marginal constraint \cite{liu2025asbs}: 

\vspace{-4mm}
\begin{equation}
\begin{gathered}
    u^\star = \arg\min_u
    \mathbb{E}_{R^u}\!\left[
    \frac12\int_0^1 \|u_t(X_t)\|^2\,\mathrm{d}t
    +\log\frac{\widehat\varphi_1(X_1)}{p_1(X_1)}
    \right],
    \\[6pt]
    \text{s.t.}\quad
    \mathrm{d}X_t =
    \bigl[f_t(X_t)+\sigma_t u_t(X_t)\bigr]\,\mathrm{d}t
    +\sigma_t\,\mathrm{d}W_t,
    \qquad X_0\sim p_0.
\end{gathered}
\label{eq:sb-terminal-cost-control}
\end{equation}
\vspace{-3mm}

where $\widehat\varphi_1(X_1)$ is the Schrödinger Potential \cite{lLeonard2014}. Which is defined as follows:

\vspace{-6mm}
\begin{subequations}
\label{eq:schrodinger-potentials}
\begin{align}
    \varphi_t(x_t)&=\int p_{1\mid t}^R(x_1\mid x_t)\varphi_1(x_1)\,\mathrm{d}x_1,
    &\varphi_0(x_0)\widehat\varphi_0(x_0)&=p_0(x_0),
    \label{eq:backward-potential}\\
    \widehat\varphi_t(x_t)
    &=\int p_{t\mid0}^R(x_t\mid x_0)\widehat\varphi_0(x_0)\,\mathrm{d}x_0,
    &\varphi_1(x_1)\widehat\varphi_1(x_1)&=p_1(x_1).
    \label{eq:terminal-density-potential}
\end{align}
\end{subequations}
\vspace{-3mm}

The Eq~\ref{eq:backward-potential} propagates $\varphi_1$ backward in time, while Eq~\ref{eq:terminal-density-potential} propagates $\widehat\varphi_0$ forward in the density convention. Their product gives the bridge marginal at time $t$: $p_t^P=\varphi_t\widehat\varphi_t$. One can as well say that Schrödinger Potential $\widehat\varphi_1(X_1)$ helps to \textit{debias} the SOC problem solution to fit the marginal $p_1$. Such result is widely used and can lead the application of SOC problems solving algorithms, such as \cite{domingoenrich2025, gushchin2023enot}, to solving the Schrodinger Bridge Problem. 

\vspace{-2mm}
\paragraph{Corrector Matching.} Several works solve the Schr\"odinger bridge problem through the SOC formulation Eq~\ref{eq:sb-terminal-cost-control} \citep{gushchin2023enot,liu2025asbs}. This formulation requires the terminal cost $\log(\widehat\varphi_1/p_1)$, which is initially unknown. \citet{gushchin2023enot} learn this cost through the entropic optimal transport dual formulation. In contrast, \citet{liu2025asbs} use Adjoint Matching Eq~\ref{eq:controller-loss}, which requires only its gradient. When the target density $p_1$ is known up to a normalizing constant, its score $\nabla_x\log p_1$ is available, leaving only the \textit{corrector} $h^*:=\nabla_x\log\widehat\varphi_1$ to be learned. The corrector admits the regression characterization:

\vspace{-3mm}
\begin{equation}
    h^*(x) := \nabla_x \log \varphi(x)= \argmin_h \mathbb{E}_{X_1, X_0 \sim P^{u^*}}[\Vert h(X_1) - \nabla_{x_1}\log p^{R}(X_1|X_0) \Vert^2].
    \label{eq:corrector-loss}
\end{equation}
\vspace{-3mm}

Note that this objective requires samples from the optimal bridge $P^{u^*}$, which is itself unknown. To resolve this interdependence, controller and corrector updates are alternated, using trajectories generated by the current controller to update the corrector \citep{liu2025asbs}.

\vspace{-3mm}
\section{Tilting a Schrödinger Bridge}
\label{sec:theory}
\vspace{-3mm}

We first define the Tilted Schrödinger Bridge Problem (TiltedSBP) in \wasyparagraph~\ref{sec:tilted-sbp}, then relate formulations using the base process $R$ and pretrained bridge $P$ as references. In \wasyparagraph~\ref{sec:relative-control}, we show equivalence of the corresponding path-space solutions (Prop.~\ref{prop:change-reference}), relate their Schrödinger potentials (Prop.~\ref{prop:relative-log-potentials}), and derive a relative SOC formulation that avoids explicit knowledge of $p_1$ (Prop.~\ref{prop:relative-soc-problem}).

\vspace{-2mm}
\subsection{Tilted Schrödinger Bridge Problem}\label{sec:tilted-sbp}
\vspace{-2mm}

First, let us recap the basic Schrodinger Bridge Problem Eq~\ref{eq:sb} between distributions $p_0, p_1$ and its solution $P=\SB_R(p_0,p_1)$, which we call \textbf{pretrained Schrödinger Bridge}. Then let us tilt the target distribution by some reward function $r(x)$, i.e., $\nur(x) \propto p_1(x)e^{r(x)}$. Then the central problem that we target is this paper is finding a Schrödinger Bridge between source distribution $p_0$ and target $p_1^r$, or more formally:

\begin{tcolorbox}[
    colback=gray!10,
    colframe=gray!60,
    boxrule=0.5pt,
    arc=3mm,
    left=3mm,
    right=3mm,
    top=2mm,
    bottom=2mm
]
\begin{definition}[TiltedSBP]
\label{def:tilted-sb}
Let $R \in P(\Omega)$ be a reference diffusion law on $\Omega = C([0, 1], \mathbb{R}^d)$.
Given a target and source probability densities $p_0, p_1 \in P(\mathbb{R}^d)$ and a measurable reward $r:\mathbb{R}^d\to\mathbb{R}$, assume $0 < \int p_1(x)e^{r(x)}dx < \infty$ and define the reward-tilted target density:

\vspace{-3mm}
\begin{equation}
    p_1^r(x) \propto p_1(x)e^{r(x)}
\end{equation}
\vspace{-4mm}

The \emph{Tilted Schrödinger Bridge problem} is to find:

\vspace{-3mm}
\begin{equation}
\begin{gathered}
    u^\star = \argmin_u
    \mathbb{E}_{R^u}\!\left[
    \frac12\int_0^1 \|u_t(X_t)\|^2\,\mathrm{d}t
    \right]
    \\[6pt]
    \text{s.t.}\quad
    \mathrm{d}X_t =
    \bigl[f_t(X_t)+\sigma_t u_t(X_t)\bigr]\,\mathrm{d}t
    +\sigma_t\,\mathrm{d}W_t,
    \qquad X_0\sim p_0,\quad X_1\sim\underline{p_1^r}.
\end{gathered}
\label{eq:tilted-sb-problem}
\end{equation}

\end{definition}
\end{tcolorbox}

Where the optimal solution is denoted by $P^{u^*}=\SB_R(p_0,p_1^r)$. While problem formulation being fully valid the practical solution can have problems. In many practical cases the distribution $p_1$ distribution can be given by empirical samples, while $r(x)$ can be given by reward, which would complicate the distribution $p_1^r$ representation.

\vspace{-2mm}
\subsection{Relative Stochastic Optimal Control formulation for TiltedSBP}
\label{sec:relative-control}
\vspace{-2mm}

As is standard in many reward finetuning methods \cite{uehara2024entropy, domingoenrich2025}, which also target reward tilted target distribution $p_1^r$, the tilting problem defined relative the to the base process, i.e., $P=\SB_R(p_0,p_1)$ in our case. Which gives us motivation to define the TiltedSBP problem rather with as $P$ as a reference, i.e., which we call \textbf{relative problem}, rather than with $R$ as a reference. We first state conditions that will be actively used throughout the paper and show that changing the reference preserves the TiltedSBP solution.

\begin{assumption}[Reference and reward] \label{assumptions}The reference starts from $p_0$, $p_{0, 1}^{R}$ and $p_0 \otimes p_1$ define equivalent measures with $\text{KL}(p_0 \otimes p_1 \Vert p^R_{0, 1}) < \infty$ and $r$ is bounded above.
\end{assumption}

\begin{proposition}[Change of reference]
\label{prop:change-reference}
Under Assumptions~\ref{assumptions},
\vspace{-1mm}
\begin{equation}
    \SB_P(p_0,\nur)=\SB_R(p_0,\nur)=\Prwd.
    \label{eq:change-reference}
\end{equation}
\vspace{-3mm}
\end{proposition}

Thus, terminal \textit{retargeting} can use the pretrained bridge as its reference without changing the desired solution. This permits learning an incremental control around the pretrained drift of $P$. The detailed proof is given in Appendix~\ref{app:proofs}. However, $\SB_P(p_0,\nur)$ and $\SB_R(p_0,\nur)$ problems while being equivalent have different Schrodinger Potentials Eq~\ref{eq:schrodinger-potentials}.

Let us further explore the connections between the TiltedSBP problem with $R$ as reference and TiltedSBP problem with $P$ as reference, in the context of Schrödinger Potentials. Let $\widehat\varphi^{\mathrm{old}}_1$, $\widehat\varphi^r_1$ and $\widehat\psi^r_1$ denote the terminal Schrodinger Potentials of $\SB_R(p_0, p_1)$, $\SB_R(p_0,\nur)$, $\SB_P(p_0,\nur)$ problems correspondingly, as defined in Eq~\ref{eq:schrodinger-potentials}. The relative SOC terminal cost $f_{\rm rel}^r(x) := \log\frac{\widehat\psi^r_1(x)}{p_1^r(x)}$ and can be expressed without $p_1$ factor:

\begin{proposition}[Relative terminal potential relation]
\label{prop:relative-log-potentials}
The terminal cost term of SOC formulated relative SB problem $\SB_P(p_0,\nur)$ admits:
\vspace{-2mm}
\begin{equation}
    f_{\rm rel}^r(x) := \log\frac{\widehat\psi^r_1(x)}{p_1^r(x)}
    =\log\widehat\varphi^r_1(x)
    -\log\widehat\varphi^{\mathrm{old}}_1(x) - r(x) + C,
    \label{eq:relative-terminal-log-potential}
\end{equation}
\vspace{-3mm}
where $C$ is a constant that doesn't depend on $x$.
\end{proposition}

The explicit target density $p_1$ cancels from the terminal cost, while its influence remains in the old potential $\widehat\varphi_1^{\mathrm{old}}$. The proof is given in Appendix~\ref{app:proofs}. Then to construct the terminal cost $f^r_{\rm rel}$ for relative SOC problem Eq~\ref{eq:sb-terminal-cost-control} one needs only the reward $r$ and Schrodinger Potentials of absolute problems $\log\widehat\varphi^r_1(x)$ with $R$ as reference, i.e., $\log\widehat\varphi^r_1(x)$ or $\log\widehat\varphi^{\mathrm{old}}_1(x)$.

Under the assumption that reward $r$ and the log-potentials admit continuous differentiable versions on a region, where $p_1(x) >0$ the identity also given the terminal-cost gradient there:

\vspace{-3mm}
\begin{equation}
    \nabla f_{\mathrm{rel}}^r(x)
    =\nabla\log\widehat\varphi^r_1(x) - \nabla\log\widehat\varphi^{\mathrm{old}}_1(x) -\nabla r(x)
    =h^r(x)-\hold(x) -\nabla r(x),
    \label{eq:relative-terminal-cost-gradient}
\end{equation}
\vspace{-3mm}

where $\nabla\log\widehat\varphi^r_1$ is $h^r$, $\nabla\log\widehat\varphi^{\mathrm{old}}_1$ is $\hold$.  Section~\ref{sec:method} makes use of Eq~\ref{eq:relative-terminal-cost-gradient} and develops a procedure to learn the terminal corrector $\nabla f_{\mathrm{rel}}^r$. Then finally, we can construct the relative to $P$ SOC Problem Eq~\ref{eq:sb-terminal-cost-control} with optimal control $u_r^\star$:

\vspace{-1mm}
\begin{tcolorbox}[
    colback=gray!10,
    colframe=gray!60,
    boxrule=0.5pt,
    arc=3mm,
    left=3mm,
    right=3mm,
    top=2mm,
    bottom=2mm
]
\begin{proposition}[Relative Schrödinger Bridge Stochastic Optimal Control Problem]
\label{prop:relative-soc-problem}
Given the Assumptions of Def~\ref{def:tilted-sb}, Schrödinger Bridge $P=\SB_R(p_0, p_1)$ with optimal control $u^\star$ and corresponding drift $b_t$ define such a controlled by $u_t$ diffusion $P^{u_t}$:

\begin{equation}
    P^{u_t}:\quad \mathrm{d}X_t=[\underbrace{f_t(X_t) + \sigma_tu^\star_t(X_t)}_{b_t(X_t)} +\sigma_t u_t(X_t)]\,\mathrm{d}t
    +\sigma_t\,\mathrm{d}W_t,\qquad X_0\sim p_0,
    \label{eq:controlled-diffusion}
\end{equation}

where $b_t(X_t)$ is the resulting controlled drift of $P$. Then optimal control $u_r^\star$ provided as a solution to the following optimization problem:
\begin{equation}
    u_r^\star = \arg\min_u
    \mathbb{E}_{P^u}\!\left[
    \frac12\int_0^1\|u_t(X_t)\|^2\,\mathrm{d}t
    +\log\widehat\varphi^r_1(x)
    -\log\widehat\varphi^{\mathrm{old}}_1(x)-r(x)\right],
    \label{eq:relative-terminal-cost-control}
\end{equation}

would result in solution to the TiltedSBP, i.e., $P^{u_r^\star} = \SB_R(p_0,\nur)$.
\end{proposition}
\end{tcolorbox}
\vspace{-1mm}

This formulation requires the pretrained bridge $P$, its terminal potential $\log\widehat\varphi^{\mathrm{old}}_1$, and reward $r$, but importantly no further access to the base target $p_1$ through samples or density evaluations. The remaining unknown potential $\log\widehat\varphi^r_1$ is addressed in the next section.

\vspace{-3mm}
\section{Tilted Schrödinger Bridge Matching}
\label{sec:method}
\vspace{-3mm}
In this section, we introduce an algorithm for solving the TiltedSBP. Prop.~\ref{prop:relative-soc-problem} gives a feasible SOC formulation, but obtaining the optimal controller $u^{\*}$ requires the corrector $\log\widehat\varphi^r_1(x)$. Since the controller and corrector are interdependent, we propose an alternating optimization procedure: update the controller $u^{(k)}$ with fixed corrector $h^{(k)}$, then update the corrector $h^{(k+1)}$ with fixed controller $u^{(k)}$. For controller optimization, we use Controller Matching Eq~\ref{eq:controller-loss} and define its iterative fixed-corrector version as the \textbf{Controller Update}:

\begin{algorithm}[t]
\caption{Tilted Schrödinger Bridge Matching (\tsbm{})}
\label{alg:tsbm}
\begin{algorithmic}[1]
\Require frozen old drift $\Fref$, old corrector $\hold$, reference process $R$, reward $r$, stages $K$, number of controller updates $M_{\rm ctrl}$, number of corrector updates $M_{\rm corr}$
\State initialize $u^{(0)} \gets0$ and $h^{(0)}\gets\hold$
\For{$k=0,\ldots,K-1$}
    \State $u \gets \text{copy}(u^{(k)})$
    \For{$m=0,\ldots,M_{\rm ctrl}-1$}
      \State Sample \textit{paths} $(X_t)_{t \in [0, 1]}$ from $P^{\overline{u}}$ by the inference of SDE with drift $\Fref+\sigma u$;
      \State $\tilde{a}^{(k)}_1\gets-\nabla r(X_1)+h^{(k)}(X_1)-\hold(X_1)$;
      \State Propagate $\tilde{a}^{(k)}_t$ by solving backward adjoint lean ODE using Eq~\ref{eq:lean-adj};
      \State Make a gradient step on $\mathcal{L}_{\rm ctrl}$ w.r.t. $u$ parameters; 
   \EndFor
   \State Sample \textit{endpoints} $(X_0, X_1)$ from $P^{\overline{u}}$ by the inference of SDE with drift $\Fref+\sigma u$;
   \State $h \gets \text{copy}(h^{(k)})$
   \For{$m=0,\ldots,M_{\rm corr}-1$}
\State Make a gradient step on $\mathcal{L}_{\rm corr}$ w.r.t. $h$ parameters;
  \EndFor $u^{(k+1)} \gets u$, $h^{(k+1)}\gets h$
  \State 
\EndFor
\State \Return $u^{(K)}, h^{(K)}$
\end{algorithmic}
\end{algorithm}

\vspace{-3mm}
\begin{equation}
    u^{(k)} = \argmin_u \mathbb{E}_{P^{\overline{u}}} [\mathcal{L}_{\mathrm{ctrl}}(u, X, h^{(k)})]
    =: \argmin_u \mathbb{E}_{P^{\overline{u}}}\int_0^1
    \left\|u(X_t,t)+\sigma_t\tilde{a}^{(k)}_t(X)\right\|^2\mathrm{d}t,
    \label{eq:tsbm-controller-update}
\end{equation}
\vspace{-3mm}

\begin{equation}
\begin{aligned}
\tilde{a}_1(X)
&= \nabla_x f^r_{\rm rel}(X_1)
 = -\nabla_x r(X_1) + h^{(k)}(X_1) - h_{\rm old}(X_1), \\
-\frac{\mathrm{d}\tilde{a}^{(k)}_t(X)}{\mathrm{d}t}
&= \nabla_x \Fref_t(X_t)^\top \tilde{a}^{(k)}_t(X), \label{eq:lean-ode-tsbm}
\end{aligned}
\end{equation}
\vspace{-3mm}

where $\overline{u}=\text{stopgrad}(\overline{u})$ and for optimization of corrector we choose Corrector Matching \wasyparagraph~\ref{sec:background} and call its iterative version as \textbf{Corrector Update}:

\vspace{-3mm}
\begin{equation}
    h^{(k+1)} = \argmin_h \mathbb{E}_{P^{u^{(k)}}}[\mathcal{L}_{\mathrm{corr}}(h, X_0, X_1)] =: \argmin_h \mathbb{E}_{P^{u^{(k)}}}[\Vert h(X_1) - \nabla_{x_1}\log p^{R}(X_1|X_0) \Vert^2].
    \label{eq:tsbm-corrector-update}
\end{equation}
\vspace{-3mm}

Now both Controller and Corrector Matching updates do relax their interdependency into dependency on the previous iteration $(k)$ results and yield a fixed point alternative updates algorithm. We call this iterative procedure \textbf{\tsbm{}}, each iteration $k$ we call \textbf{stage}. Next we show that this algorithm indeed converges to the Tilted Schrodinger Bridge:

\begin{theorem}[TSBM convergence]
\label{thm:population}
Let $P=\SB_R(p_0,p_1)$. Under Assumptions~\ref{assumptions} and~\ref{ass:population-convergence}, the exact population updates in Eq~\ref{eq:tsbm-controller-update} and Eq~\ref{eq:tsbm-corrector-update}, initialized by $u^{(0)}=0$ and $h^{(0)}=\hold$, satisfy
\begin{equation}
    \KL(P^{u^*}\Vert P^{u^{(k)}})\longrightarrow 0,
    \qquad P^{u^*}=\SB_R(p_0,\nur).
\end{equation}
\end{theorem}

The proof follows the ASBS/IPF strategy \citep{liu2025asbs} in our relative setting and is given in Appendix~\ref{app:proofs}, with the pseudo-algorithm in Algorithm~\ref{alg:tsbm} and implementation details in Appendix~\ref{app:tsbm-design-choices}. In practice, we finetune a pretrained Schrödinger bridge $P$, typically represented by forward and backward drift networks learned iteratively \citep{shi2023dsbm,vargas2023dds,liu2025asbs,tamogashev2026data,kholkin2026ipmf}. The forward drift approximates $b$ and is used in the Controller Update (Eq.~\ref{eq:tsbm-controller-update}), while the backward drift approximates the pretrained corrector $h_{\rm old}$ \cite{liu2025asbs} and is used in both updates (Eqs.~\ref{eq:tsbm-controller-update} and \ref{eq:tsbm-corrector-update}). We use the same parameterization for \tsbm{}. Replay buffers, network architectures, and other details are deferred to Appendix~\ref{app:tsbm-design-choices}.

\vspace{-3mm}
\section{Related Work}
\label{sec:related-work}
\vspace{-3mm}

\paragraph{Schrödinger bridge solvers.}
Schrödinger Bridge solvers are mostly based on diffusion models \cite{song2021scorebased}. Some methods make use of Entropic Optimal Transport connection \cite{gushchin2024light, gushchin2023enot}, while most popular methods are based on iterative procedures, such as Iterative Proportional Fitting (IPF) \cite{debortoli2021, vargas2021maximum}, Iterative Markovian Fitting (IMF) \cite{shi2023dsbm, gushchin2024adversarial, ksenofontov2025categorical} or their hybrid \cite{kholkin2026ipmf}. The Adjoint Matching is utilized as a part of SB solvers \cite{liu2025asbs, havens2025adjointsampling}. In practice, the SB models are utilized for unpaired domain translation \cite{shi2023dsbm}, trajectory inference \cite{tong2024simulationfree, noble2026twisted}, sampling from Boltzman distributions \cite{liu2025asbs, guo2026dasbs, havens2025adjointsampling, tamogashev2026data}, Mutual Information Estimation \cite{kholkin2026infobridge, zabarianska2026discrete} or even data generation \cite{shin2026asbm}. However, there are no current methods that do support exact marginal reward tilting of the pretrained SB models.

\paragraph{Reward steering and finetuning.} Reward steering adapts pretrained generative models to objectives
beyond matching the training distribution, such as human preferences
in image generation \citep{black2024ddpo}, DNA and protein design
\citep{wang2025drakes}, and physical constraints
\citep{tauberschmidt2026physics}.
Methods include inference-time guidance, such as classifier
and classifier-free guidance \citep{dhariwal2021,ho2022}, diffusion
posterior sampling \citep{chung2023dps}, and particle-based methods
\citep{delmoral2006,wu2023tds}, which can increase sampling costs or
introduce artifacts under strong guidance. Finetuning alternatives
include reward optimization
\citep{black2024ddpo,clark2024draft}
and preference learning \citep{wallace2024diffusiondpo,kim2024pfm}. However, accounting for the pretrained distribution further allows for regularization and targets its exponential tilt rather than only maximizing reward \citep{domingoenrich2025,uehara2024entropy,zhao2025d1,liu2025flowgrpo}.
Stochastic optimal control methods emerged as alternative branch of methods allowing for  targeting an exponential tilt, while the methods include matching based methods \cite{domingoenrich2024socm, domingoenrich2025, bergmeister2026ram} or differentiation through trajectory based methods \cite{wang2025drakes, vargas2023dds}. However, these works do target mostly only noise-to-data generative models.

\begin{figure}[!t]
\centering

\captionsetup[subfigure]{font=small,labelfont=bf}
\scriptsize
\setlength{\tabcolsep}{0pt}

\vspace{-3mm}
\begin{tabular}{@{}r@{\hspace{3pt}}c@{\hspace{4pt}}c@{\hspace{4pt}}c@{}}

\textbf{Input} &
\includegraphics[
    width=0.295\linewidth,
    trim=192bp 274bp 884bp 0bp,
    clip
]{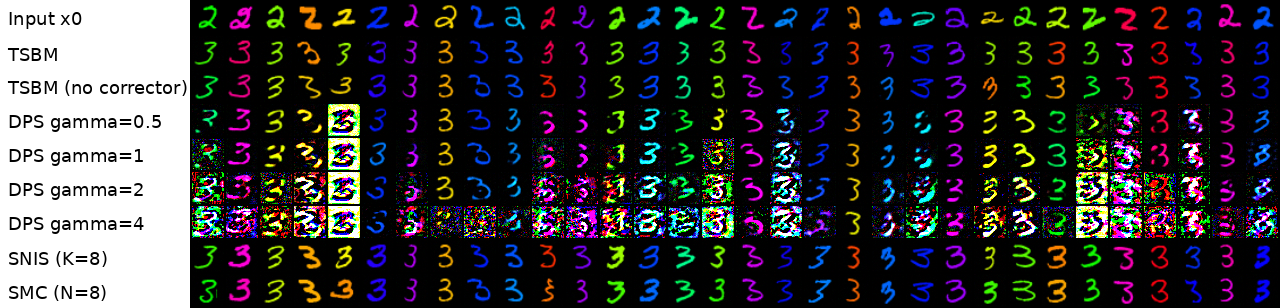} &
\includegraphics[
    width=0.295\linewidth,
    trim=192bp 886bp 884bp 0bp,
    clip
]{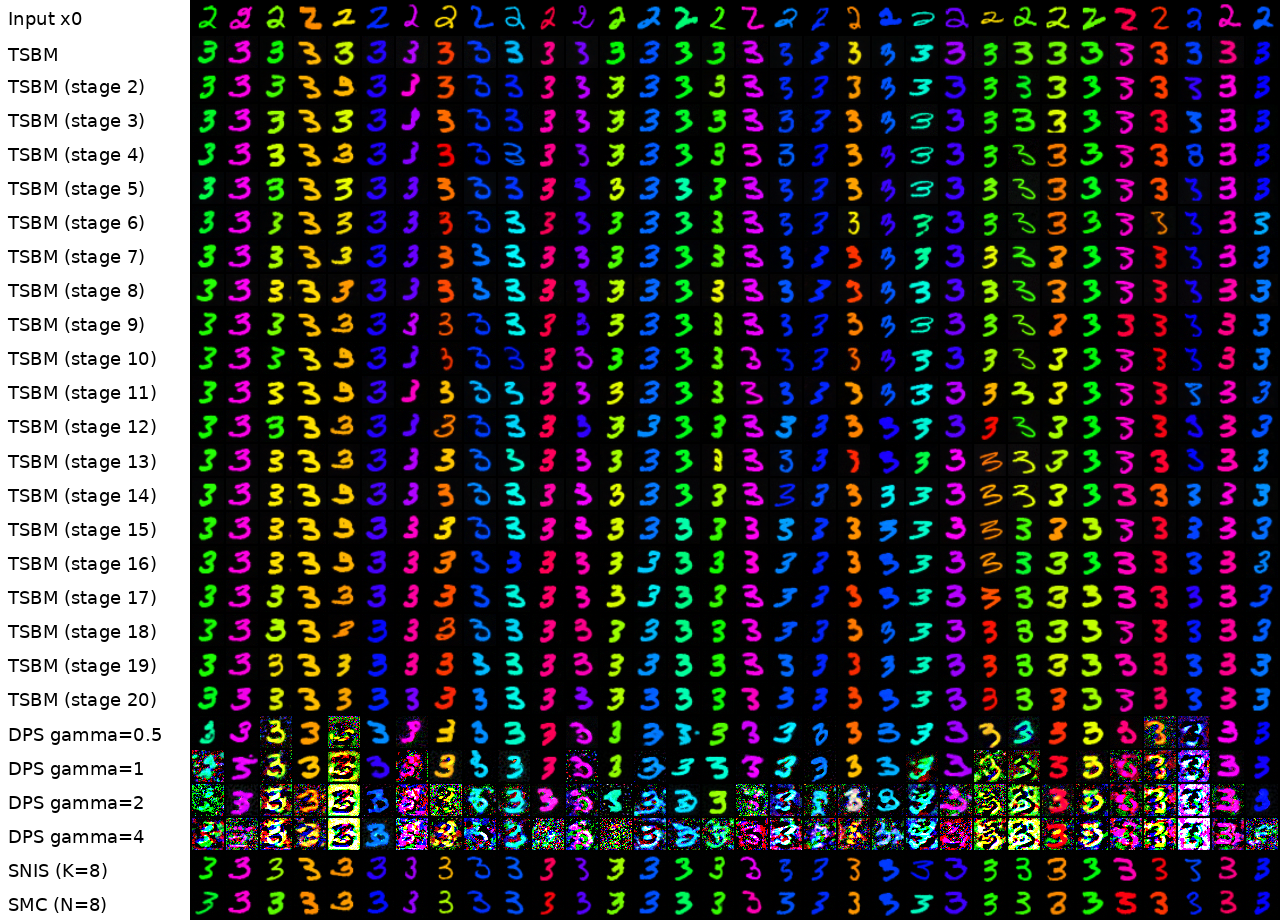} &
\includegraphics[
    width=0.295\linewidth,
    trim=192bp 274bp 884bp 0bp,
    clip
]{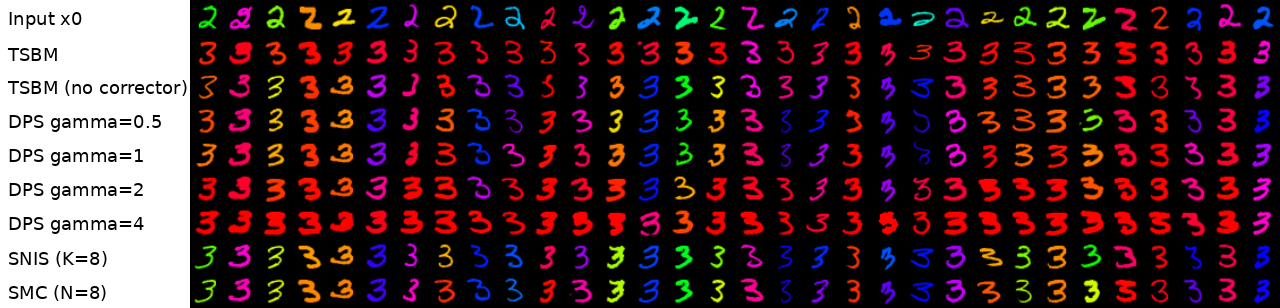}
\\[-2pt]

\textbf{DSBM} &
\includegraphics[
    width=0.295\linewidth,
    trim=192bp 274bp 884bp 34bp,
    clip
]{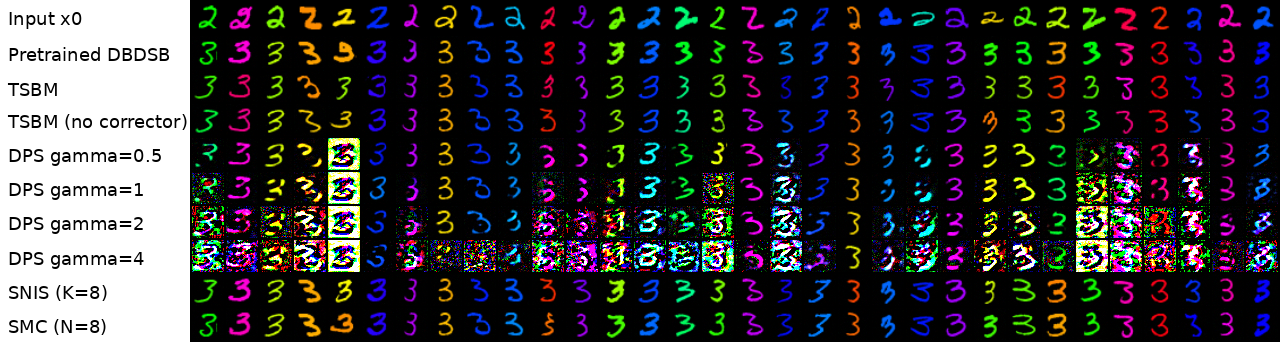} &
\includegraphics[
    width=0.295\linewidth,
    trim=192bp 886bp 884bp 34bp,
    clip
]{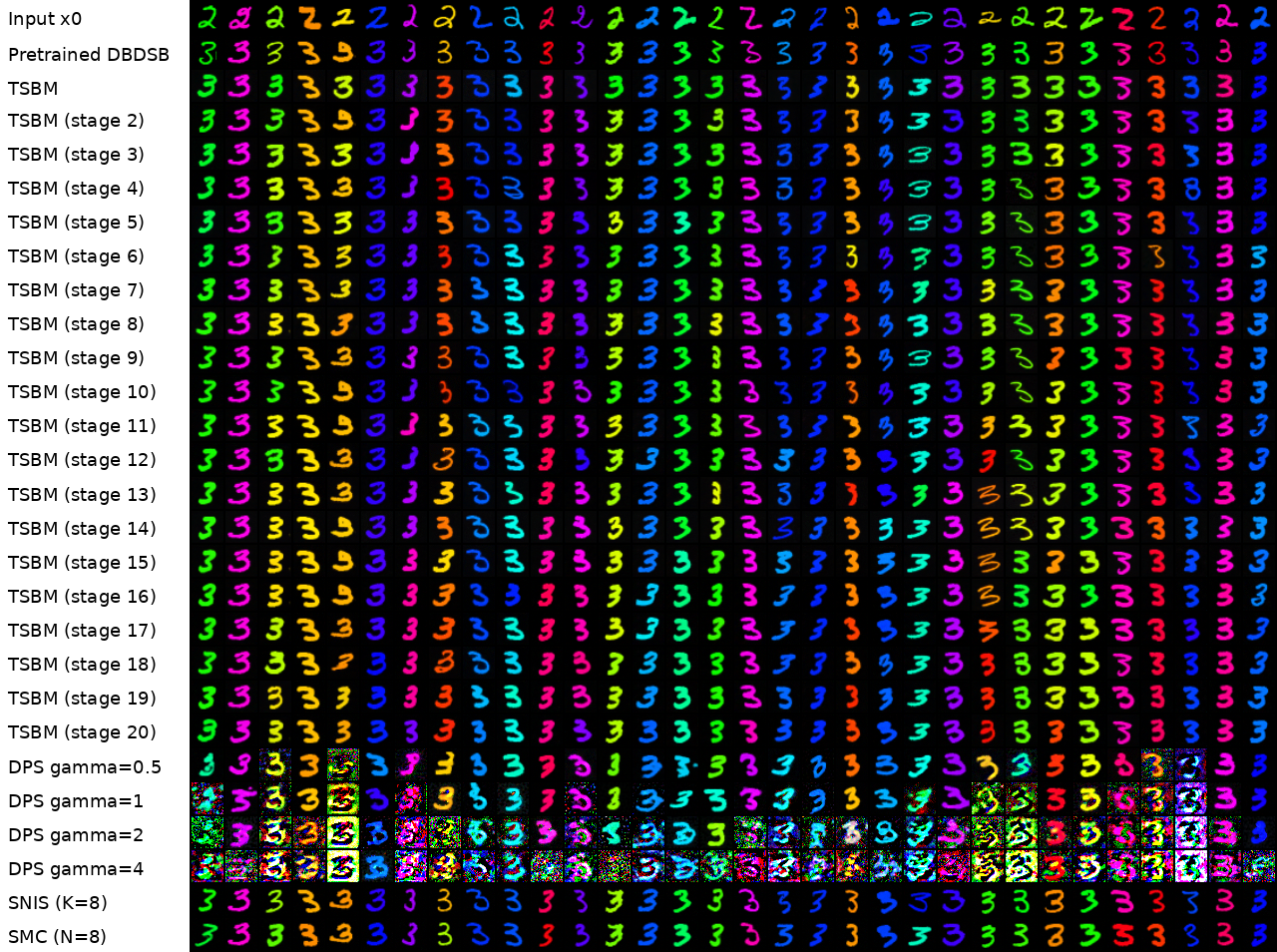} &
\includegraphics[
    width=0.295\linewidth,
    trim=192bp 546bp 884bp 34bp,
    clip
]{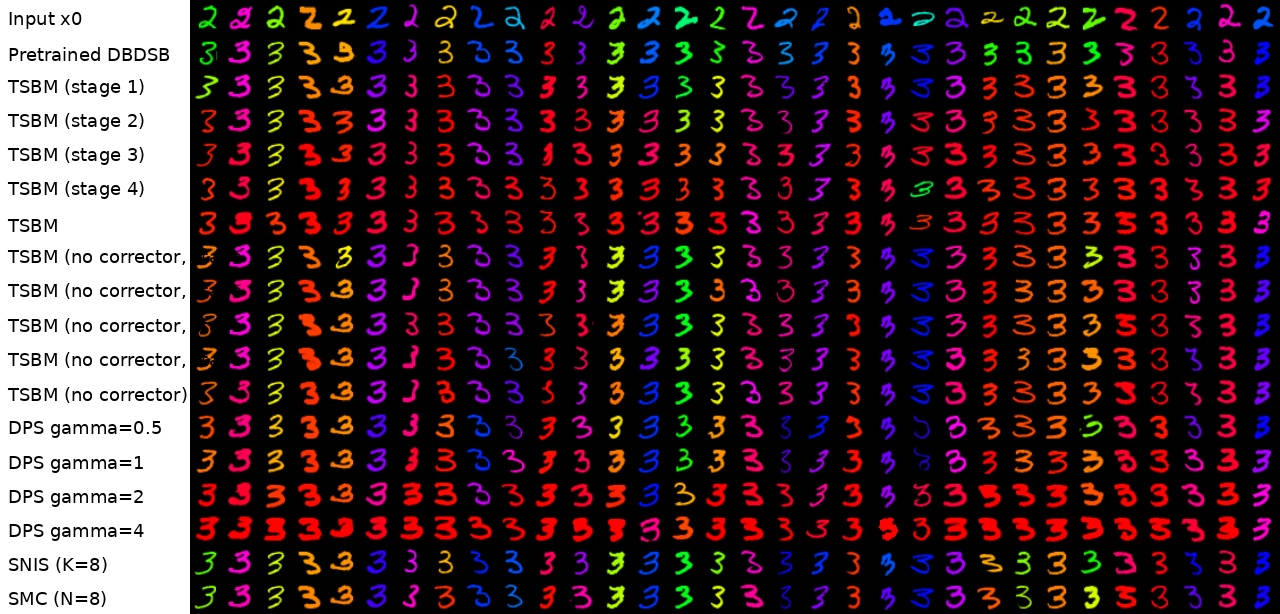}
\\[-2pt]

\textbf{TSBM (ours)} &
\includegraphics[
    width=0.295\linewidth,
    trim=192bp 240bp 884bp 34bp,
    clip
]{figures/colored_mnist_thin_target365_source.png} &
\includegraphics[
    width=0.295\linewidth,
    trim=192bp 716bp 884bp 170bp,
    clip
]{figures/colored_mnist_thick_target7_all_stages_source.png} &
\includegraphics[
    width=0.295\linewidth,
    trim=192bp 240bp 884bp 34bp,
    clip
]{figures/colored_mnist_red_chroma_target_source.png}
\\[-2pt]

\textbf{DPS 1} &
\includegraphics[
    width=0.295\linewidth,
    trim=192bp 138bp 884bp 136bp,
    clip
]{figures/colored_mnist_thin_target365_source.png} &
\includegraphics[
    width=0.295\linewidth,
    trim=192bp 138bp 884bp 748bp,
    clip
]{figures/colored_mnist_thick_target7_all_stages_source.png} &
\includegraphics[
    width=0.295\linewidth,
    trim=192bp 138bp 884bp 136bp,
    clip
]{figures/colored_mnist_red_chroma_target_source.png}
\\[-2pt]

\textbf{CondSMC} &
\includegraphics[
    width=0.295\linewidth,
    trim=192bp 2bp 884bp 272bp,
    clip
]{figures/colored_mnist_thin_target365_source.png} &
\includegraphics[
    width=0.295\linewidth,
    trim=192bp 2bp 884bp 884bp,
    clip
]{figures/colored_mnist_thick_target7_all_stages_source.png} &
\includegraphics[
    width=0.295\linewidth,
    trim=192bp 2bp 884bp 272bp,
    clip
]{figures/colored_mnist_red_chroma_target_source.png}
\\[-2pt]

&
\subcaptionbox{
    Thin stroke
    \label{fig:colored-mnist-thin-q10}
}[0.25\linewidth]{} &
\subcaptionbox{
    Thick stroke
    \label{fig:colored-mnist-thick-q90}
}[0.25\linewidth]{} &
\subcaptionbox{
    Red chroma
    \label{fig:colored-mnist-red-chroma}
}[0.25\linewidth]{}

\end{tabular}
\vspace{-3mm}
\caption{\small
\textbf{Qualitative results on Colored MNIST.}
Rows show the source input, pretrained DSBM, TSBM, DPS with
$\gamma=1$, and CondSMC with $N=8$ particles for the thin-stroke,
thick-stroke, and red-chroma targets.
Comparisons with additional baselines are in
Appendix~\ref{app:additional-qualitative-results}.
}
\vspace{-2mm}
\label{fig:colored-mnist-qualitative}
\end{figure}

\begin{figure}[!t]
\includegraphics[
    width=\linewidth
]{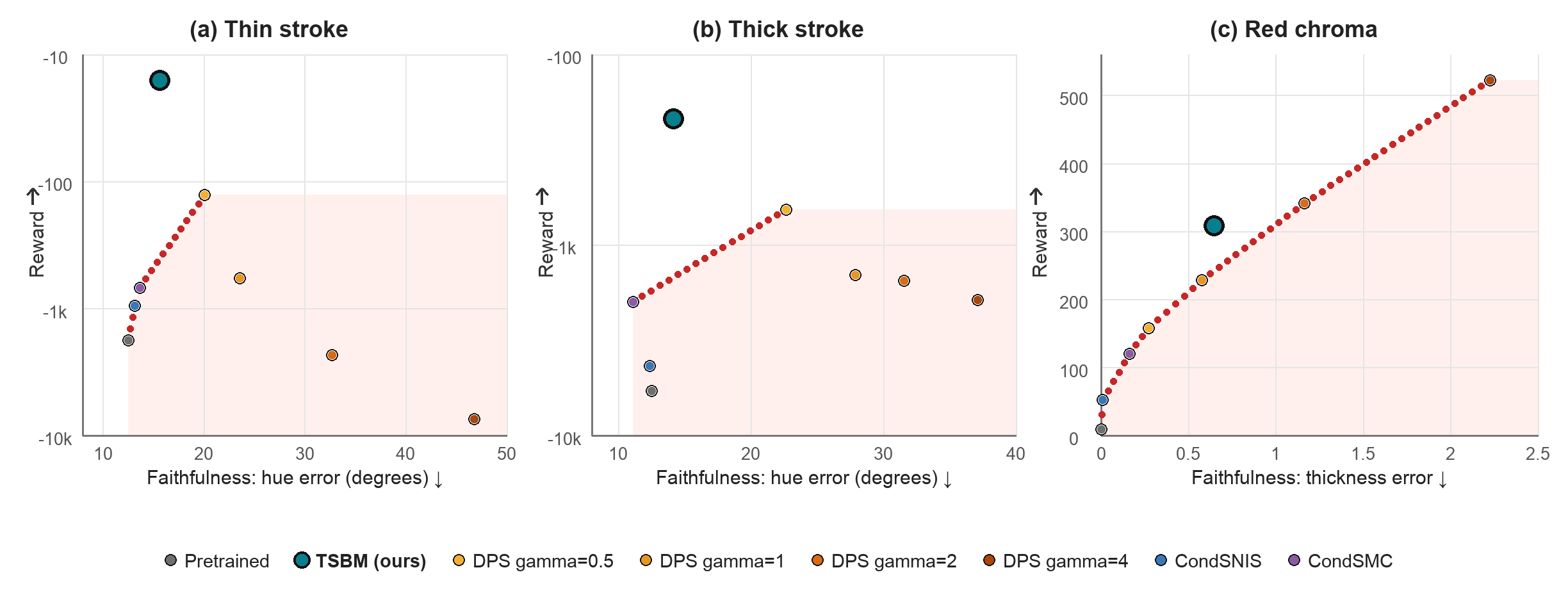}

\vspace{-3mm}
\caption{
\small
\textbf{Quantitative analysis on Colored MNIST.}
Reward $\uparrow$ is the mean unscaled terminal reward $r(x)$. The two stroke panels use logarithmic spacing by reward magnitude.
Faithfulness $\downarrow$ is foreground hue error relative to the input for thin and thick strokes, and mean thickness deviation from the input for red chroma.
The red dotted lines connect the non-TSBM Pareto points, while pale red shading marks values below them.
}
\vspace{-4mm}
\label{fig:colored-mnist-tradeoffs}

\end{figure}

\paragraph{Reward steering and diffusion bridges.} While reward steering is well studied for standard diffusion models,
extending it to diffusion bridges, including Schrödinger bridges,
requires accounting for source target dependence.
In particular, conditional tilting,
$q_{1|0}(x_1|x_0)\propto
p^P_{1|0}(x_1|x_0)\exp(r(x_1))$,
generally does not produce the desired terminal marginal
$q_1(x_1)\propto p^P_1(x_1)\exp(r(x_1))$
when the source marginal is fixed. For bridges trained on paired data \citep{zhou2024ddbm,liu2023i2sb},
CDDB applies inference-time data-consistency guidance with a
DPS-like variant \citep{chung2023cddb,chung2023dps}, \cite{he2026rne} shows the practical applicability of SMC-like  \cite{delmoral2006} population steering, but doesn't evaluate it,
while PalSB uses physics-informed finetuning through truncated
trajectory differentiation \citep{li2025palsb}. 
Schrödinger Bridges were also reward aligned to learn Boltzmann samplers
with evaluable target energies \citep{liu2025asbs,guo2026dasbs,havens2025adjointsampling},
a different setting from reward tilting of a pretrained
bridge whose terminal density is implicit. These works do not establish the joint guarantee considered in our work: 
a prescribed terminal \textit{marginal} reward tilt with preservation of the source marginal when finetuning a general pretrained Schrödinger Bridge model.

\vspace{-3mm}
\section{Experiments}
\label{sec:experiments}
\vspace{-3mm}

In this section we experimentally evaluate the \tsbm{}, as the Schrodinger Bridges reward finetuning method. In all the cases we do start from Schrodinger Bridge trained via DSBM method \cite{shi2023dsbm}. We test \tsbm{} on unpaired image translation problems: 1) Colored MNIST dataset image translation between digit "2" to digit "3" with reward tilt towards digit thickness and color; 2) Celeba dataset male to female image translation with reward tilts toward different image attributes: "natural toothy smile", "heavy makeup" and "elderly woman" with reward functions provided by ImageReward \cite{xu2023imagereward} with corresponding prompts. In addition during training the reward functions are scaled linearly by \textit{strength} parameter $\lambda$, i.e., $\lambda * r(x)$. Additional experimental details are presented in Appendix~\ref{app:exp-details} and more experiments are presented in Appendix~\ref{app:additional-experiments}.

\vspace{-2mm}
\paragraph{Competitors.} As was noted in the \wasyparagraph~\ref{sec:related-work} the reward finetuning of Schrodinger Bridge models is underexplored and has lack of established methods. However, there are approximate methods for \textit{tilting the conditional distribution}, such as CDDB \cite{chung2023cddb}, which we call just \textbf{DPS} in the experimental section and test with different \textit{strength} parameter $\gamma$. Additionally, we test the conditional population based inference time adaptation methods for \textit{tilting the conditional distribution}: Sequential Monte Carlo (\textbf{CondSMC}) \cite{delmoral2006} and Self Normalized Importance Sampling (\textbf{CondSNIS}). Competitor algorithms in particular are described in Appendix~\ref{app:competitors}.

\begingroup
\newcommand{\celebSmileGrid}{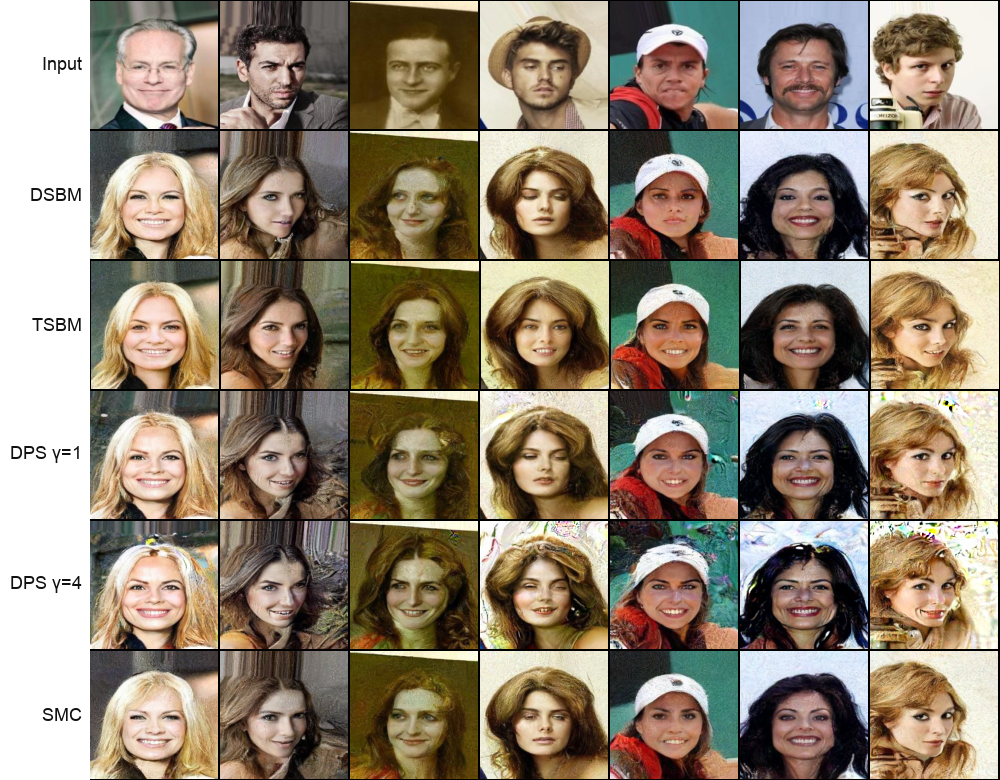}
\newcommand{\celebMakeupGrid}{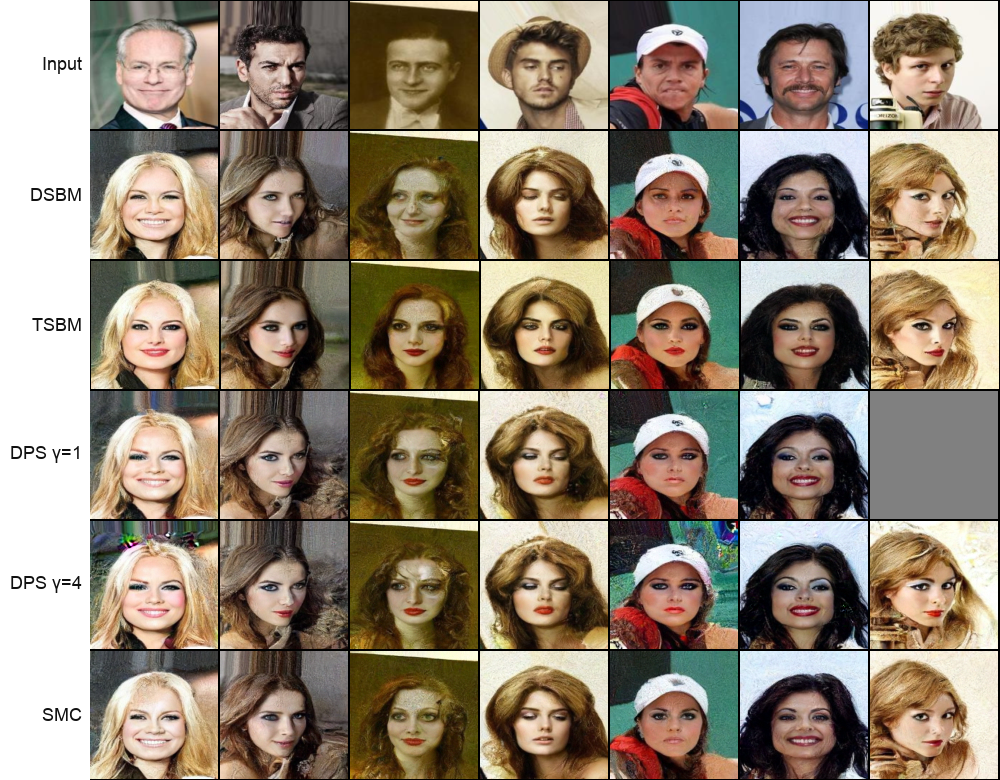}
\newcommand{\celebElderlyGrid}{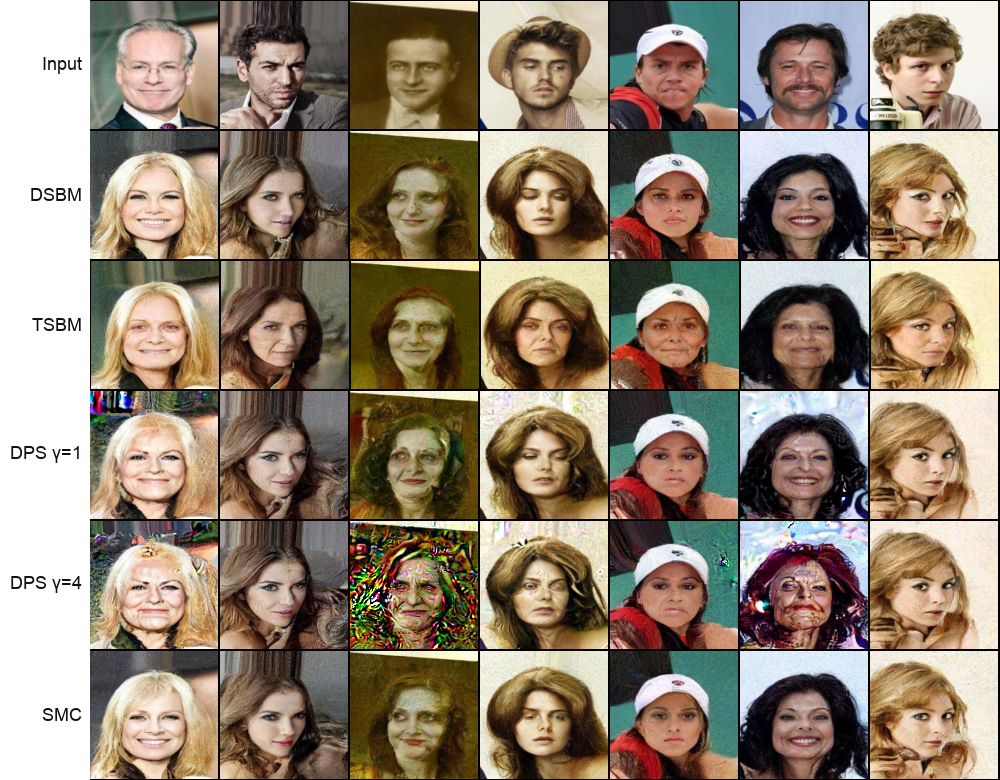}
\newcommand{\celebFace}[6][0]{\ifnum#1=1\relax
    \begin{tikzpicture}[baseline=(face.base)]
      \node[inner sep=0pt,outer sep=0pt] (face) {\includegraphics[
        width=0.071\linewidth,trim=#3bp #4bp #5bp #6bp,clip]{#2}};
      \draw[red,line width=0.8pt,overlay]
        ([xshift=0.4pt,yshift=0.4pt]face.south west) rectangle
        ([xshift=-0.4pt,yshift=-0.4pt]face.north east);
    \end{tikzpicture}\else
    \includegraphics[width=0.071\linewidth,trim=#3bp #4bp #5bp #6bp,clip]{#2}\fi
}
\newcommand{\celebMethodHead}[1]{\parbox[c]{0.071\linewidth}{\centering\scriptsize\bfseries #1}}
\newcommand{\celebSourceRow}[5]{\celebFace{\celebSmileGrid}{#1}{650}{#2}{0} &
  \celebFace{\celebSmileGrid}{#1}{520}{#2}{130} &
  \celebFace{\celebSmileGrid}{#1}{390}{#2}{260} &
  \celebFace{\celebSmileGrid}{#1}{260}{#2}{390} &
  \celebFace[#3]{\celebSmileGrid}{#1}{130}{#2}{520} &
  \celebFace{\celebSmileGrid}{#1}{0}{#2}{650} &
  \celebFace{\celebMakeupGrid}{#1}{390}{#2}{260} &
  \celebFace[#4]{\celebMakeupGrid}{#1}{260}{#2}{390} &
  \celebFace{\celebMakeupGrid}{#1}{130}{#2}{520} &
  \celebFace{\celebMakeupGrid}{#1}{0}{#2}{650} &
  \celebFace{\celebElderlyGrid}{#1}{390}{#2}{260} &
  \celebFace{\celebElderlyGrid}{#1}{260}{#2}{390} &
  \celebFace[#5]{\celebElderlyGrid}{#1}{130}{#2}{520} &
  \celebFace{\celebElderlyGrid}{#1}{0}{#2}{650}}

\begin{figure}[!t]
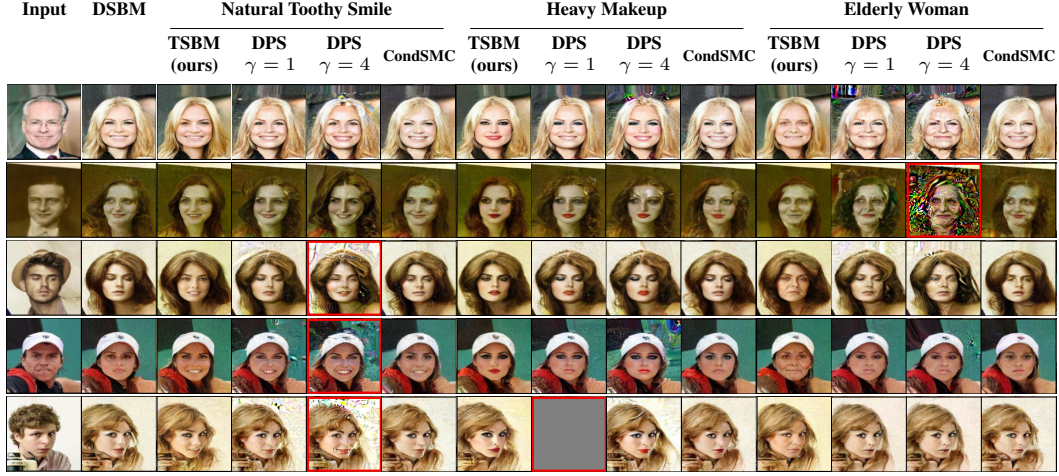
\label{fig:celeba-main}
\vspace{-3mm}
\centering
\setlength{\tabcolsep}{0pt}
\renewcommand{\arraystretch}{0.96}
\begin{tabular}{@{}*{14}{c}@{}}
\multicolumn{1}{c}{\scriptsize\bfseries Input}
& \multicolumn{1}{c}{\scriptsize\bfseries DSBM}
& \multicolumn{4}{c}{\scriptsize\bfseries Natural Toothy Smile}
& \multicolumn{4}{c}{\scriptsize\bfseries Heavy Makeup}
& \multicolumn{4}{c}{\scriptsize\bfseries Elderly Woman} \\
\cmidrule(lr){3-6}\cmidrule(lr){7-10}\cmidrule(lr){11-14}
\celebMethodHead{} & \celebMethodHead{}
& \celebMethodHead{TSBM (ours)} & \celebMethodHead{DPS\\$\gamma=1$} & \celebMethodHead{DPS\\$\gamma=4$} & \celebMethodHead{\fontsize{6}{7}\selectfont CondSMC}
& \celebMethodHead{TSBM (ours)} & \celebMethodHead{DPS\\$\gamma=1$} & \celebMethodHead{DPS\\$\gamma=4$} & \celebMethodHead{\fontsize{6}{7}\selectfont CondSMC}
& \celebMethodHead{TSBM (ours)} & \celebMethodHead{DPS\\$\gamma=1$} & \celebMethodHead{DPS\\$\gamma=4$} & \celebMethodHead{\fontsize{6}{7}\selectfont CondSMC} \\
\noalign{\vskip 4pt}
\celebSourceRow{90}{780}{0}{0}{0} \\[-2pt]
\celebSourceRow{350}{520}{0}{0}{1} \\[-2pt]
\celebSourceRow{480}{390}{1}{0}{0} \\[-2pt]
\celebSourceRow{610}{260}{1}{0}{0} \\[-2pt]
\celebSourceRow{870}{0}{1}{1}{0}
\end{tabular}
\vspace{-2mm}

\caption{
\small
\textbf{Qualitative results for CelebA experiment.} Each row uses the same source across the pretrained DSBM and three CelebA prompt groups. Within each group, we compare TSBM, DPS with $\gamma\in\{1,4\}$, and CondSMC. Red borders highlight selected outputs with visible artifacts.
}
\vspace{-5mm}
\label{fig:celeba-imagereward-grouped}
\end{figure}
\endgroup

\vspace{-2mm}
\subsection{Colored MNIST unpaired translation attributes manipulation}
\vspace{-2mm}

\paragraph{Setup.}  We test the \tsbm{} on images, starting from Colored MNIST digit "2" to digit "3" unpaired domain translation. We take the Colored MNIST translation setup from \cite{gushchin2024adversarial}. DSBM pretrain si trained by $20$ stages with $\sigma_t=1$ SB volatility coefficient, while \tsbm{} is trained for $5$ stages with strength $\lambda=1000$. The reward functions are: 1) "red chroma" which favors the red rgb channel vs others 2) "thin stroke" and "thick stroke", which do calculate the approximate thickness of a digit by computing the ratio between strokes area and foreground perimeter to then penalize the thickness different from $3.65$ and $7$ for "thin stroke" and "thick stroke" correspondingly. See Appendix~\ref{app:reward-functions} for more information of reward functions, Appendix~\ref{app:additional-experiments} for ablation studies on number of TSBM stages and corrector necessity and Appendix~\ref{app:additional-qualitative-results} for qualitative results.

\vspace{-2mm}
\paragraph{Results.} Figure~\ref{fig:colored-mnist-qualitative} compares outputs from TSBM
and the baselines with the source inputs and pretrained DSBM. TSBM produces thinner, thicker, or redder digits as specified
by the reward function, while largely preserving attributes unrelated
to the reward: thin- and thick-stroke outputs retain the source color,
whereas red-chroma outputs retain approximately the source thickness.
The quantitative results support this balance between reward
optimization and faithfulness to the source. In comparison, DPS
introduces visible artifacts, while CondSMC makes limited changes
to the targeted attributes. Figure~\ref{fig:colored-mnist-tradeoffs} reports quantitative results, i.e., \textit{reward} and \textit{faithfulness}, and one can see that \textbf{TSBM improves upon the Pareto frontier}. These results demonstrate that
\textbf{TSBM} extends to image translation and achieves the most
favorable \textbf{trade-off between reward optimization and
source faithfulness} among the evaluated methods.

\vspace{-2mm}
\subsection{CelebA unpaired translation attributes manipulation}
\vspace{-2mm}

\paragraph{Setup.} Next, we test \tsbm{} on unpaired male-to-female CelebA translation at $128\times128$ resolution. Starting from a pretrain DSBM trained for $20$ stages with $\sigma_t=1$, we train a \tsbm{} for $5$ stages. We use a frozen ImageReward model~\citep{xu2023imagereward}, which is normalized w.r.t. our CelebA data, to define our \textit{train reward} for optimization and baselines sample generation with the corresponding prompts: \emph{Natural Toothy Smile}, \emph{Heavy Makeup}, and \emph{Elderly Woman}. While PickScore~\cite{kirstain2023pickapic} with the same prompts is used for evaluation for exposing the potential \textit{reward hacking}. The reward function details can be seen at Appendix~\ref{app:reward-celeba}. In addition we report CLIP-IQA (naturalness) metric for assessment of resulting image quality. All the methods use reward strength $\lambda=30$ and 100 NFE for evaluation. See Appendix~\ref{app:additional-qualitative-results} for additional qualitative results.

\vspace{-2mm}
\paragraph{Results.} Figure~\ref{fig:celeba-imagereward-grouped} shows the samples on the same five sources for each prompt. TSBM follows the requested changes in smile, makeup, or apparent age, while other methods may struggle. Stronger DPS guidance shows itself also succesfull and slighly weaker than \tsbm{}, but introduces visible artifacts in some examples, while CondSMC produces weaker tilt. Table~\ref{tab:celeba-imagereward-grouped} quantifies the reward-faithfulness trade-off: TSBM improves ImageReward over the pretrained DSBM for all three train reward, but not as strongly as DPS.
However, TSBM obtains the \textbf{highest eval PickScore reward} in each comparison, which suggests that \textbf{TSBM results generalize well}. While DPS huge gains in train ImageReward reward does not result in visual artifacts and modest eval reward growth, which suggests that DPS is doing rather \textit{reward hacking}. TSBM also has the highest CLIP-IQA for the smile and elderly prompts, while the heavy-makeup result is less favorable and on contrary lowers CLIP-IQA, which may indicate that heavy make up tilting is not doesn't improve CLIP-IQA in general. In addition, \tsbm{} is faster than baselines on inference, see Appendix~\ref{app:celeba-inference-cost}. Thus, \textbf{TSBM improves prompt alignment, which generalizes well and keeps images visually high-fidelity}.

\begin{table}[!t]
\centering
\vspace{-2mm}
\caption{\textbf{Quantitative results for CelebA experiment}. Means are over 1024 fixed inputs per prompt, with TSBM evaluated after stage 5. ImageReward and PickScore use translation, cropping, and horizontal-flip augmentation; CLIP-IQA measures naturalness. All metrics are higher-is-better; bold marks the largest mean. ($^\ast$) indicates the appearance of visual artifacts.}
\label{tab:celeba-imagereward-grouped}

\scriptsize
\setlength{\tabcolsep}{2.5pt}
\renewcommand{\arraystretch}{0.94}

\resizebox{\textwidth}{!}{\begin{tabular}{@{}l*{3}{rrr}@{}}
\toprule
& \multicolumn{3}{c}{\textbf{Natural Toothy Smile}}
& \multicolumn{3}{c}{\textbf{Heavy Makeup}}
& \multicolumn{3}{c}{\textbf{Elderly Woman}} \\
\cmidrule(lr){2-4}
\cmidrule(lr){5-7}
\cmidrule(lr){8-10}

Method
& \shortstack{ImageReward\\(train)$\uparrow$}
& \shortstack{PickScore\\(eval)$\uparrow$}
& \shortstack{CLIP-IQA\\$\uparrow$}
& \shortstack{ImageReward\\(train)$\uparrow$}
& \shortstack{PickScore\\(eval)$\uparrow$}
& \shortstack{CLIP-IQA\\$\uparrow$}
& \shortstack{ImageReward\\(train)$\uparrow$}
& \shortstack{PickScore\\(eval)$\uparrow$}
& \shortstack{CLIP-IQA\\$\uparrow$} \\
\midrule

Pretrained
& -1.2348 & 17.6109 & 0.5950
& -1.4054 & 17.3777 & \textbf{0.5950}
& -1.4960 & 18.1592 & 0.5950 \\
\midrule

TSBM (ours)
& 0.1398 & \textbf{18.2345} & \textbf{0.6595}
& 0.9146 & \textbf{18.0390} & 0.3508
& -0.0204 & \textbf{18.8200} & \textbf{0.6340} \\
\midrule

DPS$^\ast$ $\gamma=1$
& 0.0622 & 17.7347 & 0.6192
& 0.8576 & 17.5803 & 0.3789
& -0.7328 & 18.2559 & 0.5623 \\

DPS$^\ast$ $\gamma=2$
& 0.5993 & 17.6831 & 0.6312
& 1.2507 & 17.6124 & 0.3401
& -0.2527 & 18.3329 & 0.5218 \\

DPS$^\ast$ $\gamma=4$
& \textbf{0.9783} & 17.4688 & 0.6164
& \textbf{1.5151} & 17.6236 & 0.3133
& \textbf{0.2239} & 18.3346 & 0.4618 \\

CondSMC
& -0.5728 & 17.8375 & 0.6107
& -0.1525 & 17.5395 & 0.4805
& -1.2800 & 18.1762 & 0.5927 \\
\bottomrule
\end{tabular}}
\vspace{-4mm}
\end{table}

\vspace{-2mm}
 
\vspace{-2mm}
\section{Discussion}
\label{sec:limitations}
\vspace{-2mm}

\paragraph{Potential impact.} For the problem of reward finetuning of pretrained Schrödinger Bridges we introduce \tsbm{}, the first practical algorithm for this problem, which alternates controller and
corrector matching. On Colored MNIST and CelebA, \tsbm{} steers attributes
in unpaired image translation while largely preserving source content.
These results suggest that pretrained bridges can be adapted to new
preferences without training a new bridge from scratch.

\vspace{-2mm}
\paragraph{Limitations and future work.} We assume exact optimization and validity of all theoretical assumptions, which may fail in practice. The method also requires an accessible endpoint transition score and differentiable reward, motivating extensions to non-differentiable rewards \cite{bergmeister2026ram}. Possible extensions include adapting unpaired scientific translation models to physical constraints \cite{li2025palsb}, image aesthetics improvement \cite{domingoenrich2025} or counterfactual explanations \cite{jeanneret2022diffusion}. Further methodological research includes other training objectives, such as other adjoint algorithms \cite{domingoenrich2024socm, vargas2023dds} or log variance losses \cite{tamogashev2026data}. Training \tsbm{} requires repeated bridge simulations and adjoint ODE solves, making finetuning computationally demanding, while distillation-based variant \citep{gushchin2025inverse} may offer a way to reduce this cost. Recent work on discrete state space reward optimization and adjoint matching \citep{wang2025drakes,so2026dam,guo2026dasbs} also suggests extending \tsbm{} to discrete state spaces.

\bibliography{iclr2027_conference}
\bibliographystyle{iclr2027_conference}

\clearpage
\appendix
\addtocontents{toc}{\protect\setcounter{tocdepth}{2}}
\renewcommand{\contentsname}{Appendix Contents}
\tableofcontents

\section{Proofs}
\label{app:proofs}

\subsection{Proof of Proposition~\ref{prop:change-reference}}
\label{app:change-reference}

\begin{proof}
Write $(\varphi_t^{\mathrm{old}},\widehat\varphi_t^{\mathrm{old}})$ for the potentials of $P$ relative to $R$ in Eq~\ref{eq:schrodinger-potentials}. Under Assumption~\ref{assumptions}, \citet[Proposition~2.3]{lLeonard2014}, together with the endpoint factorization and integrability results of \citet[Theorem~2.1 and Lemma~2.21]{nutz2022entropic}, gives
\[
    \frac{\mathrm{d}P}{\mathrm{d}R}
    =\frac{\varphi_1^{\mathrm{old}}(X_1)}{\varphi_0^{\mathrm{old}}(X_0)}>0,
\qquad
    \int |\log\varphi_i^{\mathrm{old}}(x)|p_i(x)\,\mathrm{d}x<\infty,\quad i\in\{0,1\}.
\]
Here $\mathrm{d}P/\mathrm{d}R$ denotes the path-law density, and Eq~\ref{eq:schrodinger-potentials} gives
\[
    \varphi_0^{\mathrm{old}}\widehat\varphi_0^{\mathrm{old}}=p_0
    \quad\Longrightarrow\quad
    \frac{\widehat\varphi_0^{\mathrm{old}}}{p_0}
    =\frac{1}{\varphi_0^{\mathrm{old}}}.
\]
Since $r$ is bounded above, for some finite $M$,
\begin{align*}
    0<w(x):=\frac{p_1^r(x)}{p_1(x)}&\leq M,\\
    \int|\log\varphi_1^{\mathrm{old}}(x)|p_1^r(x)\,\mathrm{d}x
    &\leq M\int|\log\varphi_1^{\mathrm{old}}(x)|p_1(x)\,\mathrm{d}x<\infty.
\end{align*}
The boundedness of $w$ and $w|\log w|$, together with Assumption~\ref{assumptions}, gives
\begin{align*}
    \KL(p_0\otimes p_1^r\Vert p_{0,1}^R)
    &=\iint w(x_1)p_0(x_0)p_1(x_1)
      \log\frac{p_0(x_0)p_1(x_1)}{p_{0,1}^R(x_0,x_1)}\,\mathrm{d}x_0\mathrm{d}x_1\\
    &\quad+\int w(x)\log w(x)p_1(x)\,\mathrm{d}x<\infty.
\end{align*}
Applying \citet[Theorem~2.1]{nutz2022entropic} and \citet[Proposition~2.3]{lLeonard2014} to $(p_0,p_1^r)$ yields the unique tilted bridge $\Prwd$, with
\[
    \KL(\Prwd\Vert R)<\infty,
    \qquad
    \frac{\mathrm{d}\Prwd}{\mathrm{d}R}
    =\frac{\widehat\varphi_0^r(X_0)}{p_0(X_0)}\varphi_1^r(X_1)>0,
\]
where $(\varphi_t^r,\widehat\varphi_t^r)$ are its potentials relative to $R$.

For any path law $Q$ with marginals $p_0,p_1^r$ that is absolutely continuous with respect to $R$, the chain rule gives
\begin{equation}
\begin{aligned}
    \KL(Q\Vert P)
    &=\KL(Q\Vert R)-\mathbb{E}_Q\log\frac{\mathrm{d}P}{\mathrm{d}R}\\
    &=\KL(Q\Vert R)+\int\log\varphi_0^{\mathrm{old}}(x)p_0(x)\,\mathrm{d}x\\
    &\quad-\int\log\varphi_1^{\mathrm{old}}(x)p_1^r(x)\,\mathrm{d}x\\
    &=\KL(Q\Vert R)+C(p_0,p_1,r),
\end{aligned}
\label{eq:reference-kl-shift}
\end{equation}

where $C(p_0, p_1, r)$ is constant, which does not depend on $Q$. Therefore both Schrödinger Bridge Problmes do have the same unique minimizer $\Prwd$:

\begin{equation}
    \SB_P(p_0,\nur)=\SB_R(p_0,\nur)=\Prwd.
\end{equation}

\end{proof}

An analogous change-of-reference argument appears in \citet[equation~(2.6)]{lLeonard2014}: reweighting the reference by endpoint factors shifts the KL objective by terms that are constant when the marginals are fixed.

\subsection{Proof of Proposition~\ref{prop:relative-log-potentials}}
\label{app:relative-potentials}

\begin{proof}
Let $(\varphi_t^r,\widehat\varphi_t^r)$ and $(\psi_t^r,\widehat\psi_t^r)$ be the potentials of $\Prwd$ relative to $R$ and $P$, respectively. The Eq~\ref{eq:schrodinger-potentials} and Proposition~\ref{prop:change-reference} give:
\begin{align*}
    \frac{\mathrm{d}P}{\mathrm{d}R}
    &=\frac{\widehat\varphi_0^{\mathrm{old}}(X_0)}{p_0(X_0)}
      \varphi_1^{\mathrm{old}}(X_1),
    &\frac{\mathrm{d}\Prwd}{\mathrm{d}R}
    &=\frac{\widehat\varphi_0^r(X_0)}{p_0(X_0)}\varphi_1^r(X_1),\\
    \frac{\mathrm{d}\Prwd}{\mathrm{d}P}
    &=\frac{\widehat\varphi_0^r(X_0)}{\widehat\varphi_0^{\mathrm{old}}(X_0)}
      \frac{\varphi_1^r(X_1)}{\varphi_1^{\mathrm{old}}(X_1)}
    =\frac{\widehat\psi_0^r(X_0)}{p_0(X_0)}\psi_1^r(X_1).
\end{align*}
Since $p_{0,1}^P$ and $p_0\otimes p_1$ have the same null sets, separation of the endpoint factors implies, for some constant $c>0$,
\[
    \psi_1^r(x)=c\,\frac{\varphi_1^r(x)}{\varphi_1^{\mathrm{old}}(x)}
    \qquad p_1^r\text{-almost everywhere}.
\]
Using the terminal marginal identities,
\[
    \psi_1^r\widehat\psi_1^r=p_1^r,
    \qquad \varphi_1^r\widehat\varphi_1^r=p_1^r,
    \qquad \varphi_1^{\mathrm{old}}\widehat\varphi_1^{\mathrm{old}}=p_1,
\]
we obtain
\begin{align*}
    f_{\mathrm{rel}}^r
    &=-\log\psi_1^r
      =-\log\varphi_1^r+\log\varphi_1^{\mathrm{old}}-\log c\\
    &=\log\widehat\varphi_1^r-\log\widehat\varphi_1^{\mathrm{old}}
      +\log\frac{p_1}{p_1^r}-\log c\\
    &=\log\widehat\varphi_1^r-\log\widehat\varphi_1^{\mathrm{old}}-r+C,
\end{align*}
where the last equality uses $p_1^r\propto p_1e^r$ and $C$ absorbs the target and potential normalizations. This proves Eq~\ref{eq:relative-terminal-log-potential} $p_1^r$-almost everywhere. Under the stated smoothness conditions, continuity extends the identity throughout the open region and differentiation gives Eq~\ref{eq:relative-terminal-cost-gradient}.
\end{proof}

\subsection{Proof of Proposition~\ref{prop:relative-soc-problem}}
\label{app:relative-soc}

\begin{proof}
With the terminal terms evaluated at $X_1$, Proposition~\ref{prop:relative-log-potentials} shows that the objective in Eq~\ref{eq:relative-terminal-cost-control} differs by a control-independent constant from
\begin{equation}
    \mathbb{E}_{P^u}\!\left[
       \frac12\int_0^1\|u_t(X_t)\|^2\,\mathrm{d}t
       +\log\frac{\widehat\psi_1^r(X_1)}{p_1^r(X_1)}\right].
\label{eq:relative-soc-equivalent-objective}
\end{equation}
Applying Eq~\ref{eq:sb-terminal-cost-control} with reference $P$ and target $p_1^r$, over its admissible controls, then Proposition~\ref{prop:change-reference}, gives
\begin{equation}
    P^{u_r^\star}=\SB_P(p_0,p_1^r)=\SB_R(p_0,p_1^r).
\label{eq:relative-soc-optimal-bridge}
\end{equation}
\end{proof}

\subsection{Assumptions and controlled diffusions for convergence}
\label{app:half-bridge-setup}

We introduce the forward and backward controlled diffusions used in the proof of Theorem~\ref{thm:population}:
\begin{equation}
\begin{aligned}
    P^u:\quad \mathrm{d}X_t
    &=[\Fref_t(X_t)+\sigma_tu_t(X_t)]\,\mathrm{d}t
      +\sigma_t\,\mathrm{d}W_t, &&X_0\sim p_0,\\
    Q^v:\quad \mathrm{d}Y_s
    &=[-f_{1-s}(Y_s)+\sigma_{1-s}v_s(Y_s)]\,\mathrm{d}s
      +\sigma_{1-s}\,\mathrm{d}W_s, &&Y_0\sim p_1^r.
\end{aligned}
\label{eq:controlled-half-bridges}
\end{equation}
Here $s=1-t$ and $Y_s=X_{1-s}$. All density subscripts and KL divergences use forward time $t$, so $p_0^{P^u}=p_0$ and $p_1^{Q^v}=p_1^r$. All optimization variables below are controls.

We retain the reference conditions in the notation paragraph and Assumption~\ref{assumptions}. The following additional assumptions specify the setting of the two half-bridge results and the proof of Theorem~\ref{thm:population}. The convergence of controls in that theorem is understood in the integrated mean-square sense of Eq~\ref{eq:convergence-control-limit}.

\begin{assumption}[Population matching and convergence]
\label{ass:population-convergence}
\leavevmode
\begin{enumerate}
\item \emph{Endpoint moments and reward.} For some $\lambda>0$ and
finite constant $r_{\max}$,
\begin{equation}
    \int e^{\lambda\|x\|^2}
    \bigl(p_0(x)+p_1(x)\bigr)\,\mathrm{d}x<\infty,
    \qquad
    r\in C^1(\mathbb{R}^d),
    \qquad
    r(x)\leq r_{\max}.
\label{eq:ass-endpoint-reward}
\end{equation}
\item \emph{Reference transition scores.} The density $p_{1\mid0}^R$ is positive and continuously differentiable in both endpoints, and, for a constant $C<\infty$ independent of the endpoints,
\begin{equation}
    \|\nabla_{x_0}\log p_{1\mid0}^R(x_1\mid x_0)\|
    +\|\nabla_{x_1}\log p_{1\mid0}^R(x_1\mid x_0)\|
    \leq C(1+\|x_0\|+\|x_1\|).
\label{eq:ass-reference-scores}
\end{equation}
\item \emph{Admissibility and change of drift.} The controlled SDEs below have unique nonexplosive solutions. The admissible control classes contain the half-bridge solutions and the optimal relative controller $u^*$. For each comparison of $P^u$ with $P^{\widetilde u}$ used below,
\begin{equation}
\begin{aligned}
    \mathbb{E}_{P^{\widetilde u}}\exp\!\left(
      \frac12\int_0^1\|u_t(X_t)-\widetilde u_t(X_t)\|^2\,\mathrm{d}t\right)&<\infty,\\
    \mathbb{E}_{P^u}\int_0^1
      \|u_t(X_t)-\widetilde u_t(X_t)\|^2\,\mathrm{d}t&<\infty.
\end{aligned}
\label{eq:ass-novikov-energy}
\end{equation}
These comparisons are $(u,\widetilde u)=(u,0)$ for the SOC candidates and $(u^*,u^{(k)})$ for the final control-energy identity. The analogous conditions are imposed on $R^u$ relative to $R$ when using the control formulation Eq~\ref{eq:sb}.
\item \emph{Differentiation and matching targets.} The pretrained drift $b_t$ is continuously differentiable in space, and the adjoint ODE below is well posed. The propagated potentials have the classical time and space derivatives used in their diffusion representations on $0<t<1$, with the endpoint gradients used below. For every $k\geq0$ and compact $K\subset\mathbb{R}^d$, there is a nonnegative integrable function $D_{k,K}$ such that
\begin{equation}
\begin{aligned}
    \sup_{x_1\in K}\|\nabla_{x_1}p_{1\mid0}^R(x_1\mid x_0)\|
       \widehat\varphi_0^{(k)}(x_0)&\leq D_{k,K}(x_0),\\
    \int D_{k,K}(x_0)\,\mathrm{d}x_0&<\infty.
\end{aligned}
\label{eq:ass-kernel-domination}
\end{equation}
Here $\widehat\varphi_0^{(k)}$ is the stage factor in Eq~\ref{eq:stage-controlled-potentials}. For the lean adjoint $\widetilde a^{(k)}$ in Eq~\ref{eq:appendix-population-updates},
\begin{equation}
    \mathbb{E}_{P^{u^{(k+1)}}}\!\left[
       \int_0^1\sigma_t^2\|\widetilde a_t^{(k)}\|^2\,\mathrm{d}t
       +\|\nabla_{x_1}\log p_{1\mid0}^R(X_1\mid X_0)\|^2\right]<\infty.
\label{eq:ass-matching-moments}
\end{equation}
\end{enumerate}
\end{assumption}

The domination and moment conditions justify the kernel differentiation and conditional regressions in the half-bridge proofs; Eq~\ref{eq:ass-novikov-energy} justifies their control-energy identities. The endpoint moments and score bound are used in the convergence proof.

\begin{samepage}
In particular, for $w:=p_1^r/p_1$, Assumption~\ref{assumptions} gives a finite $M$ such that

\begin{equation}
    0<w(x)\leq M,
    \qquad
    \int e^{\lambda\|x\|^2}p_1^r(x)\,\mathrm{d}x
      \leq M\int e^{\lambda\|x\|^2}p_1(x)\,\mathrm{d}x<\infty.
\label{eq:tilted-subgaussian-moment}
\end{equation}

\end{samepage}

For this derivation, index a complete stage by $(u^{(k)},h^{(k)})\mapsto(u^{(k+1)},h^{(k+1)})$, starting at $u^{(0)}=0$, $h^{(0)}=\hold$. Its exact population updates are
\begin{equation}
\begin{aligned}
    u_t^{(k+1)}(x)
    &=-\sigma_t\mathbb{E}_{P^{u^{(k+1)}}}
      [\widetilde a_t^{(k)}\mid X_t=x],\\
    \widetilde a_1^{(k)}
    &=h^{(k)}(X_1)-\hold(X_1)-\nabla r(X_1),
    \qquad
    -\frac{\mathrm{d}\widetilde a_t^{(k)}}{\mathrm{d}t}
      =\nabla_x\Fref_t(X_t)^\top\widetilde a_t^{(k)},\\
    h^{(k+1)}
    &\in\argmin_h\mathbb{E}_{P^{u^{(k+1)}}}
      \left\|h(X_1)-\nabla_{x_1}\log p_{1\mid0}^R(X_1\mid X_0)\right\|^2.
\end{aligned}
\label{eq:appendix-population-updates}
\end{equation}

For an exact stage controller $u^{(k)}$, write $(\varphi_t^{(k)},\widehat\varphi_t^{(k)})$ for its potentials relative to $R$, so that
\begin{equation}
\begin{aligned}
    \frac{\mathrm{d}P^{u^{(k)}}}{\mathrm{d}R}
    &=\frac{\widehat\varphi_0^{(k)}(X_0)}{p_0(X_0)}\varphi_1^{(k)}(X_1),
    &p_t^{P^{u^{(k)}}}&=\varphi_t^{(k)}\widehat\varphi_t^{(k)},\\
    \varphi_t^{(k)}(x)
    &=\int p_{1\mid t}^R(y\mid x)\varphi_1^{(k)}(y)\,\mathrm{d}y,
    &\widehat\varphi_0^{(k)}&=p_0/\varphi_0^{(k)},\\
    \widehat\varphi_t^{(k)}(x)
    &=\int p_{t\mid0}^R(x\mid y)\widehat\varphi_0^{(k)}(y)\,\mathrm{d}y,
    &\Fref_t+\sigma_tu_t^{(k)}&=f_t+\sigma_t^2\nabla\log\varphi_t^{(k)}.
\end{aligned}
\label{eq:stage-controlled-potentials}
\end{equation}
At $k=0$, these are the old potentials. The forward half-bridge proof below establishes this representation inductively at each subsequent stage; it does not assume that intermediate diffusions already have terminal marginal $p_1^r$.

Time reversal of $P^{u^{(k)}}$, as used in \citet[proof of Theorem~4.2]{liu2025asbs}, gives the backward control
\begin{equation}
\begin{aligned}
    v_s^{(k)}(x)
    &=\frac{f_{1-s}(x)-\Fref_{1-s}(x)}{\sigma_{1-s}}
      -u_{1-s}^{(k)}(x)
      +\sigma_{1-s}\nabla\log p_{1-s}^{P^{u^{(k)}}}(x)\\
    &=\sigma_{1-s}\nabla\log\widehat\varphi_{1-s}^{(k)}(x).
\end{aligned}
\label{eq:stage-backward-control}
\end{equation}
The diffusion $Q^{v^{(k)}}$ uses this control and starts from $p_1^r$ in reverse time. Therefore it preserves the conditional trajectories given $X_1$:
\begin{equation}
\begin{aligned}
    Q^{v^{(k)}}(\cdot\mid X_1)&=P^{u^{(k)}}(\cdot\mid X_1),\\
    \frac{\mathrm{d}Q^{v^{(k)}}}{\mathrm{d}P^{u^{(k)}}}
    &=\frac{p_1^r(X_1)}{p_1^{P^{u^{(k)}}}(X_1)},\\
    \frac{\mathrm{d}Q^{v^{(k)}}}{\mathrm{d}R}
    &=\frac{\widehat\varphi_0^{(k)}(X_0)}{p_0(X_0)}
      \frac{p_1^r(X_1)}{\widehat\varphi_1^{(k)}(X_1)}.
\end{aligned}
\label{eq:backward-controlled-density}
\end{equation}
Conditional expressions here refer to the same controlled diffusions, conditioned on their indicated endpoint.

\subsection{Controller Matching and the forward half bridge}
\label{app:controller-half-bridge}

\begin{proposition}[Controller Matching solves the forward half bridge]
\label{prop:controller-half-bridge}
Under Assumption~\ref{ass:population-convergence}, the exact population Controller Update in Eq~\ref{eq:appendix-population-updates} satisfies
\begin{equation}
    u^{(k+1)}\in\argmin_u\KL(P^u\Vert Q^{v^{(k)}}),
\label{eq:forward-half-bridge-control}
\end{equation}
where $v^{(k)}$ is defined by Eq~\ref{eq:stage-backward-control}.
\end{proposition}

\begin{proof}
At stage $k$, $h^{(k)}=\nabla\log\widehat\varphi_1^{(k)}$: this holds at initialization and follows from Proposition~\ref{prop:corrector-half-bridge} after each corrector update. Define the stage terminal cost
\[
    g^{(k)}:=\log\widehat\varphi_1^{(k)}
      -\log\widehat\varphi_1^{\mathrm{old}}-r,
    \qquad \nabla g^{(k)}=h^{(k)}-\hold-\nabla r.
\]
Using Eq~\ref{eq:backward-controlled-density}, $P^{u^{(0)}}=P$, and $p_1=\varphi_1^{\mathrm{old}}\widehat\varphi_1^{\mathrm{old}}$, we obtain for a candidate controller $u$,
\begin{align*}
    \KL(P^u\Vert Q^{v^{(k)}})
    &=\KL(P^u\Vert P^{u^{(0)}})
      +\mathbb{E}_{P^u}\log\frac{\mathrm{d}P^{u^{(0)}}/\mathrm{d}R}
                                      {\mathrm{d}Q^{v^{(k)}}/\mathrm{d}R}\\
    &=\KL(P^u\Vert P^{u^{(0)}})
      +\int p_0\log\frac{\widehat\varphi_0^{\mathrm{old}}}
                             {\widehat\varphi_0^{(k)}}\,\mathrm{d}x
      +\mathbb{E}_{P^u}\log\frac{\widehat\varphi_1^{(k)}(X_1)
                                      \varphi_1^{\mathrm{old}}(X_1)}{p_1^r(X_1)}\\
    &=\mathbb{E}_{P^u}\left[\frac12\int_0^1\|u_t(X_t)\|^2\,\mathrm{d}t
                    +g^{(k)}(X_1)\right]+C_k.
\end{align*}
Here $C_k$ is independent of $u$; the last equality uses $p_1^r\propto p_1e^r$ and the control-energy identity for the reference $P^{u^{(0)}}$. Adjoint Matching characterizes the SOC optimizer through the self-consistent update Eq~\ref{eq:appendix-population-updates}, with reference drift $\Fref$ and terminal gradient $\nabla g^{(k)}$ \citep[Appendix~A.1, equations~(26)--(27)]{liu2025asbs}. Thus
\[
    u^{(k+1)}\in\argmin_u\KL(P^u\Vert Q^{v^{(k)}}).
\]
To identify the resulting controlled diffusion explicitly, set
\begin{equation}
\begin{aligned}
    \varphi_1^{(k+1)}&:=\frac{p_1^r}{\widehat\varphi_1^{(k)}},
    &\varphi_t^{(k+1)}(x)&:=\int p_{1\mid t}^R(y\mid x)
                              \varphi_1^{(k+1)}(y)\,\mathrm{d}y,\\
    \widehat\varphi_0^{(k+1)}&:=\frac{p_0}{\varphi_0^{(k+1)}},
    &u_t^{(k+1)}&=\sigma_t\nabla\log
                    \frac{\varphi_t^{(k+1)}}{\varphi_t^{\mathrm{old}}}.
\end{aligned}
\label{eq:forward-control-potential-update}
\end{equation}
The SOC endpoint tilt then yields
\begin{align*}
    \frac{\mathrm{d}P^{u^{(k+1)}}}{\mathrm{d}R}
    &=\frac{\varphi_1^{(k+1)}(X_1)}{\varphi_0^{(k+1)}(X_0)},
    &p_0^{Q^{v^{(k)}}}&=\widehat\varphi_0^{(k)}\varphi_0^{(k+1)},\\
    \frac{\mathrm{d}P^{u^{(k+1)}}}{\mathrm{d}Q^{v^{(k)}}}
    &=\frac{p_0(X_0)}{p_0^{Q^{v^{(k)}}}(X_0)}.
\end{align*}
This proves the stage representation Eq~\ref{eq:stage-controlled-potentials} and identifies the forward half bridge: the initial marginal is replaced by $p_0$, with conditional trajectories given $X_0$ preserved.
\end{proof}

\subsection{Corrector Matching and the backward half bridge}
\label{app:corrector-half-bridge}

\begin{proposition}[Corrector Matching identifies the backward half bridge]
\label{prop:corrector-half-bridge}
Under Assumption~\ref{ass:population-convergence}, after the exact Controller Update, the backward control in Eq~\ref{eq:stage-backward-control} satisfies
\begin{equation}
    v^{(k+1)}\in\argmin_v\KL(Q^v\Vert P^{u^{(k+1)}}),
    \qquad v_0^{(k+1)}=\sigma_1h^{(k+1)},
\label{eq:backward-half-bridge-control}
\end{equation}
where $h^{(k+1)}$ is the exact Corrector Update in Eq~\ref{eq:appendix-population-updates}.
\end{proposition}

\begin{proof}
For the controller $u^{(k+1)}$ from Proposition~\ref{prop:controller-half-bridge}, the conditional density of $X_0$ given $X_1=x$ follows from Eq~\ref{eq:stage-controlled-potentials}:
\[
    p_{0\mid1}^{P^{u^{(k+1)}}}(x_0\mid x)
    =\frac{p_{1\mid0}^R(x\mid x_0)\widehat\varphi_0^{(k+1)}(x_0)}
           {\widehat\varphi_1^{(k+1)}(x)}.
\]
The moment bound Eq~\ref{eq:ass-matching-moments} and domination Eq~\ref{eq:ass-kernel-domination} give
\begin{align*}
    h^{(k+1)}(x)
    &=\mathbb{E}_{P^{u^{(k+1)}}}\!\left[
        \nabla_x\log p_{1\mid0}^R(x\mid X_0)\mid X_1=x\right]\\
    &=\frac{\int\nabla_xp_{1\mid0}^R(x\mid x_0)
                      \widehat\varphi_0^{(k+1)}(x_0)\,\mathrm{d}x_0}
            {\widehat\varphi_1^{(k+1)}(x)}\\
    &=\nabla\log\widehat\varphi_1^{(k+1)}(x)
      =\frac{v_0^{(k+1)}(x)}{\sigma_1}.
\end{align*}
The full backward control is given by Eq~\ref{eq:stage-backward-control}. For any candidate $v$, the KL chain rule at $X_1$ gives
\begin{align*}
    \KL(Q^v\Vert P^{u^{(k+1)}})
    &=\KL(p_1^r\Vert p_1^{P^{u^{(k+1)}}})\\
    &\quad+\int p_1^r(x)\,
       \KL\!\left(Q^v(\cdot\mid X_1=x)
            \Vert P^{u^{(k+1)}}(\cdot\mid X_1=x)\right)\,\mathrm{d}x\\
    &\geq\KL(p_1^r\Vert p_1^{P^{u^{(k+1)}}}).
\end{align*}
By Eq~\ref{eq:backward-controlled-density}, equality is attained by $v=v^{(k+1)}$, proving Eq~\ref{eq:backward-half-bridge-control}.
\end{proof}

\subsection{Proof of Theorem~\ref{thm:population}}
\label{app:population-convergence}

We prove the theorem under Assumption~\ref{ass:population-convergence}.

\begin{proof}
Let $u^*$ be the optimal relative controller, so that $P^{u^*}=\SB_P(p_0,p_1^r)$. The initialization gives
\begin{equation}
    P^{u^{(0)}}=P,
    \qquad v_s^{(0)}=\sigma_{1-s}\nabla\log\widehat\varphi_{1-s}^{\mathrm{old}},
    \qquad
    \frac{\mathrm{d}Q^{v^{(0)}}}{\mathrm{d}P^{u^{(0)}}}
    =\frac{p_1^r(X_1)}{p_1(X_1)}.
\label{eq:convergence-initialization}
\end{equation}
Propositions~\ref{prop:controller-half-bridge}--\ref{prop:corrector-half-bridge} yield
\begin{equation}
\begin{aligned}
    \frac{\mathrm{d}P^{u^{(k+1)}}}{\mathrm{d}Q^{v^{(k)}}}
    &=\frac{p_0(X_0)}{p_0^{Q^{v^{(k)}}}(X_0)},\\
    \frac{\mathrm{d}Q^{v^{(k+1)}}}{\mathrm{d}P^{u^{(k+1)}}}
    &=\frac{p_1^r(X_1)}{p_1^{P^{u^{(k+1)}}}(X_1)}.
\end{aligned}
\label{eq:controlled-ipf-density-updates}
\end{equation}
In particular, Eq~\ref{eq:controlled-ipf-density-updates} gives the endpoint updates
\begin{equation}
\begin{aligned}
    p_{0,1}^{Q^{v^{(k)}}}(x_0,x_1)
    &=p_{0,1}^{P^{u^{(k)}}}(x_0,x_1)
      \frac{p_1^r(x_1)}{p_1^{P^{u^{(k)}}}(x_1)},\\
    p_{0,1}^{P^{u^{(k+1)}}}(x_0,x_1)
    &=p_{0,1}^{Q^{v^{(k)}}}(x_0,x_1)
      \frac{p_0(x_0)}{p_0^{Q^{v^{(k)}}}(x_0)}.
\end{aligned}
\label{eq:convergence-endpoint-updates}
\end{equation}
The stage factors in Eq~\ref{eq:stage-controlled-potentials} give the endpoint densities relative to the same reference kernel:
\begin{equation}
\begin{aligned}
    p_{0,1}^{P^{u^{(k)}}}(x_0,x_1)
    &=p_{1\mid0}^{R}(x_1\mid x_0)
      \widehat\varphi_0^{(k)}(x_0)\varphi_1^{(k)}(x_1),\\
    p_{0,1}^{P^{u^{(0)}}}(x_0,x_1)
    &=p_{1\mid0}^{R}(x_1\mid x_0)
      \widehat\varphi_0^{\mathrm{old}}(x_0)\varphi_1^{\mathrm{old}}(x_1).
\end{aligned}
\label{eq:convergence-endpoint-factorizations}
\end{equation}
Thus the log-factors relative to the pretrained diffusion are
\begin{equation}
    \alpha^{(k)}:=\log\frac{\widehat\varphi_0^{(k)}}{\widehat\varphi_0^{\mathrm{old}}},
    \qquad
    \beta^{(k)}:=\log\frac{\varphi_1^{(k)}}{\varphi_1^{\mathrm{old}}},
    \qquad \alpha^{(0)}=\beta^{(0)}=0.
\label{eq:convergence-log-factors}
\end{equation}
Dividing the two densities in Eq~\ref{eq:convergence-endpoint-factorizations} cancels their common kernel:
\begin{equation}
\begin{aligned}
    \frac{p_{0,1}^{P^{u^{(k)}}}(x_0,x_1)}{p_{0,1}^{P^{u^{(0)}}}(x_0,x_1)}
    &=\frac{\widehat\varphi_0^{(k)}(x_0)}{\widehat\varphi_0^{\mathrm{old}}(x_0)}
      \frac{\varphi_1^{(k)}(x_1)}{\varphi_1^{\mathrm{old}}(x_1)}\\
    &=e^{\alpha^{(k)}(x_0)+\beta^{(k)}(x_1)}.
\end{aligned}
\label{eq:convergence-relative-endpoint-density}
\end{equation}
For the two marginal normalizations, Eq~\ref{eq:stage-controlled-potentials} gives
\begin{equation}
\begin{aligned}
    \int p_{0,1}^{P^{u^{(0)}}}(x_0,x_1)e^{\alpha^{(k)}(x_0)}\,\mathrm{d}x_0
    &=\varphi_1^{\mathrm{old}}(x_1)
      \int p_{1\mid0}^{R}(x_1\mid x_0)\widehat\varphi_0^{(k)}(x_0)\,\mathrm{d}x_0\\
    &=\varphi_1^{\mathrm{old}}(x_1)\widehat\varphi_1^{(k)}(x_1),\\
    \int p_{0,1}^{P^{u^{(0)}}}(x_0,x_1)e^{\beta^{(k+1)}(x_1)}\,\mathrm{d}x_1
    &=\widehat\varphi_0^{\mathrm{old}}(x_0)
      \int p_{1\mid0}^{R}(x_1\mid x_0)\varphi_1^{(k+1)}(x_1)\,\mathrm{d}x_1\\
    &=\widehat\varphi_0^{\mathrm{old}}(x_0)\varphi_0^{(k+1)}(x_0).
\end{aligned}
\label{eq:convergence-factor-integrals}
\end{equation}
The endpoint replacements in Eq~\ref{eq:controlled-ipf-density-updates}, with the factor normalization fixed in Eq~\ref{eq:forward-control-potential-update}, therefore yield
\begin{equation}
\begin{aligned}
    \beta^{(k+1)}(x_1)
    &=\log\frac{\varphi_1^{(k+1)}(x_1)}{\varphi_1^{\mathrm{old}}(x_1)}
      =\log\frac{p_1^r(x_1)}{\widehat\varphi_1^{(k)}(x_1)\varphi_1^{\mathrm{old}}(x_1)}\\
    &=\log p_1^r(x_1)
      -\log\int p_{0,1}^{P^{u^{(0)}}}(x_0,x_1)e^{\alpha^{(k)}(x_0)}\,\mathrm{d}x_0,\\
    \alpha^{(k+1)}(x_0)
    &=\log\frac{\widehat\varphi_0^{(k+1)}(x_0)}{\widehat\varphi_0^{\mathrm{old}}(x_0)}
      =\log\frac{p_0(x_0)}{\varphi_0^{(k+1)}(x_0)\widehat\varphi_0^{\mathrm{old}}(x_0)}\\
    &=\log p_0(x_0)
      -\log\int p_{0,1}^{P^{u^{(0)}}}(x_0,x_1)e^{\beta^{(k+1)}(x_1)}\,\mathrm{d}x_1.
\end{aligned}
\label{eq:convergence-sinkhorn-recursions}
\end{equation}

By Eq~\ref{eq:convergence-initialization}, $p_{0,1}^{Q^{v^{(0)}}}=w p_{0,1}^{P^{u^{(0)}}}$. Shifting the terminal factor in Eq~\ref{eq:convergence-sinkhorn-recursions} gives
\begin{equation}
\begin{aligned}
    \widetilde\beta^{(k)}(x_1)
    &:=\beta^{(k)}(x_1)-\log w(x_1),\\
    \widetilde\beta^{(k+1)}(x_1)
    &=\log p_1^r(x_1)
      -\log\int p_{0,1}^{Q^{v^{(0)}}}(x_0,x_1)e^{\alpha^{(k)}(x_0)}\,\mathrm{d}x_0,\\
    \alpha^{(k+1)}(x_0)
    &=\log p_0(x_0)
      -\log\int p_{0,1}^{Q^{v^{(0)}}}(x_0,x_1)e^{\widetilde\beta^{(k+1)}(x_1)}\,\mathrm{d}x_1.
\end{aligned}
\label{eq:convergence-shifted-recursions}
\end{equation}
Since $p_1^{Q^{v^{(0)}}}=p_1^r$, Eq~\ref{eq:convergence-shifted-recursions} gives $\widetilde\beta^{(1)}=0$. These are the Sinkhorn updates initialized at the first backward half bridge, with endpoint cost
\begin{equation}
    c(x_0,x_1):=\log\frac{p_0(x_0)p_1(x_1)}{p_{0,1}^{P^{u^{(0)}}}(x_0,x_1)}.
\label{eq:convergence-endpoint-cost}
\end{equation}

To verify the growth conditions for the endpoint cost in Eq~\ref{eq:convergence-endpoint-cost}, set $c_R(x_0,x_1):=-\log p_{1\mid0}^R(x_1\mid x_0)$. Integrating Eq~\ref{eq:ass-reference-scores} along line segments gives
\begin{equation}
\begin{aligned}
    |c_R(x_0,x_1)-c_R(0,x_1)|
    &\leq C'\bigl(\|x_0\|+\|x_0\|^2+\|x_0\|\|x_1\|\bigr),\\
    |c_R(0,x_1)|&\leq C'(1+\|x_1\|^2).
\end{aligned}
\label{eq:convergence-reference-cost-growth}
\end{equation}
The same bounds hold with the endpoints exchanged. The old Schr\"odinger system Eq~\ref{eq:schrodinger-potentials} can be written as
\begin{equation}
\begin{aligned}
    -\log\varphi_0^{\mathrm{old}}(x_0)
      &=-\log\int e^{-c_R(x_0,x_1)-\log\widehat\varphi_1^{\mathrm{old}}(x_1)}
                    p_1(x_1)\,\mathrm{d}x_1,\\
    -\log\widehat\varphi_1^{\mathrm{old}}(x_1)
      &=-\log\int e^{-c_R(x_0,x_1)-\log\varphi_0^{\mathrm{old}}(x_0)}
                    p_0(x_0)\,\mathrm{d}x_0.
\end{aligned}
\label{eq:convergence-old-conjugates}
\end{equation}
Assumption~\ref{assumptions} supplies finite entropy; Eq~\ref{eq:ass-endpoint-reward} and Eq~\ref{eq:convergence-reference-cost-growth} supply the moment and growth hypotheses. Applied to Eq~\ref{eq:convergence-old-conjugates}, \citet[Lemma~5.2 and Theorem~5.7]{ghosal2025sinkhorn} give, after fixing the additive constant,
\begin{equation}
    |\log\varphi_0^{\mathrm{old}}(x_0)|\leq C_0(1+\|x_0\|^2),
    \qquad
    |\log\widehat\varphi_1^{\mathrm{old}}(x_1)|\leq C_0(1+\|x_1\|^2).
\label{eq:convergence-old-potential-growth}
\end{equation}
Substitution of Eq~\ref{eq:convergence-endpoint-factorizations} and the terminal identity $p_1=\varphi_1^{\mathrm{old}}\widehat\varphi_1^{\mathrm{old}}$ into Eq~\ref{eq:convergence-endpoint-cost} gives

\begin{equation}
\begin{aligned}
    c(x_0,x_1)
      &=c_R(x_0,x_1)+\log\varphi_0^{\mathrm{old}}(x_0)
         +\log\widehat\varphi_1^{\mathrm{old}}(x_1)\\
      &=a_0(x_0)+a_1(x_1)+\widehat c(x_0,x_1),\\
    a_0(x_0)&:=\log\varphi_0^{\mathrm{old}}(x_0),\\
    a_1(x_1)&:=c_R(0,x_1)+\log\widehat\varphi_1^{\mathrm{old}}(x_1),\\
    \widehat c(x_0,x_1)&:=c_R(x_0,x_1)-c_R(0,x_1).
\end{aligned}
\label{eq:convergence-cost-decomposition}
\end{equation}
By Eq~\ref{eq:convergence-reference-cost-growth} and Eq~\ref{eq:convergence-old-potential-growth}, the terms in Eq~\ref{eq:convergence-cost-decomposition} satisfy $|a_i(x)|\leq C_2(1+\|x\|^2)$ for $i=0,1$, and $|\widehat c(x_0,x_1)|\leq C'(\|x_0\|+\|x_0\|^2+\|x_0\|\|x_1\|)$. Eq~\ref{eq:convergence-endpoint-cost} and these bounds, together with the endpoint moments in Eq~\ref{eq:ass-endpoint-reward} and Eq~\ref{eq:tilted-subgaussian-moment}, give

\begin{equation}
    e^{-c(x_0,x_1)}p_0(x_0)p_1^r(x_1)=p_{0,1}^{Q^{v^{(0)}}}(x_0,x_1),
    \qquad
    \KL(p_0\otimes p_1^r\Vert p_{0,1}^{Q^{v^{(0)}}})
       =\iint c\,p_0p_1^r<\infty.
\label{eq:convergence-cost-feasibility}
\end{equation}

Eq~\ref{eq:convergence-cost-feasibility} identifies the Sinkhorn initialization with the first backward half bridge and supplies a finite-entropy feasible coupling. Let $\alpha^*,\beta^*$ be the optimal relative endpoint factors, so that

\begin{equation}
\begin{aligned}
    p_{0,1}^{P^{u^*}}(x_0,x_1)
      &=p_{0,1}^{P^{u^{(0)}}}(x_0,x_1)e^{\alpha^*(x_0)+\beta^*(x_1)}\\
      &=p_{0,1}^{Q^{v^{(0)}}}(x_0,x_1)e^{\alpha^*(x_0)+\widetilde\beta^*(x_1)},\\
    \widetilde\beta^*&:=\beta^*-\log w.
\end{aligned}
\label{eq:convergence-optimal-factors}
\end{equation}

Eq~\ref{eq:convergence-endpoint-cost}--Eq~\ref{eq:convergence-cost-decomposition} establish quadratic growth of the separate endpoint terms and the required bound on their interaction. Together with Eq~\ref{eq:convergence-cost-feasibility} and the sub-Gaussian moments in Eq~\ref{eq:ass-endpoint-reward} and Eq~\ref{eq:tilted-subgaussian-moment}, these results verify the hypotheses of \citet[Theorem~5.7 and Corollary~5.8]{ghosal2025sinkhorn}. Applied to the iterates in Eq~\ref{eq:convergence-shifted-recursions} and the optimal factor in Eq~\ref{eq:convergence-optimal-factors}, those results give the following uniform quadratic bound, which implies the exponential-moment bound:
\begin{equation}
\begin{aligned}
    |\alpha^{(k)}(x)|+|\alpha^*(x)|&\leq C_3(1+\|x\|^2),\qquad k\geq0,\\
    M_\gamma:=\sup_{k\geq0}\int
       e^{\gamma|\alpha^{(k)}-\alpha^*|}p_0\,\mathrm{d}x
       &\leq e^{\gamma C_3}\int e^{\lambda\|x\|^2}p_0(x)\,\mathrm{d}x<\infty,
       \qquad 0<\gamma C_3\leq\lambda.
\end{aligned}
\label{eq:convergence-potential-moments}
\end{equation}
The constant $C_3>0$ is independent of $k$.

Define the endpoint entropies
\begin{equation}
\begin{aligned}
    F_k&:=\KL\!\left(p_{0,1}^{P^{u^*}}\Vert p_{0,1}^{P^{u^{(k)}}}\right),
    &S_k&:=\KL\!\left(p_{0,1}^{P^{u^*}}\Vert p_{0,1}^{Q^{v^{(k)}}}\right),\\
    d_k&:=\KL(p_1^r\Vert p_1^{P^{u^{(k)}}}),
    &m_k&:=\KL(p_0^{Q^{v^{(k)}}}\Vert p_0).
\end{aligned}
\label{eq:convergence-entropy-sequences}
\end{equation}
Proposition~\ref{prop:change-reference} and the endpoint KL chain rule give $F_0\leq\KL(P^{u^*}\Vert P)<\infty$.
Using the definitions in Eq~\ref{eq:convergence-entropy-sequences}, the endpoint updates Eq~\ref{eq:convergence-endpoint-updates} imply
\begin{equation}
\begin{aligned}
    F_k-S_k
      &=\iint p_{0,1}^{P^{u^*}}\log\frac{p_{0,1}^{Q^{v^{(k)}}}}
                    {p_{0,1}^{P^{u^{(k)}}}}\,\mathrm{d}x_0\mathrm{d}x_1
        =\int p_1^r\log\frac{p_1^r}{p_1^{P^{u^{(k)}}}}\,\mathrm{d}x_1=d_k,\\
    S_k-F_{k+1}
      &=\int p_0\log\frac{p_0}{p_0^{Q^{v^{(k)}}}}\,\mathrm{d}x_0
        =\KL(p_0\Vert p_0^{Q^{v^{(k)}}})\geq0,\\
    \sum_{k=0}^{n}d_k&\leq F_0-F_{n+1}\leq F_0<\infty.
\end{aligned}
\label{eq:convergence-entropy-descent}
\end{equation}
Eq~\ref{eq:convergence-entropy-descent} implies $d_k\to0$. Marginalizing the terminal replacement in Eq~\ref{eq:convergence-endpoint-updates} then bounds $m_k$ from Eq~\ref{eq:convergence-entropy-sequences}:
\begin{equation}
\begin{aligned}
    0\leq m_k
      &\leq\KL\!\left(p_{0,1}^{Q^{v^{(k)}}}\Vert p_{0,1}^{P^{u^{(k)}}}\right)\\
      &=\int p_1^r\log\frac{p_1^r}{p_1^{P^{u^{(k)}}}}\,\mathrm{d}x_1
        =d_k\longrightarrow0.
\end{aligned}
\label{eq:convergence-marginal-entropy}
\end{equation}
The entropy inequality and Eq~\ref{eq:convergence-potential-moments} give

\begin{equation}
    \int|\alpha^{(k)}-\alpha^*|p_0^{Q^{v^{(k)}}}\,\mathrm{d}x
      \leq\gamma^{-1}(m_k+\log M_\gamma)<\infty.
\label{eq:convergence-factor-integrability}
\end{equation}

Eq~\ref{eq:convergence-potential-moments} ensures that $C_4:=2\gamma^{-1}(3/2+\log M_\gamma)$ is finite and independent of $k$. Combining the stage factorization Eq~\ref{eq:convergence-relative-endpoint-density} with the terminal replacement Eq~\ref{eq:convergence-endpoint-updates} and the optimal factorization Eq~\ref{eq:convergence-optimal-factors} gives

\begin{equation}
\begin{aligned}
    p_{0,1}^{Q^{v^{(k)}}}(x_0,x_1)
      &=p_{0,1}^{Q^{v^{(0)}}}(x_0,x_1)
        e^{\alpha^{(k)}(x_0)+\widetilde\beta^{(k+1)}(x_1)},\\
    \log\frac{p_{0,1}^{Q^{v^{(k)}}}(x_0,x_1)}
              {p_{0,1}^{P^{u^*}}(x_0,x_1)}
      &=\alpha^{(k)}(x_0)-\alpha^*(x_0)
        +\widetilde\beta^{(k+1)}(x_1)-\widetilde\beta^*(x_1).
\end{aligned}
\label{eq:convergence-backward-factor-ratio}
\end{equation}

Both laws in Eq~\ref{eq:convergence-backward-factor-ratio} have terminal marginal $p_1^r$. Conditioning on $X_1$, their terminal factor cancels in the symmetric KL; Eq~\ref{eq:convergence-factor-integrability} justifies the remaining integral. With $F_{k+1}\leq S_k$ from Eq~\ref{eq:convergence-entropy-descent} and $m_k\to0$ from Eq~\ref{eq:convergence-marginal-entropy}, we obtain
\begin{equation}
\begin{aligned}
    0\leq F_{k+1}\leq S_k
      &\leq\KL\!\left(p_{0,1}^{Q^{v^{(k)}}}\Vert p_{0,1}^{P^{u^*}}\right)+S_k\\
      &=\int(\alpha^{(k)}-\alpha^*)
               (p_0^{Q^{v^{(k)}}}-p_0)\,\mathrm{d}x\\
      &\leq C_4\left(\sqrt{m_k}+\frac12m_k\right)\longrightarrow0.
\end{aligned}
\label{eq:convergence-endpoint-kl}
\end{equation}
The last inequality is the weighted entropy estimate of \citet[Lemma~3.1]{ghosal2025sinkhorn}, with its exponential-moment hypothesis verified in Eq~\ref{eq:convergence-potential-moments}. This proves endpoint entropy convergence directly from the two half-bridge updates.

Eq~\ref{eq:convergence-endpoint-kl} proves $F_k\to0$, where $F_k$ is the endpoint KL defined in Eq~\ref{eq:convergence-entropy-sequences}. Each multiplier in Eq~\ref{eq:controlled-ipf-density-updates} depends only on the endpoints, and the optimal bridge also preserves the pretrained conditional trajectories given both endpoints. The conditional term in the KL chain rule therefore vanishes, so the same bound gives path-space convergence, for $k\geq1$:
\begin{equation}
\begin{aligned}
    P^{u^{(k)}}(\cdot\mid X_0,X_1)
    &=Q^{v^{(k)}}(\cdot\mid X_0,X_1)
      =P^{u^{(0)}}(\cdot\mid X_0,X_1),\\
    P^{u^*}(\cdot\mid X_0,X_1)
    &=P^{u^{(0)}}(\cdot\mid X_0,X_1),\\
    \KL(P^{u^*}\Vert P^{u^{(k)}})
    &=\KL\!\left(p_{0,1}^{P^{u^*}}\Vert p_{0,1}^{P^{u^{(k)}}}\right)\\
    &\quad+\mathbb{E}_{P^{u^*}}
      \underbrace{\KL\!\left(P^{u^*}(\cdot\mid X_0,X_1)
                   \Vert P^{u^{(k)}}(\cdot\mid X_0,X_1)\right)}_{=0}
      \\
    &=F_k\leq C_4\left(\sqrt{m_{k-1}}+\frac12m_{k-1}\right)
      \longrightarrow0.
\end{aligned}
\label{eq:convergence-diffusion-kl}
\end{equation}
This is the endpoint-to-diffusion lifting described by \citet[Section~3.5]{debortoli2021}.

Finally, both controlled diffusions start from $p_0$ and have volatility $\sigma_t$. The control-energy identity justified by Eq~\ref{eq:ass-novikov-energy}, together with the KL convergence in Eq~\ref{eq:convergence-diffusion-kl}, gives
\begin{equation}
\begin{aligned}
    \mathbb{E}_{P^{u^*}}\int_0^1
        \|u_t^{(k)}(X_t)-u_t^*(X_t)\|^2\,\mathrm{d}t
    &=2\KL(P^{u^*}\Vert P^{u^{(k)}})\longrightarrow0,\\
    P^{u^*}
    &=\SB_P(p_0,p_1^r)
      \stackrel{\text{Proposition~\ref{prop:change-reference}}}{=}
      \SB_R(p_0,p_1^r).
\end{aligned}
\label{eq:convergence-control-limit}
\end{equation}
Thus the controllers converge in integrated mean square under the optimal controlled diffusion, with the limit solving the tilted SB problem.
\end{proof}
 \section{Competitors}\label{app:competitors}

\paragraph{Shared pretrained sampler.}
All competitors use the frozen bridge $P$. For source $x_0$ and grid
$0=t_0<\cdots<t_L=1$, its Euler step is
\begin{equation}
    X_{i+1}=X_i+\Delta_i b_{t_i}(X_i)
    +\sigma\sqrt{\Delta_i}\,\xi_i,\qquad
    \Delta_i=t_{i+1}-t_i,\quad \xi_i\sim\mathcal N(0,I).
    \label{eq:competitor-base-step}
\end{equation}
Image runs omit noise in the final step. Let $r$ denote the terminal log
reward: $r=-\lambda E$ for images and the
terminal log density ratio for the two-dimensional experiment. The pretrained drift predicts the endpoint as
\begin{equation}
    \widehat x_1(x,t)=x+(1-t)b_t(x),\qquad
    \widehat x_1(x,1)=x.
    \label{eq:competitor-endpoint-estimate}
\end{equation}

\subsection{Reward tilting with non-memoryless priors}
\label{app:reward-tilting-nonmemoryless}

Reward finetuning for the memoryless diffusion models \cite{song2021scorebased} is widely explored research field \cite{uehara2024entropy, domingoenrich2025, black2024ddpo}. A lot of memoryless diffusion reward finetuning methodology is based on a fact of independence of start of generation $X_0$ and the end of generation $X_1$, i.e., memorylessness property: 

Let $R_{\mathrm{diff}}$ be a memoryless diffusion with marginals $p_0$ and $p_1$:
\begin{equation}
\begin{aligned}
R_{\mathrm{diff}}:\quad
\mathrm{d}X_t
&=f_t^{\mathrm{diff}}(X_t)\,\mathrm{d}t
  +\sigma_t^{\mathrm{diff}}\,\mathrm{d}W_t,\\
p_{0,1}^{R_{\mathrm{diff}}}(x_0,x_1)
&=p_0(x_0)p_1(x_1).
\end{aligned}
\end{equation}

Then considering the diffusion bridges \cite{zhou2024ddbm}, which in general do not have a memoryless property:

\begin{equation}
\begin{aligned}
R_{\mathrm{bridge}}:\quad
\mathrm{d}X_t
&=b_t(X_t)\,\mathrm{d}t+\sigma_t\,\mathrm{d}W_t,\\
p_{0,1}^{R_{\mathrm{bridge}}}(x_0,x_1)
&=p_0(x_0)p_{1\mid0}^{R_{\mathrm{bridge}}}(x_1\mid x_0)
 \ne p_0(x_0)p_1(x_1)\quad\text{in general}.
\end{aligned}
\end{equation}

For either reference $S\in\{R_{\mathrm{diff}},R_{\mathrm{bridge}}\}$,
let $S^r$ be the SOC solution with terminal reward $r$ in
Eq~\ref{eq:soc_def}. When the controlled solution exists and the
conditional exponential moment is finite and positive,
Eq~\ref{eq:soc-optimal-law} gives:
\begin{equation}
\begin{aligned}
V_0^S(x_0)
&=-\log\int p_{1\mid0}^S(y\mid x_0)e^{r(y)}\,\mathrm{d}y,\\
p_{0,1}^{S^r}(x_0,x_1)
&=p_0(x_0)p_{1\mid0}^S(x_1\mid x_0)
  e^{r(x_1)+V_0^S(x_0)}.
\end{aligned}
\end{equation}

Write $\bar Z_r=\int p_1(y)e^{r(y)}\,\mathrm{d}y$, so that
$p_1^r\propto p_1e^r$. For the memoryless diffusion,
$V_0^{R_{\mathrm{diff}}}(x_0)=-\log\bar Z_r$, and therefore the SOC solution marginalization gives unbiased reward tilted $x_1$ distribution $p_1^{R_{\mathrm{diff}}^r}(x_1)=p_1^r(x_1)$:
\begin{equation}
\begin{aligned}
p_1^{R_{\mathrm{diff}}^r}(x_1)
&=\int p_0(x_0)p_1(x_1)
       e^{r(x_1)-\log\bar Z_r}\,\mathrm{d}x_0=\frac{p_1(x_1)e^{r(x_1)}}{\bar Z_r}
 =p_1^r(x_1).
\end{aligned}
\end{equation}

For the diffusion bridge, $V_0^{R_{\mathrm{bridge}}}(x_0)$ generally
depends on the source. Marginalizing the SOC solution gives 
\begin{equation}
\begin{aligned}
p_1^{R_{\mathrm{bridge}}^r}(x_1)
&=e^{r(x_1)}
  \int p_0(x_0)
       p_{1\mid0}^{R_{\mathrm{bridge}}}(x_1\mid x_0)
       e^{V_0^{R_{\mathrm{bridge}}}(x_0)}
       \,\mathrm{d}x_0\\
&=p_1(x_1)e^{r(x_1)}
  \mathbb E_{R_{\mathrm{bridge}}}
  \!\left[e^{V_0^{R_{\mathrm{bridge}}}(X_0)}
          \mid X_1=x_1\right]\\
&=p_1^r(x_1)\,
  \underbrace{\bar Z_r\,
  \mathbb E_{R_{\mathrm{bridge}}}
  \!\left[e^{V_0^{R_{\mathrm{bridge}}}(X_0)}
          \mid X_1=x_1\right]}_{\mathcal B_r(x_1)}.
\end{aligned}
\end{equation}

The factor $\mathcal B_r$ need not be constant, so the \textbf{reward alone
does not produce $p_1^r$ in general}.

\subsection{DPS-style guidance}

DPS differentiates a likelihood at a predicted endpoint
\citep{chung2023dps} and CDDB-deep applies this idea to diffusion
bridges for inverse problems \citep{chung2023cddb}. Our methodology is similar to CDDB, we guide the
frozen bridge with the terminal reward:
\begin{equation}
    X_{i+1}=X_i+\Delta_i b_{t_i}(X_i)
    +\gamma\sigma^2\Delta_i
       \left.\nabla_x r\!\left(\widehat x_1(x,t_i)\right)\right|_{x=X_i}
    +\sigma\sqrt{\Delta_i}\,\xi_i.
    \label{eq:competitor-dps-step}
\end{equation}
$\gamma$ is the guidance scale, separate from reward strength. We use the full Jacobian of $\widehat x_1$, with
$\gamma\in\{0.5,1,2,4\}$. 

\subsection{Conditional self-normalized importance sampling (CondSNIS)}

For each fixed source $x_0$, CondSNIS samples a population of $K$
independent old-bridge trajectories and targets:
\begin{equation}
    p^r_{1\mid0}(y\mid x_0)
    \propto e^{r(y)}p^P_{1\mid0}(y\mid x_0)\qquad
    Z_r(x_0)=\mathbb E_P[e^{r(X_1)}\mid X_0=x_0].
    \label{eq:competitor-conditional-tilt}
\end{equation}
For terminals $Y^1,\ldots,Y^K$, the weighted population is
\begin{equation}
    \widehat p^r_{1\mid0}(y\mid x_0)
    =\sum_{j=1}^K w_j\delta_{Y^j}(\mathrm dy),\qquad
    w_j=\frac{e^{r(Y^j)}}{\sum_{k=1}^K e^{r(Y^k)}},
    \label{eq:competitor-snis}
\end{equation}

where $\delta$ is the Dirac delta function. As the number of particles $N \rightarrow{} \infty$ this procedure results in conditional tilt $p^r_{1\mid0}(y\mid x_0)
    \propto e^{r(y)}p^P_{1\mid0}(y\mid x_0)$ \cite{delmoral2006}. At the end of the procedure we sample the single final sample with log weights $w$.

\subsection{Conditional sequential Monte Carlo (CondSMC)}

CondSMC propagates $N$ old-bridge particles per source
\citep{delmoral2006}. For this base proposal, RNE's bridge-compatible
path-ratio weights reduce to our potential ratios \citep{he2026rne}.
We define $\rho_i(x)$ as particle $i$ score and $\ell_i$ as particle $i$ log importance weights, where $\rho_i(x)=r(\widehat x_1(x,t_i))$, $\rho_L(x)=r(x)$
and $\ell_0^j=\rho_0(x_0)$. With $s$ the preceding potential step,
update:
\begin{equation}
    \ell_i^j=\ell_s^j+\rho_i(X_i^j)-\rho_s(X_s^j),\qquad
    \mathrm{ESS}_i=
    \left(\sum_{j=1}^N
    \operatorname{softmax}(\ell_i)_j^2\right)^{-1}.
    \label{eq:competitor-smc-weight}
\end{equation}
Each step resamples systematically if $\mathrm{ESS}_i<N/2$,
resetting weights but retaining potentials. As the number of particles $N \rightarrow{} \infty$ this procedure results in conditional tilt $p^r_{1\mid0}(y\mid x_0)
    \propto e^{r(y)}p^P_{1\mid0}(y\mid x_0)$ \cite{delmoral2006}. At the end of the procedure we sample the single final sample with log weights $l$.

 \section{General design choices for \tsbm{}}\label{app:tsbm-design-choices}

\wasyparagraph~\ref{sec:method} states the \tsbm{} updates in terms of an incremental controller $u$ and a full terminal corrector $h$. Here we describe the neural parameterization and sampling procedure used in the experiments.

\paragraph{Neural parameterization.}
DSBM pretraining \citep{shi2023dsbm} provides a forward drift $b_\psi$ and an old-bridge corrector $h_{\mathrm{old},\psi}$. We freeze both networks during \tsbm{} training. We initialize a trainable full forward drift $b_\theta$ and full corrector $h_\phi$ from these pretrained networks, then carry their parameters forward between stages. The incremental controller in \wasyparagraph~\ref{sec:method} is represented implicitly by
\begin{equation}
    u_\theta(x,t)=\frac{b_\theta(x,t)-b_\psi(x,t)}{\sigma_t}.
    \label{eq:practical-controller-parameterization}
\end{equation}
Thus, simulating the current process requires one evaluation of $b_\theta$ per time step rather than separate evaluations of $b_\psi$ and a controller network. The frozen $b_\psi$ is still used to form the controller-matching loss and the lean-adjoint target. Likewise, the frozen $h_{\mathrm{old},\psi}$ is subtracted from $h_\phi$ in the terminal adjoint condition.

\paragraph{Replay buffers.}
We maintain a buffer of $N$ sampling trajectories and their corresponding backpropagated "lean" adjoints from the current drift $b_\psi^{(k)}$ and replace it in full every $L$ controller-gradient steps. Each step uses a uniformly sampled minibatch of stored trajectories. After updating the drift, we draw endpoint pairs $(X_0,X_1)$ under that drift $b_\psi^{(k)}$ for the corrector $h_\phi^{(k)}$ update and do not refresh their buffer due to the start of the next stage $k+1$. 

\paragraph{Reference process.} In all our experiments the reference process is Wiener, i.e.,  $R = W^\epsilon$. Which leads to $\sigma_t=\epsilon$ and $s_R(X_0,X_1) = -\frac{X_1 - X_0}{\epsilon^2}$.

\paragraph{Simulation and regression targets.}
We simulate the controlled SDE with drift $b_\theta$ by Euler-Maruyama. Along each sampled trajectory, we integrate the lean-adjoint ODE backward from
\begin{equation}
    \widetilde a_1=-\nabla r(X_1)+h_\phi(X_1)-h_{\mathrm{old},\psi}(X_1),
    \qquad
    -\frac{\mathrm{d}\widetilde a_t}{\mathrm{d}t}
    =\nabla_x b_\psi(X_t,t)^\top\widetilde a_t.
    \label{eq:practical-lean-adjoint}
\end{equation}
The Jacobian-vector product uses the frozen pretrained drift, not the trainable drift. We regress $b_\phi(X_t,t)^{(k)}$ against $b_\theta-\sigma_t^2\widetilde a_t$ using Eq~\ref{eq:controller-loss}. For the corrector update, we regress $h_\phi(X_1)$ against the endpoint transition score $s_R(X_0,X_1) = s_{W^\epsilon}(X_0,X_1)=-\frac{X_1 - X_0}{\epsilon^2}$ of the original reference process $W^\epsilon$, as in Eq~\ref{eq:tsbm-corrector-update}.

\begin{algorithm}[t]
\caption{Tilted Schrödinger Bridge Matching (\tsbm{}) Practical}
\label{alg:tsbm_practical}
\begin{algorithmic}[1]
\Require frozen old drift $\Fref_\theta$, old corrector $ h_{\text{old}, \theta}$, reference process $R$, reward $r$, stages $K$, number of controller updates $M_{\rm ctrl}$, number of corrector updates $M_{\rm corr}$, freq of Replay Buffer updates $N$
\State initialize $b^{(0)}_\psi \gets \Fref_\theta$, $h^{(0)}_\psi\gets h_{\text{old}, \theta}$, $S \gets \emptyset$
\For{$k=0,\ldots,K-1$}
    \State $b_\phi \gets \text{copy}(b^{(k)}_\phi)$
    \For{$m=0,\ldots,M_{\rm ctrl}-1$}
    \If{$m \bmod N = 0$}
            \State Sample \textit{paths} $(X_t)_{t \in [0,1]}$ by Euler-Maruyama inference of the SDE with drift $b_\phi$
            \State $\tilde{a}^{(k)}_1 \gets -\nabla r(X_1) + h^{(k)}_\phi(X_1) -  h_{\text{old}, \theta}(X_1)$
            \State Propagate $\tilde{a}^{(k)}_t$ by solving the backward lean-adjoint ODE using Eq~\ref{eq:lean-adj}
            \State Add paths $(X_t)_{t \in [0,1]}$ and lean adjoints $(\tilde{a}^{(k)}_t)_{t \in [0,1]}$ to replay buffer $S$
        \EndIf
\State Sample \textit{paths} $(X_t)_{t \in [0,1]}$ and lean adjoints $\tilde{a}$ from replay buffer $S$
\State Make a gradient step on $\nabla_\phi \mathcal{L}_{\rm ctrl}$ w.r.t $b_\phi$ parameters;
   \EndFor
   \State Sample many \textit{endpoints} $(X_0, X_1)$ by the inference of SDE with drift $ b_\phi $; 
   \State $h_\phi \gets \text{copy}(h^{(k)}_\phi)$
   \For{$m=0,\ldots,M_{\rm corr}-1$}
      \State Sample a minibatch of endpoint pairs      \State Make a gradient step on $\nabla_\phi \mathcal{L}_{\rm corr}$  w.r.t $b_\phi$ parameters;
  \EndFor $b^{(k+1)}_\phi \gets b_\phi $, $h^{(k+1)}_\phi\gets h_\phi$
  \State 
\EndFor
\State \Return $u^{(K)}, h^{(K)}$
\end{algorithmic}
\end{algorithm}
 \section{Reward Functions}
\label{app:reward-functions}

\subsection{Toy illustrative experiment}

For the 2D toy experiment, we use the log-density-ratio reward
\begin{equation}
    r(x)=\log\frac{p_1^{\mathrm{tilt}}(x)}{p_1(x)}, \quad p_1 = \sum_{k=1}^4 \alpha_k \mathcal{N}(\mu_k, \sigma_k), \quad  p_1^{\rm tilt} = \sum_{k=1}^4 \alpha^{\rm tilt}_k \mathcal{N}(\mu_k, \sigma_k)
\end{equation}
where the mixture weights change from
$\alpha=(0.25,0.25,0.25,0.25)$ to $\alpha^{\rm tilt}=(0,0.2,0.2,0.6)$.
Thus, $p_1(x)e^{r(x)}=p_1^{\mathrm{tilt}}(x)$, while the Gaussian
means and covariances remain unchanged.

\begin{equation}
    E(x) = -r(x) = \log p_1(x) - \log p^{\rm tilt}_1(x)
\end{equation}

\subsection{Colored MNIST unpaired translation}\label{app:reward-cmnist}

This appendix gives the terminal energies used for the image experiments in Section~\ref{sec:experiments}. 

\begin{table}[H]
\caption{Colored MNIST terminal energies and desired outputs. Each pair shows a source $x_0$ (left) and the terminal TSBM generation from the same column of the corresponding comparison grid (right).}
\label{tab:colored-mnist-reward-intent}
\centering
\small
\setlength{\tabcolsep}{3pt}
\begin{tabular}{p{0.18\linewidth}p{0.37\linewidth}p{0.32\linewidth}}
\toprule
Reward & Intended terminal attribute & Desired output \\
\midrule
Red chroma & A digit 3 with pronounced red chroma and a neutral background. & \includegraphics[height=30pt]{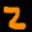}\,$\rightarrow$\,\includegraphics[height=30pt]{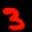} \\
Thin stroke (target 3.65) & A digit 3 with estimated stroke thickness close to $3.65$. & \includegraphics[height=30pt]{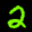}\,$\rightarrow$\,\includegraphics[height=30pt]{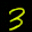} \\
Thick stroke (target 7) & A digit 3 with estimated stroke thickness close to $7$. & \includegraphics[height=30pt]{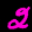}\,$\rightarrow$\,\includegraphics[height=30pt]{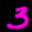} \\
\bottomrule
\end{tabular}
\end{table}

\paragraph{Red-chroma reward.}
Let $x\in[-1,1]^{3\times H\times W}$ and index spatial locations by $i$. The red-chroma score favors red image channel $x_R$ relative to green and blue channels , i.e., $x_G$ and $x_B$, while neutral gray backgrounds do self cancel and do contribute to this reward:
\begin{align}
c_{\mathrm{red}}(x)
&=\frac{1}{HW}\sum_i
\left[x_{R,i}-\frac{x_{G,i}+x_{B,i}}{2}\right],\\
E_{\mathrm{red}}(x)&=-c_{\mathrm{red}}(x).
\label{eq:red-chroma-reward}
\end{align}

\paragraph{Stroke-thickness energies.}
For normalized to $[0, 1]$ image $u=(x+1)/2\in[0,1]^{3\times H\times W}$, we first construct a differentiable, digit stroke area estimate:
\begin{align}
I_i(x)
&=\sum_c u_{c,i}
  \frac{\exp(\kappa_{\mathrm{rgb}}u_{c,i})}
       {\sum_{c'}\exp(\kappa_{\mathrm{rgb}}u_{c',i})},
&M_i(x)&=\sigma\!\left(\kappa_m[I_i(x)-t_{\mathrm{fg}}]\right),\\
A(x)&=\sum_i M_i(x),
\end{align}

where $I_i$ is a \textit{brightness-like score} at pixel \(i\), $M_i(x)$ is the \textit{soft foreground mask}, $A(x)$ estimates the \textit{stroke’s area}. $I_i$  gives more weight to whichever color channel is strongest, acting like a smooth maximum of red, green, and blue. $M_i(x)$ acts as the soft foregrouhnd mask, where the sigmoid turns \(I_i\) into a value near 0 for a dark background pixel and near 1 for a bright stroke pixel. Other parameters: $u_{c, i}$ is the color $c$ channel of pixel $i$, $\kappa_{\rm rgb}=20$ is the strength of concentration on the brightest channel, $\kappa_{\rm m}=40$ controls how sharply the mask changes around that threshold and the $t_{\mathrm{fg}}=0.1$ is the foreground threshold applied to \(I_i\). These parameters were selected empirically.

\begin{align}
&P(x)=\sum_i\sqrt{(K_x*M)_i^2+(K_y*M)_i^2+\epsilon}, \qquad
w(x)=\frac{2A(x)}{P(x)+\epsilon}.
\end{align}

Here $P$ approximates \textit{foreground perimeter} obtained by adding the strength of the mask’s edges and $w(x)$ is the ratio between strokes area and foreground perimeter or \textbf{resulting stroke-width estimate}, measured roughly in pixels. Other parameters and functions are: $K_x$ and $K_y$ are the $3\times3$ Sobel filters, scaled by $1/8$ and applied with replicate padding, $\epsilon=10^{-8}$ a tiny constant that prevents numerical problems at zero.  

The \textbf{thin-stroke} and \textbf{thick-stroke} energies are:
\begin{align}
r_{\mathrm{thin}}(x)&=-\left(\frac{w(x)-3.65}{0.9935646}\right)^2,\\
r_{\mathrm{thick}}(x)&=-\left(\frac{w(x)-7}{0.9935646}\right)^2,
\label{eq:stroke-thickness-reward}
\end{align}

which are just approximate stroke width $w(x)$ MSE towards targets, i.e., $3.65$ and $7$, normalized by $0.9935646$, which is difference between the 25th and 75th quantiles of $w(x)$ measure on 32x32 Colored MNIST dataset digit 3s.

\subsection{CelebA unpaired translation}\label{app:reward-celeba}

Let $S(x,p)$ denote the contribution of a preference model \cite{xu2023imagereward, kirstain2023pickapic, wu2023hpsv2} for image $x$ and prompt $p$. The prompt-conditioned reward can be written as:
\begin{equation}
r(x;p^+,p^-)=S(x,p^+)-S(x,p^-).
\label{eq:celeba-prompt-reward}
\end{equation}
For a positive-only target we set $S(x,p^-)=0$.

\paragraph{Prompts.} The positive prompts used for the three $128\times128$ male-to-female translation experiments are quoted verbatim from the run configurations:
\begin{itemize}
    \item \emph{Natural Toothy Smile}: ``a natural portrait photo of a woman with a relaxed smile and naturally visible teeth''
    \item \emph{Heavy Makeup}: ``a natural portrait photo with glamorous eye makeup, eyeliner, and lipstick''
    \item \emph{Elderly Woman}: ``a natural portrait of an elderly woman with visible age lines and mature facial features''
\end{itemize}
The elderly-woman reward additionally uses the negative prompt \textbf{only during training} ``a natural portrait of a young woman with smooth youthful facial features''. The smile and makeup energies use only their positive prompts.

\paragraph{ImageReward reward.} We use the frozen ImageReward-v1.0 text-image preference model~\citep{xu2023imagereward} to define $S$ in Eq.~\ref{eq:celeba-prompt-reward} as the preference score, with the run-specific calibration where configured. For scoring, images are clamped to $[-1,1]$ and bicubic-resized to $224\times224$. The main-paper ImageReward column reports the raw \emph{positive-prompt} score, not the calibrated reward or elderly-woman margin. For this evaluation, each saved output is clamped, then given a seeded reflect-padded translation of up to four pixels, random resized crop, and horizontal flip before the $224\times224$ resize.

\paragraph{Normalization.} In addition, the ImageReward was normalized w.r.t. our CelebA data to stabilize the learning procedure and make choice of strength parameter $\lambda$ easier. We pick $1024$ train samples from the dataset, score them with unnormalized ImageReward and normalized by mean and std.

\paragraph{PickScore reward.} PickScore-v1~\citep{kirstain2023pickapic} is a CLIP-H model finetuned on Pick-a-Pic preference comparisons. Its image--text similarity could provide $S$ in Eq.~\ref{eq:celeba-prompt-reward}, but PickScore is not the training reward in these three runs: we use it only for post-hoc evaluation. The reported PickScore column is the raw positive-prompt similarity, with no subtraction of the elderly-woman negative prompt. We apply the same seeded augmentation to each saved output: clamp to $[-1,1]$, reflect-padded translation of up to four pixels, random resized crop, and horizontal flip, followed by bicubic resizing to $224\times224$ for scoring. 
 \section{Experimental details}\label{app:exp-details}
\label{app:image-protocols}

Tables~\ref{tab:tsbm-hyperparameters} and~\ref{tab:dsbm-hyperparameters}
summarize the hyperparameters of the \tsbm{} and pretrained DSBM models respectively.
Here $\epsilon$ is the diffusion coefficient in the SDE. Gaussian noise is omitted on the final sampling step for image experiments. Parameter counts are for each trainable network, excluding its frozen pretrained counterpart. The EMA states for Exponential Moving Average. The Adam is used for optimization \cite{kingma2015adam}. The SDE was simulated using Euler Maryama and then "lean" adjoint ODE was integrated using the SDE trajectory and Euler integrator. Image data the was normalized to $[-1, 1]$. For DSBM hyperparamter names explanation visit corresponding Appendix in \cite{shi2023dsbm}.

For image experiments we've used the official DSBM github repository:

\begin{center}
    \url{https://github.com/yuyang-shi/dsbm-pytorch}
\end{center}

\begin{table}[H]
\centering
\caption{TSBM hyperparamters.}
\label{tab:tsbm-hyperparameters}
\footnotesize
\setlength{\tabcolsep}{4pt}
\renewcommand{\arraystretch}{1.13}
\begin{tabular}{@{}p{0.27\linewidth}p{0.20\linewidth}p{0.24\linewidth}p{0.21\linewidth}@{}}
\toprule
Hyperparameter & Toy 2D & Colored MNIST & CelebA \\
\midrule
Controller gradient steps / stage & 1,000 & 1,000 & 1,000 \\
Corrector gradient steps / stage & 782 & 1,000 & 1,000 \\
Training batch size & 128 & 32 & 32 \\
Controller replay size & None (fresh batch for every step) & 512 (red); 1,024 (strokes) & 1,024 \\
Controller replay refresh & Every step & Every 50 steps & Every 100 steps \\
Corrector data-set size / stage & 100,000 endpoints & 4,096 endpoints & 8,192 endpoints \\
EMA rate & None & 0.99 & 0.99 \\
Optimizer and learning rate & Adam, $10^{-3}$ & Adam, $10^{-5}$ (red); $3\cdot10^{-5}$ (strokes) & Adam, $10^{-5}$ \\
Controller / corrector network & Two-hidden-layer SiLU MLP, width 64 & U-Net, 128 base channels, 2 residual blocks / scale & U-Net, 128 base channels, 2 residual blocks / scale \\
Controller parameters & 4,546 & $\approx39.6$M & $\approx38.3$M \\
Corrector parameters & 4,482 & $\approx39.6$M & $\approx38.3$M \\
Diffusion coefficient $\epsilon$ & 1 & 1 & 1 \\
Reward strength $\lambda$ & 1 & 1,000 & 30 \\
Controller / corrector damping & 1 / 1 & 0.9 / 0.9 (red, thick); 1 / 1 (thin) & 0.9 / 0.9 \\
NFE & 40 & 30 & 100 \\
CondSNIS particles $K$ per input & 64 & 8 & 4 \\
CondSMC particles $N$ per input & 64 & 8 & 4 \\
TSBM stages & 20 & 5 & 5 \\
Training wall clock time & 15min & 1h & 13h \\
Hardware & CPU & 2 GPU & 2 GPU \\
\bottomrule
\end{tabular}
\end{table}

\begin{table}[H]
\centering
\caption{DSBM pretraining hyperparamters.}
\label{tab:dsbm-hyperparameters}
\footnotesize
\setlength{\tabcolsep}{4pt}
\renewcommand{\arraystretch}{1.13}
\begin{tabular}{@{}p{0.27\linewidth}p{0.20\linewidth}p{0.24\linewidth}p{0.21\linewidth}@{}}
\toprule
Hyperparameter & Toy 2D & Colored MNIST & CelebA \\
\midrule
Forward / backward gradient steps, initial IMF iter & 10,000 & 100,000 & 200,000 \\
Forward / backward gradient steps & 10,000 & 5,000  & 20,000 \\
Training batch size & 128 & 128 & 64 \\
Coupling cache size & 100,000 pairs & 10,000 pairs & 1,000 pairs \\
Coupling cache refresh & Once / direction / cycle & Every 1,000 steps & Every 1,000 steps \\
EMA rate & None & 0.999 & 0.999 \\
Optimizer and learning rate & Adam, $10^{-4}$ & Adam, $10^{-4}$ & Adam, $10^{-4}$ \\
Forward / backward network & Two-hidden-layer SiLU MLP, width 128 & U-Net, 128 base channels, 2 residual blocks / scale & U-Net, 128 base channels, 2 residual blocks / scale \\
Forward model parameters & 17,282 & $\approx39.6$M & $\approx38.3$M \\
Backward model parameters & 17,282 & $\approx39.6$M & $\approx38.3$M \\
Diffusion coefficient $\epsilon$ & 1 & 1 & 1 \\
NFE & 40 & 30 & 100 \\
DSBM stages used & 20 & 20 & 20 \\
\bottomrule
\end{tabular}
\end{table}

 \clearpage
\section{Additional qualitative results}
\label{app:additional-qualitative-results}
\raggedbottom

\subsection{Colored MNIST}

This section gives the complete qualitative Colored MNIST comparisons corresponding to Figure~\ref{fig:colored-mnist-qualitative}. In addition to the methods retained in the main text, the grids show 16 rather than six aligned sources and include DPS with $\gamma=0.5$ and CondSNIS with $K=8$ proposals per fixed source.

\begin{figure}[H]
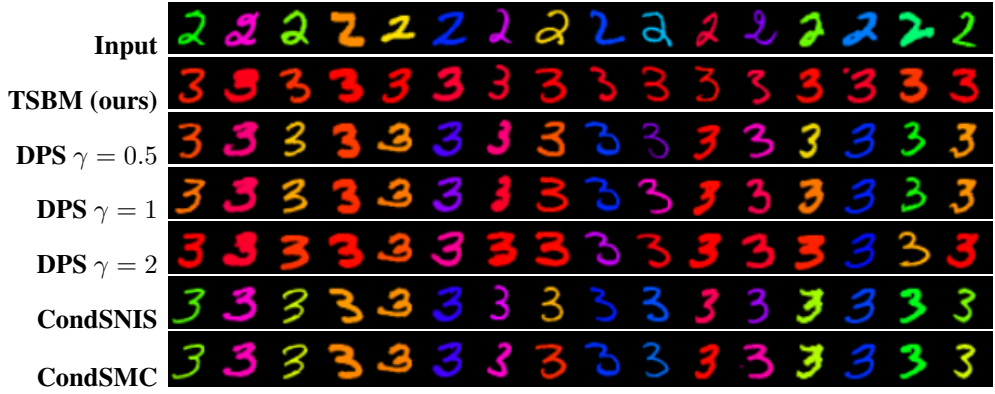

\centering
\setlength{\tabcolsep}{0pt}
\begin{tabular}{@{}r@{\hspace{4pt}}p{0.78\linewidth}@{}}
\textbf{Input} & \includegraphics[width=\linewidth,trim=192bp 274bp 544bp 0bp,clip]{figures/colored_mnist_red_chroma_target_source.png}\\[-2pt]
\textbf{TSBM (ours)} & \includegraphics[width=\linewidth,trim=192bp 240bp 544bp 34bp,clip]{figures/colored_mnist_red_chroma_target_source.png}\\[-2pt]
\textbf{DPS $\gamma=0.5$} & \includegraphics[width=\linewidth,trim=192bp 172bp 544bp 102bp,clip]{figures/colored_mnist_red_chroma_target_source.png}\\[-2pt]
\textbf{DPS $\gamma=1$} & \includegraphics[width=\linewidth,trim=192bp 138bp 544bp 136bp,clip]{figures/colored_mnist_red_chroma_target_source.png}\\[-2pt]
\textbf{DPS $\gamma=2$} & \includegraphics[width=\linewidth,trim=192bp 104bp 544bp 170bp,clip]{figures/colored_mnist_red_chroma_target_source.png}\\[-2pt]
\textbf{CondSNIS} & \includegraphics[width=\linewidth,trim=192bp 36bp 544bp 238bp,clip]{figures/colored_mnist_red_chroma_target_source.png}\\[-2pt]
\textbf{CondSMC} & \includegraphics[width=\linewidth,trim=192bp 0bp 544bp 272bp,clip]{figures/colored_mnist_red_chroma_target_source.png}
\end{tabular}
\caption{Full red-chroma qualitative comparison on Colored MNIST.}
\label{fig:colored-mnist-red-chroma-full}
\end{figure}

\begin{figure}[H]
\centering
\setlength{\tabcolsep}{0pt}
\begin{tabular}{@{}r@{\hspace{4pt}}p{0.78\linewidth}@{}}
\textbf{Input} & \includegraphics[width=\linewidth,trim=192bp 886bp 544bp 0bp,clip]{figures/colored_mnist_thick_target7_all_stages_source.png}\\[-2pt]
\textbf{TSBM (ours)} & \includegraphics[width=\linewidth,trim=192bp 716bp 544bp 170bp,clip]{figures/colored_mnist_thick_target7_all_stages_source.png}\\[-2pt]
\textbf{DPS $\gamma=0.5$} & \includegraphics[width=\linewidth,trim=192bp 172bp 544bp 714bp,clip]{figures/colored_mnist_thick_target7_all_stages_source.png}\\[-2pt]
\textbf{DPS $\gamma=1$} & \includegraphics[width=\linewidth,trim=192bp 138bp 544bp 748bp,clip]{figures/colored_mnist_thick_target7_all_stages_source.png}\\[-2pt]
\textbf{DPS $\gamma=2$} & \includegraphics[width=\linewidth,trim=192bp 104bp 544bp 782bp,clip]{figures/colored_mnist_thick_target7_all_stages_source.png}\\[-2pt]
\textbf{CondSNIS} & \includegraphics[width=\linewidth,trim=192bp 36bp 544bp 850bp,clip]{figures/colored_mnist_thick_target7_all_stages_source.png}\\[-2pt]
\textbf{CondSMC} & \includegraphics[width=\linewidth,trim=192bp 0bp 544bp 884bp,clip]{figures/colored_mnist_thick_target7_all_stages_source.png}
\end{tabular}
\caption{Full Thick-stroke qualitative comparison on Colored MNIST.}
\label{fig:colored-mnist-thick-q90-full}
\end{figure}

\begin{figure}[H]
\centering
\setlength{\tabcolsep}{0pt}
\begin{tabular}{@{}r@{\hspace{4pt}}p{0.78\linewidth}@{}}
\textbf{Input} & \includegraphics[width=\linewidth,trim=192bp 274bp 544bp 0bp,clip]{figures/colored_mnist_thin_target365_source.png}\\[-2pt]
\textbf{TSBM (ours)} & \includegraphics[width=\linewidth,trim=192bp 240bp 544bp 34bp,clip]{figures/colored_mnist_thin_target365_source.png}\\[-2pt]
\textbf{DPS $\gamma=0.5$} & \includegraphics[width=\linewidth,trim=192bp 172bp 544bp 102bp,clip]{figures/colored_mnist_thin_target365_source.png}\\[-2pt]
\textbf{DPS $\gamma=1$} & \includegraphics[width=\linewidth,trim=192bp 138bp 544bp 136bp,clip]{figures/colored_mnist_thin_target365_source.png}\\[-2pt]
\textbf{DPS $\gamma=2$} & \includegraphics[width=\linewidth,trim=192bp 104bp 544bp 170bp,clip]{figures/colored_mnist_thin_target365_source.png}\\[-2pt]
\textbf{CondSNIS} & \includegraphics[width=\linewidth,trim=192bp 36bp 544bp 238bp,clip]{figures/colored_mnist_thin_target365_source.png}\\[-2pt]
\textbf{CondSMC} & \includegraphics[width=\linewidth,trim=192bp 0bp 544bp 272bp,clip]{figures/colored_mnist_thin_target365_source.png}
\end{tabular}
\caption{Full Thin-stroke qualitative comparison on Colored MNIST.}
\label{fig:colored-mnist-thin-q10-full}
\end{figure}
\flushbottom

\subsection{CelebA}

This section gives the complete qualitative CelebA comparisons corresponding to Figure~\ref{fig:celeba-imagereward-grouped}. In addition to the methods retained in the main text we include DPS with $\gamma=0.5, 1, 2$ and CondSNIS with $K=4$ proposals per fixed source.

\newcommand{\celebaTransposedRows}[1]{\begin{tabular}{@{}r@{\hspace{4pt}}c@{}}
& \includegraphics[width=0.80\textwidth]{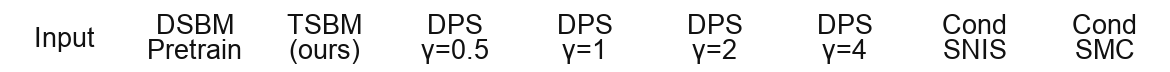}\\[-2pt]
\textbf{0} & \includegraphics[width=0.80\textwidth,trim=0bp 1950bp 0bp 0bp,clip]{#1}\\[-2pt]
\textbf{1} & \includegraphics[width=0.80\textwidth,trim=0bp 1820bp 0bp 130bp,clip]{#1}\\[-2pt]
\textbf{2} & \includegraphics[width=0.80\textwidth,trim=0bp 1690bp 0bp 260bp,clip]{#1}\\[-2pt]
\textbf{3} & \includegraphics[width=0.80\textwidth,trim=0bp 1560bp 0bp 390bp,clip]{#1}\\[-2pt]
\textbf{4} & \includegraphics[width=0.80\textwidth,trim=0bp 1430bp 0bp 520bp,clip]{#1}\\[-2pt]
\textbf{5} & \includegraphics[width=0.80\textwidth,trim=0bp 1300bp 0bp 650bp,clip]{#1}\\[-2pt]
\textbf{6} & \includegraphics[width=0.80\textwidth,trim=0bp 1170bp 0bp 780bp,clip]{#1}\\[-2pt]
\textbf{7} & \includegraphics[width=0.80\textwidth,trim=0bp 1040bp 0bp 910bp,clip]{#1}\\[-2pt]
\textbf{8} & \includegraphics[width=0.80\textwidth,trim=0bp 910bp 0bp 1040bp,clip]{#1}\\[-2pt]
\textbf{9} & \includegraphics[width=0.80\textwidth,trim=0bp 780bp 0bp 1170bp,clip]{#1}\\[-2pt]
\textbf{10} & \includegraphics[width=0.80\textwidth,trim=0bp 650bp 0bp 1300bp,clip]{#1}\\[-2pt]
\textbf{11} & \includegraphics[width=0.80\textwidth,trim=0bp 520bp 0bp 1430bp,clip]{#1}\\[-2pt]
\textbf{12} & \includegraphics[width=0.80\textwidth,trim=0bp 390bp 0bp 1560bp,clip]{#1}\\[-2pt]
\textbf{13} & \includegraphics[width=0.80\textwidth,trim=0bp 260bp 0bp 1690bp,clip]{#1}\\[-2pt]
\textbf{14} & \includegraphics[width=0.80\textwidth,trim=0bp 130bp 0bp 1820bp,clip]{#1}\\[-2pt]
\textbf{15} & \includegraphics[width=0.80\textwidth,trim=0bp 0bp 0bp 1950bp,clip]{#1}
\end{tabular}}

\begin{figure}[H]
\centering
\scriptsize
\setlength{\tabcolsep}{0pt}
\resizebox{0.7\linewidth}{!}{ \celebaTransposedRows{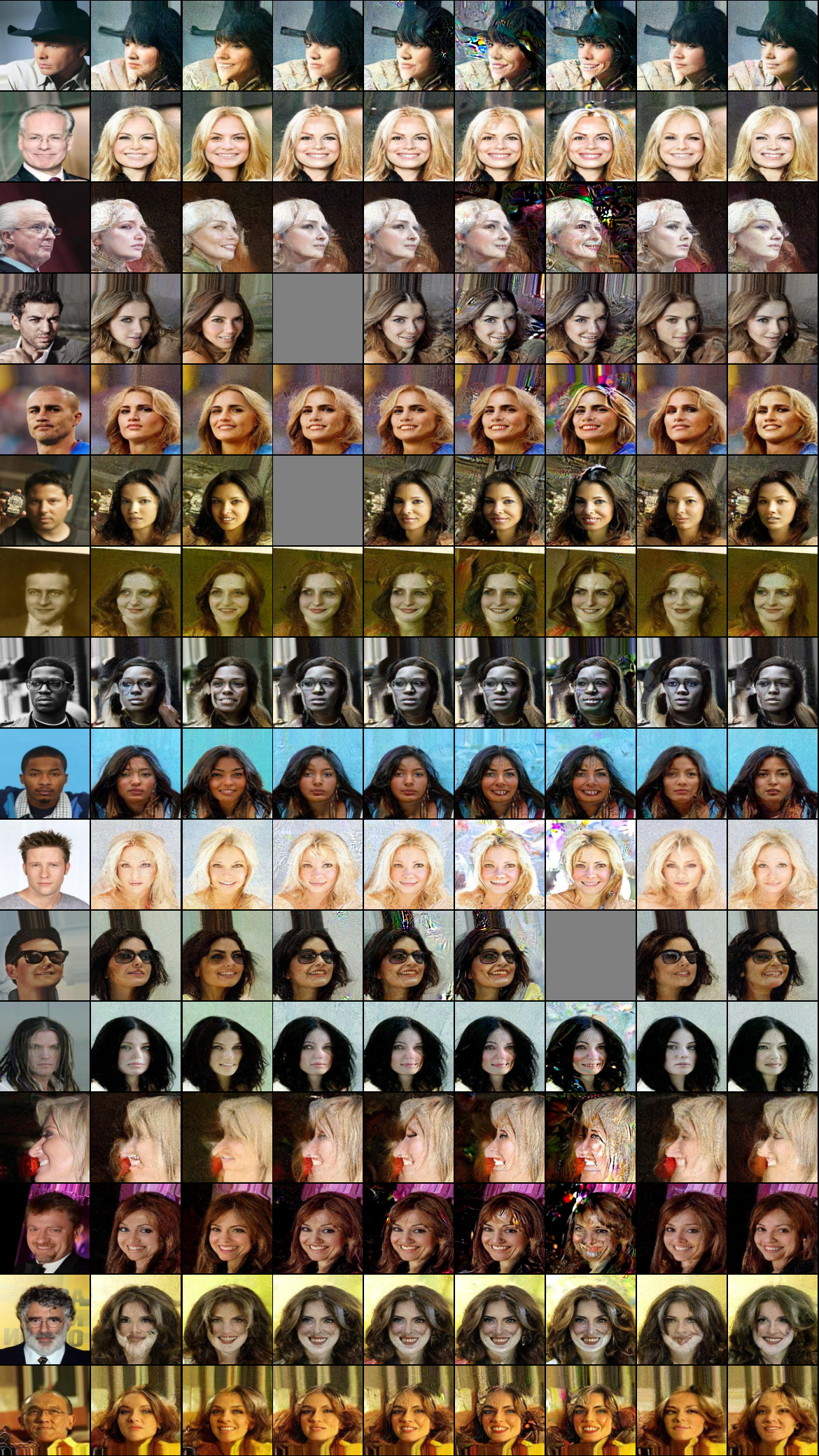}}
\caption{Full qualitative comparison for "Natural Toothy Smile" reward on CelebA. Gray outputs and visual artifacts are retained from the saved samples.}
\label{fig:celeba-smile-full}
\end{figure}

\clearpage
\begin{figure}[H]
\centering
\scriptsize
\setlength{\tabcolsep}{0pt}
\celebaTransposedRows{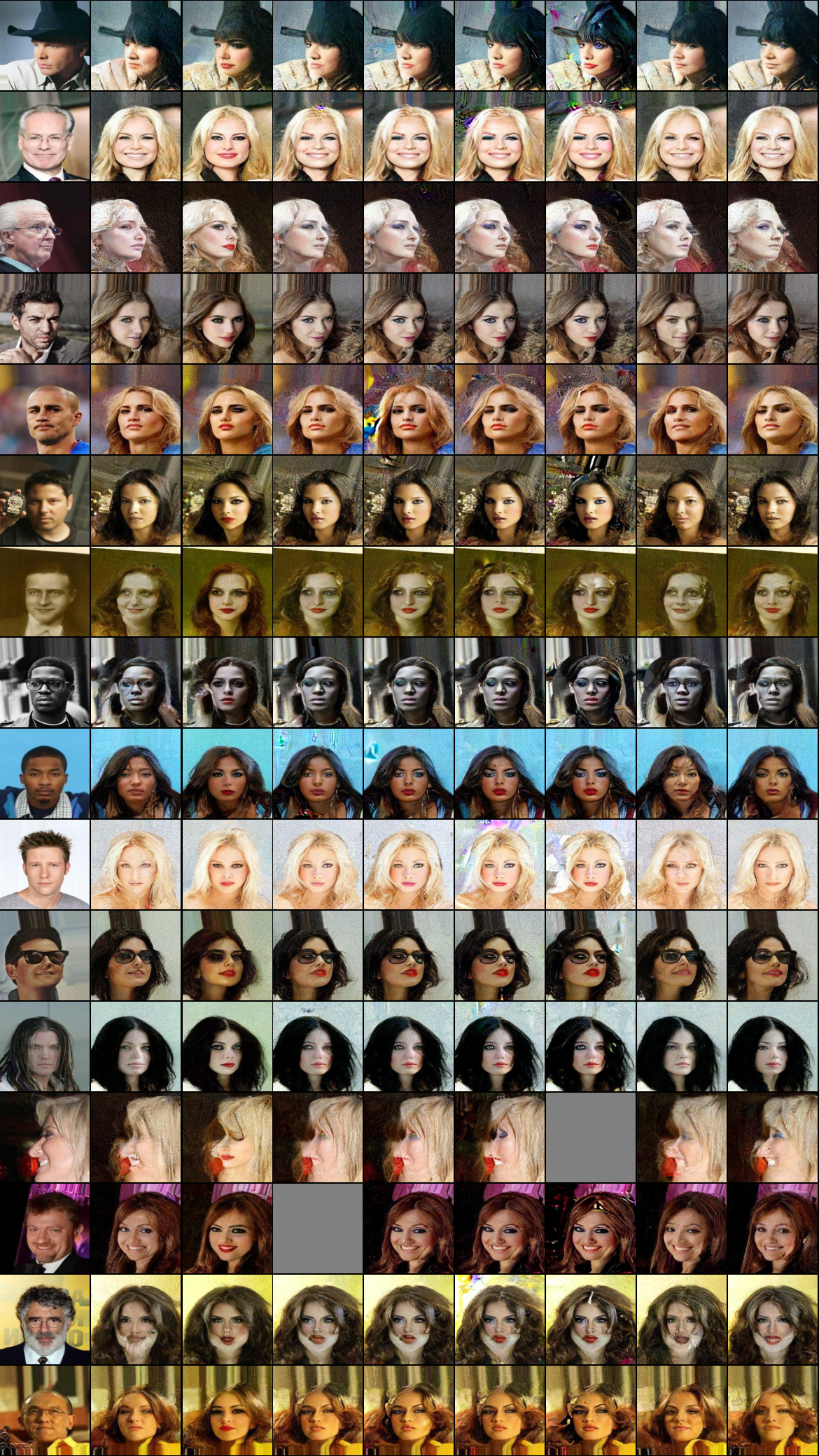}
\caption{Full qualitative comparison for "Heavy Makeup" reward on CelebA. Gray outputs and visual artifacts are retained from the saved samples.}
\label{fig:celeba-makeup-full}
\end{figure}

\clearpage
\begin{figure}[H]
\centering
\scriptsize
\setlength{\tabcolsep}{0pt}
\celebaTransposedRows{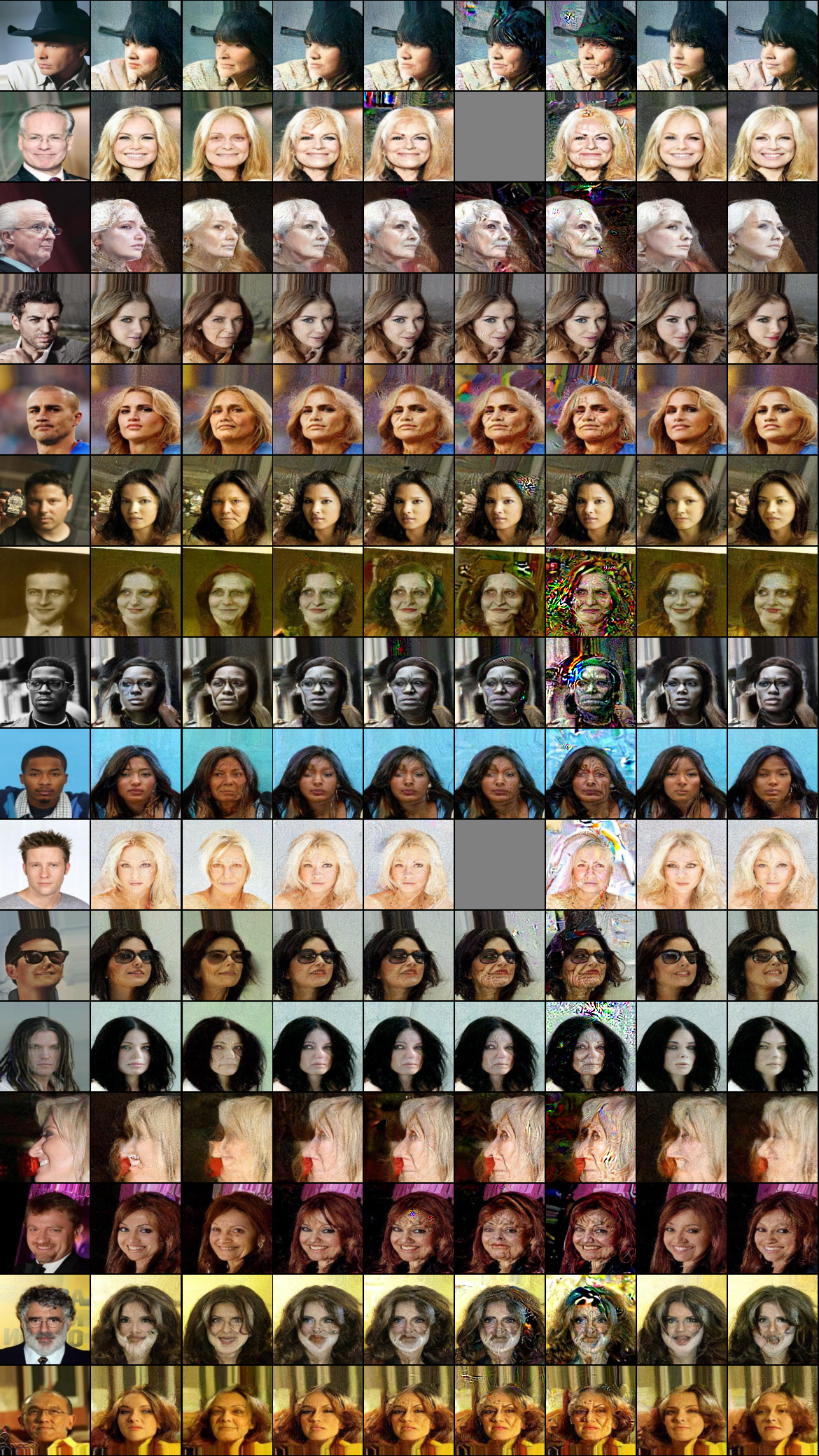}
\caption{Full qualitative comparison for "Elderly Woman" reward on CelebA.  Gray outputs and visual artifacts are retained from the saved samples.}
\label{fig:celeba-elderly-full}
\end{figure}
\flushbottom \clearpage
\section{Additional experiments}
\label{app:additional-experiments}

\subsection{Toy illustrative experiment}

\paragraph{Setup.} The source is an equally weighted mixture of eight Gaussians, and the original target is an equally weighted mixture of four Gaussians. We train the pretrain bridge with $20$ DSBM stages and $\sigma=1$ using $10^{5}$ training samples from each marginal. The reward changes the four target mixture weights from $(0.25,0.25,0.25,0.25)$ to $(0,0.2,0.2,0.6)$ and is kep with strength $\lambda=1$. We finetune the frozen DSBM for $20$ TSBM stages. Figure~\ref{fig:toy-quantitative} reports all methods on the same $10^{4}$ validation sources, with 40 Euler steps. Other experimental details are described in Appendix~\ref{app:exp-details}.

\paragraph{Results.} Figure~\ref{fig:toy-recovery} shows the original and tilted targets alongside samples from the pretrained DSBM and TSBM. While Figure~\ref{fig:toy-quantitative} shows quantitative evaluation w.r.t. tilted distribution: Total Variance, slide Wasserstein-$1$ distance and average cost of translation from source till target. One can see from quantitative results that TSBM shows the best target distributions fitting metrics, i.e., TV and sliced $W_1$, while outperfroming significantly both the inference time alignment method and TSBM without corrector. We also apply Sinkhorn to GT tilted validation samples to estimate the \textit{transport cost} of the GT tilted SB. TSBM achieves the closest cost to this reference, supporting recovery of the reference GT bridge's transport behavior. These results show that \textbf{TSBM} is the only evaluated method that both \textbf{closely matches the tilted target marginal} and \textbf{approximates the reference bridge's transport cost}.

\begingroup
\newsavebox{\toyPanelBox}
\newcommand{\toyPanel}[2]{\sbox{\toyPanelBox}{\includegraphics[
      width=0.49\linewidth,
      trim=#1bp 0bp #2bp 0bp,
      clip
    ]{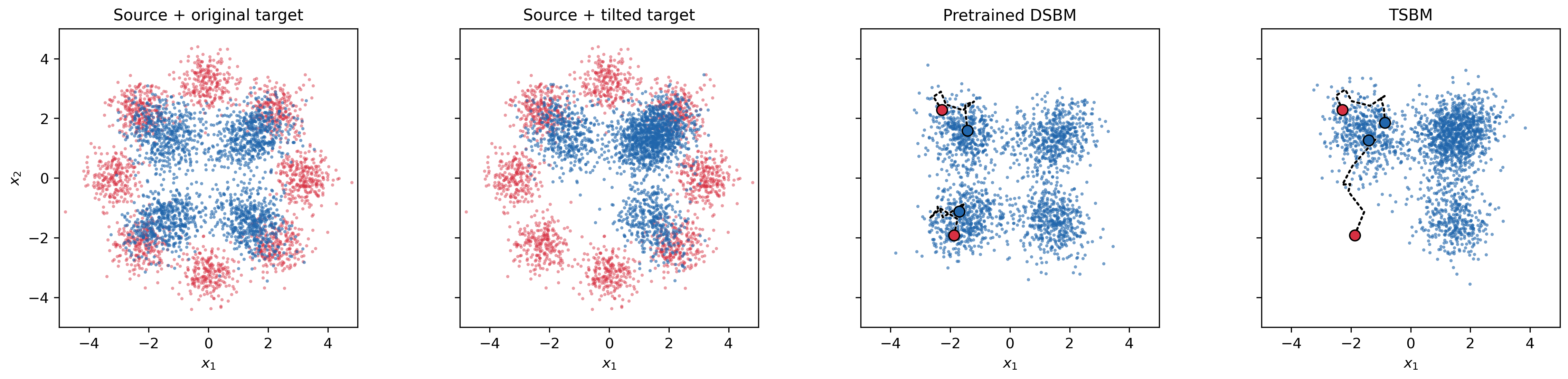}}\usebox{\toyPanelBox}\llap{\makebox[\wd\toyPanelBox][l]{\includegraphics[
      height=\ht\toyPanelBox,
      trim=0bp 0bp 1049.76bp 0bp,
      clip
    ]{figures/toy_2d_neural_four_panels.png}}}}

\begin{figure}[t]
\centering
\captionsetup{font=footnotesize}
\captionsetup[subfigure]{justification=centering}

\begin{subfigure}[t]{0.42\textwidth}
  \centering
  \vspace{0pt}

  \toyPanel{0}{835.2}\hfill
  \toyPanel{278.4}{556.8}\\[-1pt]
  \toyPanel{556.8}{278.4}\hfill
  \toyPanel{835.2}{0}

  \caption{}
  \label{fig:toy-recovery}
\end{subfigure}
\hfill
\begin{subfigure}[t]{0.56\textwidth}
  \centering
  \vspace{0pt}
\setlength{\tabcolsep}{2.2pt}
  \renewcommand{\arraystretch}{1.08}

  \begin{tabular}{@{}lrrr@{}}
    \toprule
    Method
    & TV $\downarrow$
    & \shortstack{sliced\\$W_1$ $\downarrow$}
    & Cost \\
    \midrule
    Exact Data Sinkhorn & -             & -             & 2.300 \\
    DSBM                & .354          & .788          & 1.604 \\
    DPS $\gamma=0.5$    & .309          & .560          & 1.785 \\
    DPS $\gamma=1$      & .281          & .493          & 1.858 \\
    DPS $\gamma=2$      & .254          & .441          & 1.896 \\
    DPS $\gamma=4$      & .234          & .402          & 1.933 \\
    SNIS        & .298          & .536          & 1.788 \\
    SMC         & .287          & .511          & 1.828 \\
TSBM (ours)         & \textbf{.117} & \textbf{.117} & \textbf{2.180} \\
    \bottomrule
  \end{tabular}

  \caption{}
  \label{fig:toy-metrics}
  \label{fig:toy-quantitative}
\end{subfigure}

\caption{\textbf{Two-dimensional toy experiment.}
(a) The first two panels overlay the red source with the blue
original and tilted targets; the remaining panels show terminal
samples and generation trajectories from DSBM and TSBM.Red and blue dots with black outlines mark trajectory starts and
endpoints, respectively; black dotted lines trace the trajectories.
(b) Total variation and sliced $W_1$ are measured against the
ground-truth tilted distribution. Cost denotes
$\mathbb{E}\lVert \hat{x}_1-x_0\rVert_2$.
Bold marks the best distributional metrics and the cost closest
to the Sinkhorn reference.
}
\label{fig:toy-results}
\end{figure}
\endgroup

\subsection{CelebA-128 inference time and working memory}
\label{app:celeba-inference-cost}

We benchmark inference in the representative CelebA $128\times128$ male-to-female translation setting with the "Natural Toothy Smile" reward at strength $\lambda=30$. The benchmark uses the corresponding stage-2 TSBM checkpoint, DPS with $\gamma=1$, and one NVIDIA A100-SXM4-80GB. Every trajectory uses 100 drift-network evaluations, i.e., NFE. Model and checkpoint loading, dataset construction, post-hoc metrics, serialization, and etc are excluded from the timed region, and CUDA is synchronized around each trial. The single-output protocol runs the algorithm to process one input image, while for CondSNIS and CondSMC this means that they run processing of $K$ parallel samples, i.e., on population. Table~\ref{tab:celeba-inference-cost} reports mean and standard deviation over 16 timed trials after one warm-up. The batch-16 protocol test the \textit{batched processing} and processes 16 input images and measures the time needed to produce one output for each image. CondSNIS and CondSMC use $K$ particles per input, with at most 16 trajectories evaluated simultaneously. Table~\ref{tab:celeba-inference-cost} reports the median total time over three timed repetitions, divided by 16 to give the time per output. Working memory is the peak increase in live PyTorch GPU tensor memory during a trial, measured relative to the allocation immediately before it starts. It measures the extra memory inference needs, rather than the model’s total GPU memory footprint.

\begin{table}[H]
\caption{\textbf{Inference cost on CelebA $128\times128$.} Single-output wall-clock time is the mean $\pm$ standard deviation over 16 trials. Batch-16 time is the median total time divided by 16 source outputs. Relative time uses the single-output TSBM mean. An asterisk ($^\ast$) indicates visual artifacts.}
\label{tab:celeba-inference-cost}
\centering
\small
\setlength{\tabcolsep}{4pt}
\begin{tabular}{lrrrrr}
\toprule
Method & $K$ & \shortstack{Single-output\\time (s)} & \shortstack{Batch-16\\s/output} & \shortstack{Relative to\\TSBM} & \shortstack{Working\\memory (GiB)} \\
\midrule
Pretrained DSBM & 1 & $1.468 \pm 0.010$ & 0.355 & $1.00\times$ & 0.180 \\
TSBM (ours) & 1 & $1.460 \pm 0.005$ & 0.357 & $1.00\times$ & 0.180 \\
DPS$^\ast$ $\gamma=1$ & 1 & $7.116 \pm 0.011$ & 1.836 & $4.88\times$ & 2.229 \\
SNIS & 4 & $2.089 \pm 0.004$ & 1.452 & $1.43\times$ & 0.379 \\
SNIS & 8 & $3.094 \pm 0.001$ & 2.905 & $2.12\times$ & 0.758 \\
SNIS & 16 & $5.794 \pm 0.002$ & 5.811 & $3.97\times$ & 1.512 \\
SMC & 4 & $4.282 \pm 0.006$ & 3.629 & $2.93\times$ & 0.383 \\
SMC & 8 & $7.431 \pm 0.002$ & 7.256 & $5.09\times$ & 0.767 \\
SMC & 16 & $14.459 \pm 0.004$ & 14.512 & $9.91\times$ & 1.527 \\
\bottomrule
\end{tabular}
\end{table}

TSBM has essentially the same inference cost as the pretrained bridge because it uses the same network architecture and requires neither reward evaluation nor differentiation at sampling time. In the single-output protocol, DPS is $4.88\times$ slower and requires $2.229$~GiB of working memory, compared with $0.180$~GiB for TSBM. Its larger workspace comes from differentiating through the endpoint prediction and reward model at every step. SNIS and SMC costs increase with $K$ because all particles are propagated for each fixed input. Their working memory grows approximately linearly with $K$. SMC is slower than SNIS at a fixed $K$ because it additionally evaluates incremental rewards and performs sequential weighting and resampling. 

All methods are measured in one process that keeps the DSBM, TSBM, and reward models resident. Consequently, working memory is the appropriate method comparison, whereas the absolute process peak is not a standalone deployment-memory requirement. We define:
\[
M_{\mathrm{working}}
=
M_{\mathrm{peak\ allocated\ during\ trial}}
-
M_{\mathrm{allocated\ immediately\ before\ trial}}.
\]

\subsection{Colored MNIST Red-Chroma. Static TSBM corrector and TSBM stages ablation.}\label{app:cmnist-red-chroma-static-corrector-ablation}

In this section we provide additional analysis for the Colored MNIST translation with "red chroma" reward. We assert the stage-wise dynamics and convergence of \tsbm{} in Figure~\ref{fig:red-chroma-iteration-curves}, which tracks reward and red target hue error, i.e., deviation of translated image HUE from the target red one, across five TSBM stages, which asserts how much the model does follow the reward. Furthermore, we compare the full method (\tsbm{}) with the no-corrector ablation. In that case, \tsbm{} corrector isn't trained and $h^{(k)}=\hold$, while method is trained with the same computational budget in terms of controller optimization. Both metrics are evaluated on the same $1,024$ fixed sources with 30 solver steps. Lower values are better for both reward and red-target hue error. Finally, we present qualitative evaluation of \tsbm{} vs \tsbm{} with static corrector in Figure~\ref{fig:red-chroma-last-sixteen}.

\begin{figure}[H]
\centering
\includegraphics[width=\textwidth]{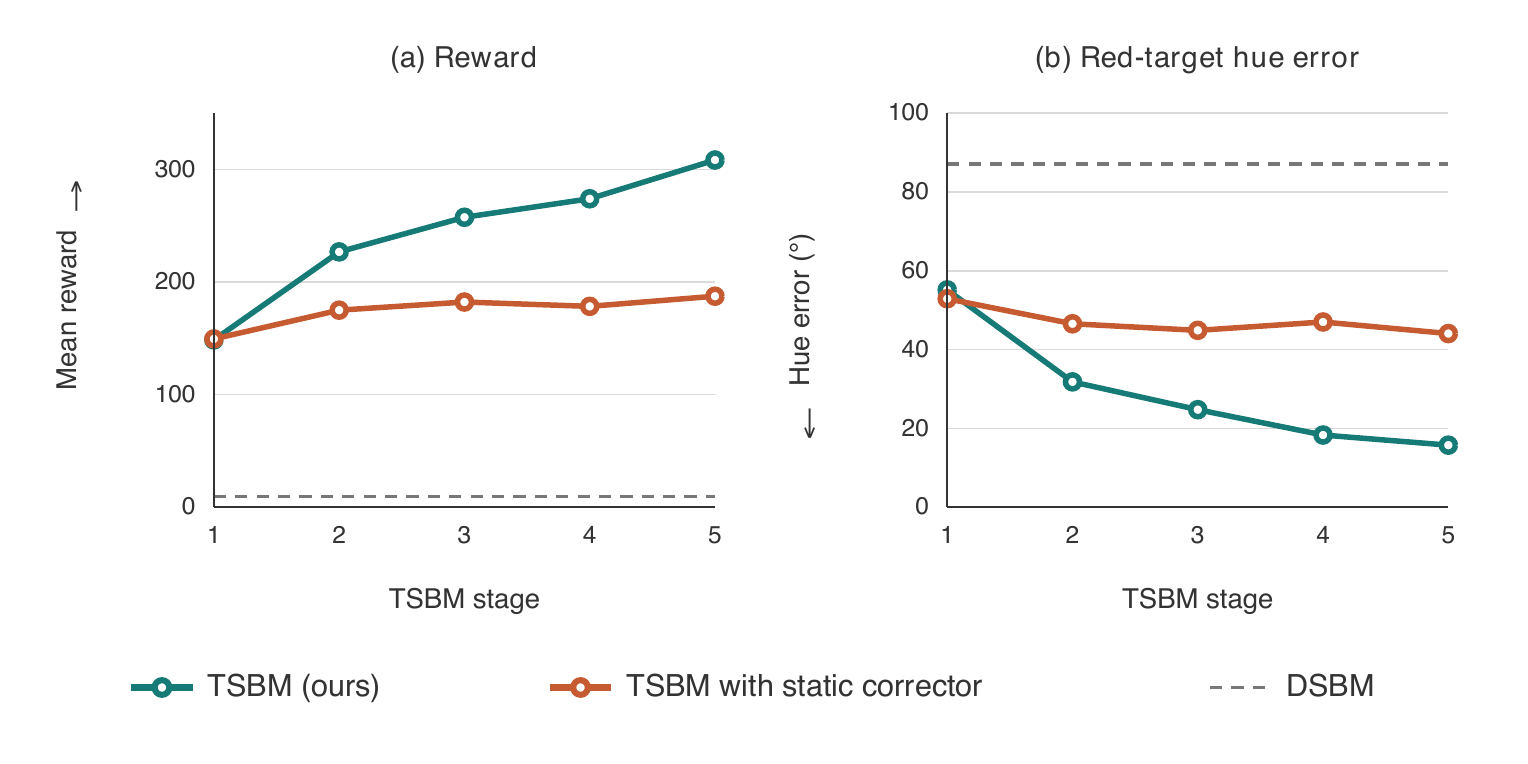}
\caption{Quantitative comparison of \tsbm{} and \tsbm{} with static corrector on Colored MNIST dataset with "red chroma" reward. The left panel shows mean terminal reward and the right panel the red-target hue error. Solid curves are TSBM and the no-corrector ablation, while dashed gray lines mark the frozen Pretrained DSBM result.}
\label{fig:red-chroma-iteration-curves}
\end{figure}

One can see that \tsbm{} improves both the reward and hue error with iteration, while \tsbm{} with static corrector does improve with iterations, but much slower than regular \tsbm{}, see Figure~\ref{fig:red-chroma-iteration-curves}. Figure~\ref{fig:red-chroma-last-sixteen} confirms the idea that regular \tsbm{} follows the reward noticeably better than \tsbm{} with static corrector. These results supports the theory and the need for accurate corrector $h^{(k)}$.

\begin{figure}[H]
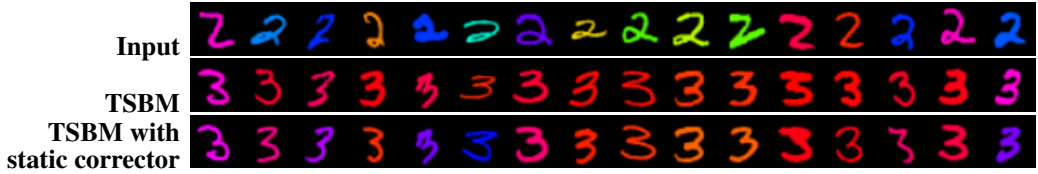

\centering
\setlength{\tabcolsep}{0pt}
\begin{tabular}{@{}p{0.17\linewidth}@{\hspace{4pt}}p{0.80\linewidth}@{}}
\raggedleft\textbf{Input} & \includegraphics[width=\linewidth,trim=736bp 274bp 0bp 0bp,clip]{figures/colored_mnist_red_chroma_target_source.png}\\[-2pt]
\raggedleft\textbf{TSBM} & \includegraphics[width=\linewidth,trim=736bp 240bp 0bp 34bp,clip]{figures/colored_mnist_red_chroma_target_source.png}\\[-2pt]
\raggedleft\textbf{\shortstack[r]{TSBM with\\static corrector}} & \includegraphics[width=\linewidth,trim=736bp 206bp 0bp 68bp,clip]{figures/colored_mnist_red_chroma_target_source.png}
\end{tabular}
\caption{Qualitative comparison of \tsbm{} and \tsbm{} with static corrector samples on Colored MNIST dataset with "red chroma" reward.}
\label{fig:red-chroma-last-sixteen}
\end{figure}

\subsection{Colored MNIST Thick-target. TSBM stages ablation.}\label{app:cmnist-thick-stages-ablation}

In this section we provide additional analysis on TSBM dynamics across stages on the Colored MNIST dataset with "thick stroke" reward. We train perform the 20 \tsbm{} stages. In Figure~\ref{fig:thick-target7-iteration-curves} we show the reward, stroke thickness as \textit{reward optimization metrics}  and  hue error w.r.t. input as \textit{fidelity metric}. 

One can see that \tsbm{} improves reward across the \tsbm{} stages, while most of the gains are received in the first several stages. Furthermore, the actual thickness target seems almost achieved in the first \tsbm{} stages, however it indeed improves within \tsbm{} later stages. The hue error seems somewhat stochastic, however it is almost always within $13^\circ$ and $16^\circ$ degrees and doesn't diverge much. We assume that its stochasticy is the product of stocastic evaluation and reward side effects.

\begin{figure}[H]
\centering
\begin{subfigure}[t]{\textwidth}
\centering
\includegraphics[width=\linewidth]{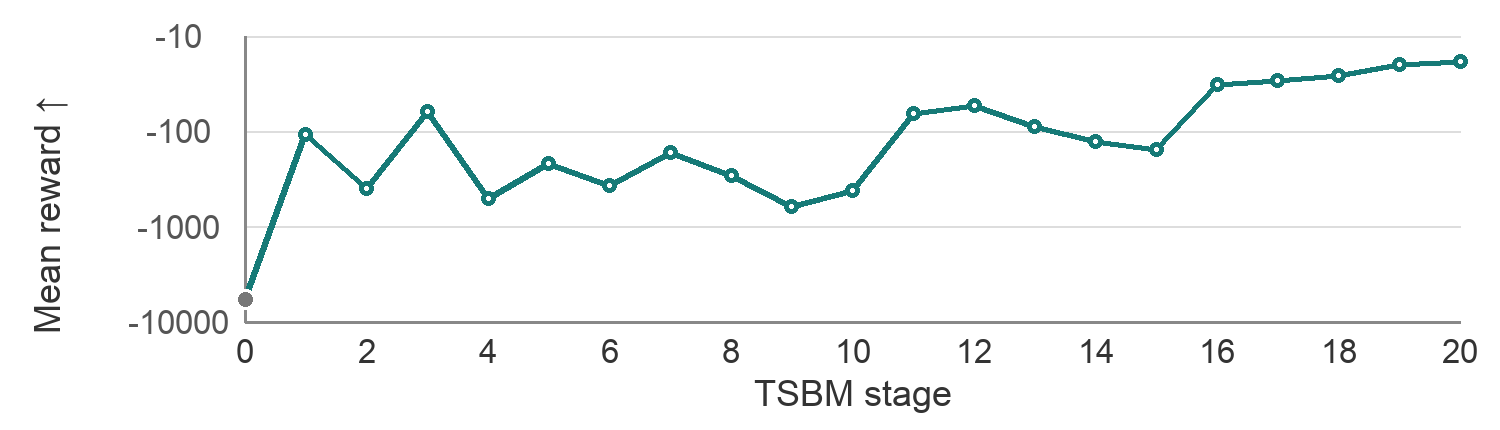}
\caption{Mean reward $\uparrow$; logarithmic magnitude scale.}
\end{subfigure}

\begin{subfigure}[t]{\textwidth}
\centering
\includegraphics[width=\linewidth]{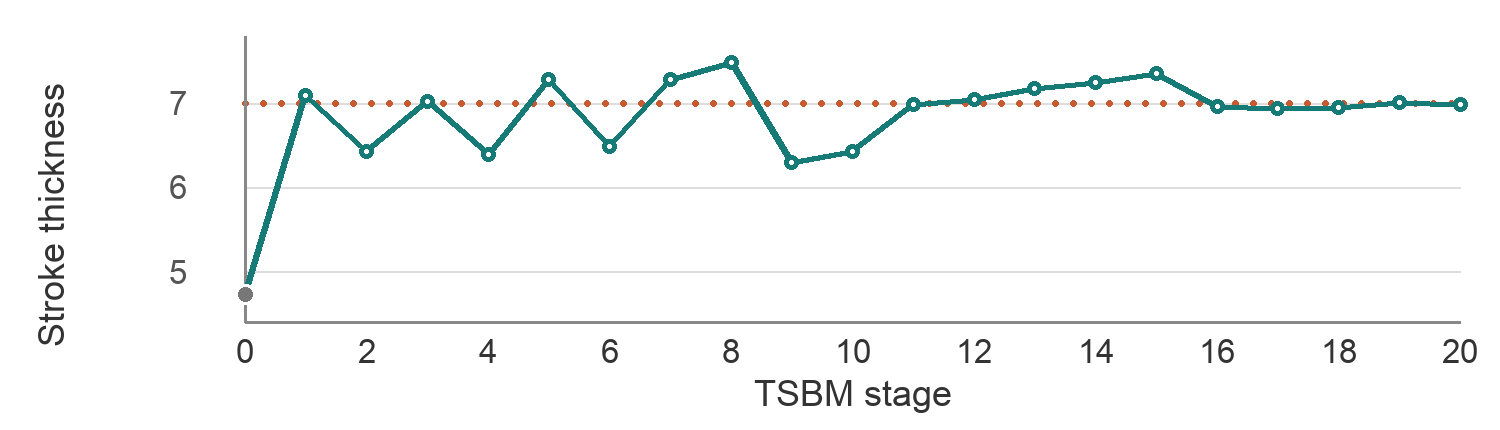}
\caption{Stroke thickness.}
\end{subfigure}

\begin{subfigure}[t]{\textwidth}
\centering
\includegraphics[width=\linewidth]{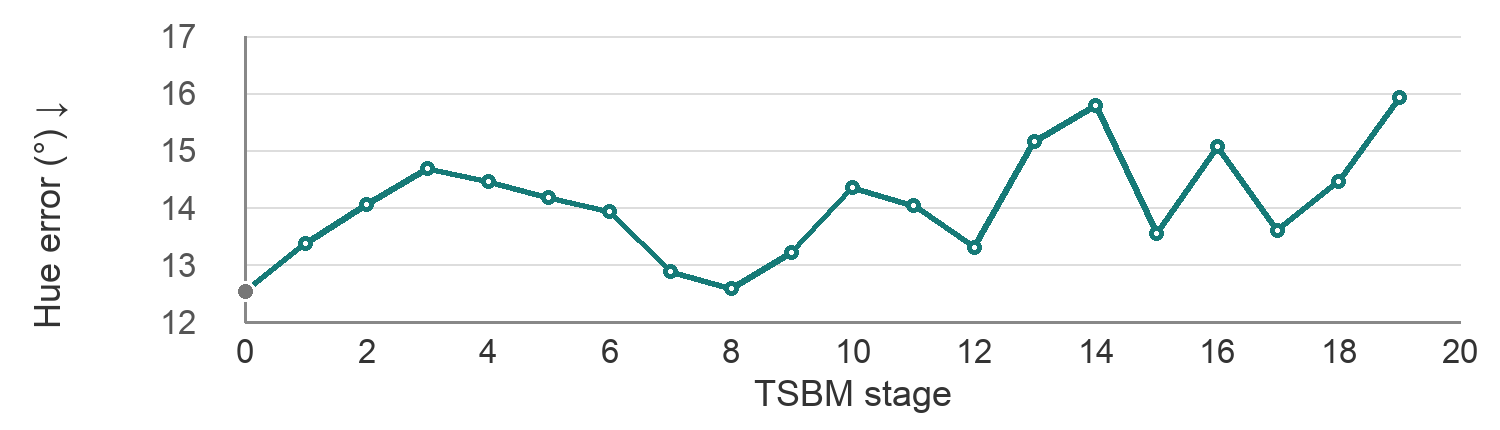}
\caption{Foreground hue error $\downarrow$.}
\end{subfigure}

\medskip
{\normalsize
\textcolor[HTML]{167A77}{\rule{1.5em}{1.4pt}}\,TSBM (ours)\hspace{1em}
\tikz[baseline=-0.6ex]{\fill[gray] (0,0) circle (0.8mm);}\,Pretrained DSBM (stage 0 reference)\hspace{1em}
\textcolor[HTML]{C65B32}{\tikz[baseline=-0.6ex]{\draw[dotted,line width=1.4pt] (0,0)--(0.45,0);}}\,Target thick stroke\par}
\caption{
Quantitative analysis of different \tsbm{} stages on Colored MNIST dataset with "thick stroke" reward. The panels show mean reward, stroke thickness, and foreground hue error over $1,024$ fixed test source images. The gray marker is the frozen pretrained DSBM. The dotted line marks the stroke-thickness target.}
\label{fig:thick-target7-iteration-curves}
\end{figure}

\subsection{CelebA additional quantitative analysis}
\label{app:celeba-additional-quantitative}

Here we present additional metrics complementing the study in Table~\ref{tab:celeba-imagereward-grouped} for CelebA dataset with all the rewards. Table~\ref{tab:celeba-additional-quantitative} shows source-to-output distances, i.e., MSE and LPIPS \cite{zhang2018lpips}, and two additional prompt-alignment scores, i.e., HPSV2.1 \cite{wu2023hpsv2} and CLIPScore \cite{hessel2021clipscore}, which recieve data preprocessed the same way as for ImageReward and PickScore, see Appendix~\ref{app:reward-celeba}. MSE and LPIPS measure how far an output $y$ moves from its aligned input $x$, i.e., $\text{MSE}(x, y)$ and $\text{LPIPS}(x, y)$. 

One can see that \tsbm{} and DPS $\gamma=4$ are the two best methods by HPS and CLIPScore. This confirms the strong \tsbm{} generalization properties among several rewards, but shows DPS $\gamma=4$ as competitor. However, as we noticed before DPS methods, especially with higher $\gamma$ introduce a significant level of artifacts, see Figures~\ref{fig:celeba-elderly-full}, \ref{fig:celeba-smile-full}, \ref{fig:celeba-makeup-full}, \ref{fig:celeba-imagereward-grouped}.  Across all three prompts, TSBM has the largest source-to-output MSE but lower LPIPS than every DPS variant. It therefore changes more pixels while remaining closer to the input under the perceptual feature metric. CondSNIS and CondSMC have smaller input-to-output distances but also weaker prompt scores, so those distances alone do not indicate better edits.

\begin{table}[H]
\centering
\caption{Additional CelebA quantitative results. MSE and LPIPS compare each output with its input, smaller values indicate less change. HPS v2.1 and CLIPScore use the same prompts as in Appendix~\ref{app:reward-celeba} with the same translation, crop, and horizontal-flip augmentation, larger values indicate better prompt alignment. \textbf{Bold} and \underline{underline} marks the largest and second largest HPS or CLIPScore mean within each prompt. ($^\ast$) denotes the DPS settings marked for visual artifacts in Table~\ref{tab:celeba-imagereward-grouped}.}
\label{tab:celeba-additional-quantitative}
\small
\setlength{\tabcolsep}{9pt}
\renewcommand{\arraystretch}{0.94}
\begin{tabular}{@{}lrrrr@{}}
\toprule
Method & \shortstack{MSE\\input-output} & \shortstack{LPIPS\\input-output} & \shortstack{HPS v2.1 $\uparrow$} & \shortstack{CLIPScore $\uparrow$} \\
\midrule
\multicolumn{5}{@{}l}{\textbf{Natural Toothy Smile}} \\
Pretrained DSBM & 0.1677 & 0.2983 & 0.1806 & 19.6923 \\
TSBM (ours) & 0.1923 & 0.3074 & \textbf{0.2012} & \underline{21.5642} \\
DPS $\gamma=0.5$ & 0.1815 & 0.3363 & 0.1867 & 20.2365 \\
DPS$^\ast$ $\gamma=1$ & 0.1827 & 0.3409 & 0.1928 & 20.8415 \\
DPS$^\ast$ $\gamma=2$ & 0.1843 & 0.3483 & 0.1982 & 21.3689 \\
DPS$^\ast$ $\gamma=4$ & 0.1901 & 0.3653 & \underline{0.1997} & \textbf{21.6390} \\
CondSNIS & 0.1676 & 0.2981 & 0.1830 & 19.9052 \\
CondSMC & 0.1682 & 0.2975 & 0.1890 & 20.4165 \\
\midrule
\multicolumn{5}{@{}l}{\textbf{Heavy Makeup}} \\
Pretrained DSBM & 0.1677 & 0.2983 & 0.1540 & 18.1912 \\
TSBM (ours) & 0.1955 & 0.3159 & \underline{0.1858} & \textbf{21.6400} \\
DPS $\gamma=0.5$ & 0.1828 & 0.3402 & 0.1687 & 19.7912 \\
DPS$^\ast$ $\gamma=1$ & 0.1820 & 0.3460 & 0.1761 & 20.4257 \\
DPS$^\ast$ $\gamma=2$ & 0.1848 & 0.3532 & 0.1825 & 20.8767 \\
DPS$^\ast$ $\gamma=4$ & 0.1827 & 0.3528 & \textbf{0.1872} & \underline{21.1343} \\
CondSNIS & 0.1677 & 0.2983 & 0.1564 & 18.4695 \\
CondSMC & 0.1686 & 0.3021 & 0.1653 & 19.5327 \\
\midrule
\multicolumn{5}{@{}l}{\textbf{Elderly Woman}} \\
Pretrained DSBM & 0.1677 & 0.2983 & 0.1885 & 18.8912 \\
TSBM (ours) & 0.1878 & 0.3018 & \textbf{0.2003} & \textbf{23.4400} \\
DPS $\gamma=0.5$ & 0.1722 & 0.3160 & 0.1877 & 20.3012 \\
DPS$^\ast$ $\gamma=1$ & 0.1725 & 0.3279 & 0.1903 & 21.2336 \\
DPS$^\ast$ $\gamma=2$ & 0.1761 & 0.3452 & 0.1935 & 22.0910 \\
DPS$^\ast$ $\gamma=4$ & 0.1859 & 0.3811 & \underline{0.1953} & \underline{22.7870} \\
CondSNIS & 0.1675 & 0.2987 & 0.1884 & 19.2562 \\
CondSMC & 0.1666 & 0.2998 & 0.1875 & 19.8498 \\
\bottomrule
\end{tabular}
\end{table}

\clearpage
\subsection{Additional CelebA rewards}
\label{app:celeba-additional-rewards}

We evaluate two further ImageReward preferences on CelebA $128\times128$ male-to-female translation. Both use the contrastive reward in Eq.~\ref{eq:celeba-prompt-reward} with $\lambda=30$. The positive and negative prompts are, respectively:
\begin{itemize}
\item \emph{Aesthetics}: ``a sharp, clean, well-exposed photograph of a face'' and ``a blurry, noisy, distorted, low-quality photograph of a face''.
\item \emph{Eyeglasses with Thick Frames}: ``a natural portrait of a person wearing eyeglasses with thick visible frames'' and ``a natural portrait of a person without eyeglasses''.
\end{itemize}
As in the main CelebA experiment, ImageReward and PickScore are evaluated on the positive prompt alone, after translation, crop, and horizontal-flip augmentation. CLIP-IQA measures naturalness. Each table reports means over the same \textit{128} fixed inputs within its reward experiment, using 100 bridge steps. CondSNIS and CondSMC use four particles per input. The TSBM was ran for 5 stages. 

\begin{table}[H]
\centering
\caption{Aesthetics reward on CelebA-128. ImageReward and PickScore use the positive prompt with translation, crop, and horizontal-flip augmentation; CLIP-IQA measures naturalness. Means are over 128 fixed inputs, with TSBM evaluated after stage 5. All metrics are higher-is-better. Bold marks the largest mean. ($^\ast$) indicates visible artifacts in the saved samples.}
\label{tab:celeba-aesthetics}
\small
\setlength{\tabcolsep}{7pt}
\begin{tabular}{@{}lrrr@{}}
\toprule
Method & ImageReward $\uparrow$ & PickScore $\uparrow$ & CLIP-IQA $\uparrow$ \\
\midrule
Pretrained DSBM & -0.2353 & 18.0223 & 0.5446 \\
\midrule
TSBM (ours) & \textbf{0.6283} & \textbf{18.6727} & \textbf{0.5738} \\
\midrule
DPS$^\ast$ $\gamma=0.5$ & 0.1858 & 18.1216 & 0.5453 \\
DPS$^\ast$ $\gamma=1$ & 0.2383 & 17.9490 & 0.4901 \\
DPS$^\ast$ $\gamma=2$ & 0.5524 & 17.5891 & 0.3833 \\
DPS$^\ast$ $\gamma=4$ & 0.5592 & 17.2191 & 0.2649 \\
CondSNIS & -0.1098 & 18.0474 & 0.5684 \\
CondSMC & -0.0164 & 18.2051 & 0.5632 \\
\bottomrule
\end{tabular}
\end{table}

\begin{table}[H]
\centering
\caption{Eyeglasses with thick frames reward on CelebA-128. Metrics and augmentation follow Table~\ref{tab:celeba-aesthetics}. Means are over 128 fixed inputs, with TSBM evaluated after stage 5. All metrics are higher-is-better. Bold marks the largest mean. ($^\ast$) indicates visible artifacts in the saved samples.}
\label{tab:celeba-eyeglasses-thick-frames}
\small
\setlength{\tabcolsep}{7pt}
\begin{tabular}{@{}lrrr@{}}
\toprule
Method & ImageReward $\uparrow$ & PickScore $\uparrow$ & CLIP-IQA $\uparrow$ \\
\midrule
Pretrained DSBM & -1.8730 & 17.2151 & 0.5446 \\
\midrule
TSBM (ours) & \textbf{-0.0468} & \textbf{18.1322} & 0.4754 \\
\midrule
DPS$^\ast$ $\gamma=0.5$ & -1.8729 & 17.1511 & 0.5538 \\
DPS$^\ast$ $\gamma=1$ & -1.8260 & 17.0452 & 0.5499 \\
DPS$^\ast$ $\gamma=2$ & -1.7740 & 16.9616 & 0.5365 \\
DPS$^\ast$ $\gamma=4$ & -1.7151 & 16.8643 & 0.4952 \\
CondSNIS & -1.8324 & 17.1915 & 0.5462 \\
CondSMC & -1.8651 & 17.1799 & \textbf{0.5547} \\
\bottomrule
\end{tabular}
\end{table}

TSBM has the largest ImageReward and PickScore means for both rewards. For aesthetics it also has the largest CLIP-IQA naturalness mean. For eyeglasses, CLIP-IQA falls from 0.5446 for the pretrained DSBM to 0.4754 for TSBM, while CondSMC has the largest naturalness mean at 0.5547. The full sample comparisons in Figures~\ref{fig:celeba-aesthetics-full} and~\ref{fig:celeba-eyeglasses-thick-frames-full} show the corresponding edits and the artifacts produced by DPS. The task of adding eyeglasses to a person appears to be particularly challenging: the baselines achieve poor reward alignment under both the training ImageReward and evaluation PickScore metrics. While TSBM substantially improves over the baselines, it does not fully solve the problem, as shown in Figure~\ref{fig:celeba-eyeglasses-thick-frames-full}.

\clearpage
\begin{figure}[H]
\centering
\scriptsize
\setlength{\tabcolsep}{0pt}
\resizebox{0.7\linewidth}{!}{\celebaTransposedRows{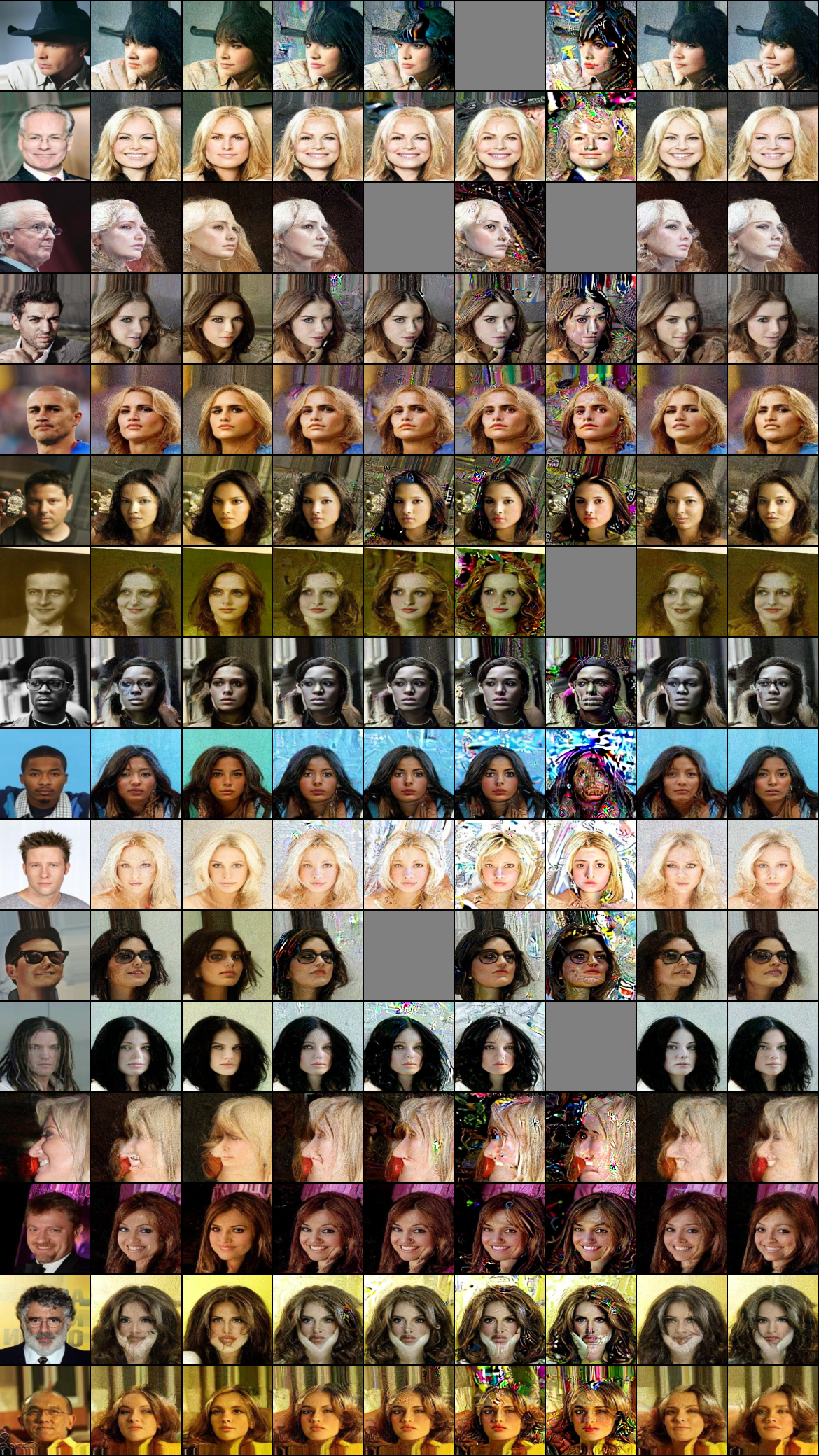}}
\caption{Full qualitative comparison for the aesthetics reward on CelebA. Rows show the first 16 fixed sources; columns show the input, pretrained DSBM, TSBM (ours) at stage 10, DPS with $\gamma\in\{0.5,1,2,4\}$, CondSNIS, and CondSMC. Gray outputs and visual artifacts are retained from the saved samples.}
\label{fig:celeba-aesthetics-full}
\end{figure}

\clearpage
\begin{figure}[H]
\centering
\scriptsize
\setlength{\tabcolsep}{0pt}
\resizebox{0.7\linewidth}{!}{\celebaTransposedRows{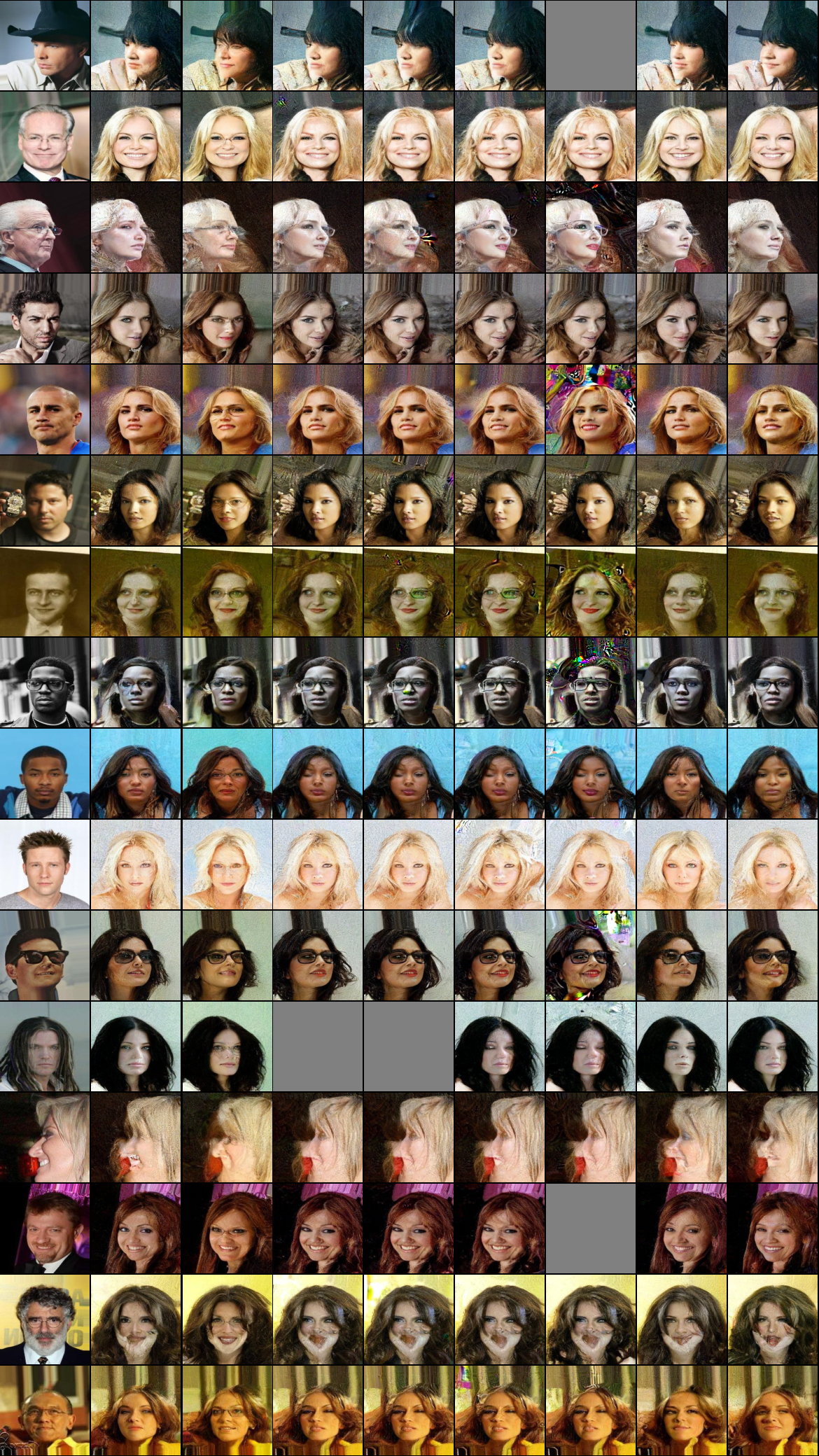}}
\caption{Full qualitative comparison for the eyeglasses with thick frames reward on CelebA. Rows show the first 16 fixed sources; columns show the input, pretrained DSBM, TSBM (ours) at stage 5, DPS with $\gamma\in\{0.5,1,2,4\}$, CondSNIS, and CondSMC. Gray outputs and visual artifacts are retained from the saved samples.}
\label{fig:celeba-eyeglasses-thick-frames-full}
\end{figure}
 
\end{document}